\documentclass[journal,twoside]{IEEEtran}

\usepackage{cite} 
\usepackage[cmex10]{amsmath}
\usepackage{amsfonts,amssymb,amsthm,mathtools,bbm,verbatim}
\usepackage[utf8]{inputenc} 
\usepackage{hyperref}
\hypersetup{colorlinks = true, linkcolor = blue, citecolor = blue, urlcolor=blue}
\usepackage[all]{hypcap} 
\usepackage{physics} 
\usepackage[usenames,dvipsnames]{color}
\usepackage{url} 

\usepackage{breakurl}
\usepackage[pdftex]{graphicx} 
\usepackage{array}
\usepackage{mleftright}       
\mleftright                        
\usepackage[capitalize]{cleveref}
\usepackage{cuted}
\usepackage{caption}
\usepackage{subcaption}
\usepackage{tikz}
\usetikzlibrary{positioning, arrows.meta}
\allowdisplaybreaks           
\crefformat{equation}{(#2#1#3)} 
\crefmultiformat{equation}{(#2#1#3)}
{ and~(#2#1#3)}{, (#2#1#3)}{, and~(#2#1#3)}
\usepackage{romannum} 
\crefname{assumption}{Assumption}{Assumptions}
\Crefname{assumption}{Assumption}{Assumptions}
\usepackage{balance}   

\newcommand{\vertiii}[1]{{\left\vert\kern-0.25ex\left\vert\kern-0.25ex\left\vert #1 
    \right\vert\kern-0.25ex\right\vert\kern-0.25ex\right\vert}}

\DeclareSymbolFont{extraitalic}      {U}{zavm}{m}{it}
\DeclareMathSymbol{\stigma}{\mathord}{extraitalic}{168}
\renewcommand{\complement}{\mathrm{c}}
\def\thr{{\mathsf{thr}}}

\def\calS{{\mathcal{S}}}

\def\Z{{\mathcal{Z}}}
\def\G{{\mathcal{G}}}

\def\calM{{\mathcal{M}}}
\def\P{{\mathbb{P}}}
\def\E{{\mathbb{E}}}
\def \PE{{\mathcal{P}_\mathcal{E}}}
\def\var{{\mathrm{var}}}

\def\R{{\mathbb{R}}}
\def\N{{\mathbb{N}}}
\def\Q{{\mathbb{Q}}}
\def\1{{\mathbf{1}}}
\def\0{{\mathbf{0}}}
\def\I{{\mathbbm{1}}}
\def \dmax {{d}}
\def \calE {{\mathcal{E}}}

\def \pihat {{\hat{\pi}}}

\def \Shat {{\hat{S}}}
\def \S {{S}}
\def \Q {{Q}}

\def \calP {{\mathcal{P}}}

\DeclareRobustCommand{\BTLh}{\ensuremath{\mathsf{BTL}_h}}

\def\Ber{{\mathsf{Bernoulli}}}
\def \pisb {{\pi_{B}^{*}}}
\def \Pisb {{\Pi_{B}^{*}}}

\DeclareMathOperator\diag{{diag}}

\newcommand{\T}{\mathrm{T}} 
\newcommand{\F}{\mathrm{F}} 
\renewcommand{\ip}[2]{\langle #1, #2 \rangle}
\DeclareMathOperator*{\argmax}{arg\,max}

\newtheorem{theorem}{Theorem}
\newtheorem{lemma}{Lemma}
\newtheorem{proposition}{Proposition}
\newtheorem{corollary}{Corollary}
\newtheorem{assumption}{Assumption}
\newtheorem{definition}{Definition}

\begin{document}

\pagenumbering{arabic}
\bstctlcite{IEEEexample:BSTcontrol} 

\title{Minimax Hypothesis Testing for the Bradley-Terry-Luce Model} 

\author{Anuran~Makur~and~Japneet~Singh
\thanks{The author ordering is alphabetical. This work was supported by the National Science Foundation (NSF) CAREER Award under Grant CCF-2337808. An earlier version of this paper was presented in part at the 2023 IEEE International Symposium on Information Theory \cite{MakurSinghBTL}.}
\thanks{
Anuran Makur is with the Department of Computer Science and the Elmore Family School of Electrical and Computer Engineering, Purdue University, West Lafayette, IN 47907, USA (e-mail: amakur@purdue.edu).}
\thanks{
Japneet Singh is with the Elmore Family School of Electrical and Computer Engineering, Purdue University, West Lafayette, IN 47907, USA (e-mail: sing1041@purdue.edu).}
}

\maketitle

\begin{abstract}
The Bradley-Terry-Luce (BTL) model is one of the most widely used models for ranking a collection of items or agents based on pairwise comparisons among them. Specifically, given $n$ agents, the BTL model endows each agent $i$ with a latent skill score $\alpha_i > 0$ and posits that the probability that agent $i$ is preferred over agent $j$ in a comparison is $\alpha_i/(\alpha_i + \alpha_j)$. In this work, our objective is to formulate a hypothesis test that determines whether a given pairwise comparison dataset, with $k$ comparisons per pair of agents, originates from an underlying BTL model. We formalize this testing problem in the minimax sense and define the critical threshold of the problem. We then establish upper bounds on the critical threshold for general observation graphs and highlight their dependence on two fundamental structural properties: principal ratio and edge expansion. Additionally, we derive lower bounds on the critical threshold for the special case of complete induced graphs, thereby demonstrating that the critical threshold scales as ${\Theta}((nk)^{-1/2})$ in a minimax sense in this case. In particular, our test statistic for the upper bounds is based on a new approximation we derive for the separation distance between general pairwise comparison models and the class of BTL models. To further assess the performance of our statistical test, we prove upper bounds on conditional probabilities of error. Additionally, we derive several other auxiliary results over the course of our analysis, such as bounds on principal ratios of graphs, $\ell^2$-bounds on BTL parameter estimation under model mismatch, stability of rankings under the BTL model with small model mismatch, etc. Finally, we conduct several experiments on synthetic and real-world datasets to validate some of our theoretical results. Moreover, we also propose an approach based on permutation testing to determine the threshold of our test in a data-driven manner in these experiments. 
\end{abstract}

\begin{IEEEkeywords}
Bradley-Terry-Luce model, hypothesis testing, minimax risk, spectral methods, principal ratio of Markov matrix.
\end{IEEEkeywords}

\section{Introduction}

In recent years, the availability of pairwise comparison data and its subsequent analysis has significantly increased across diverse domains. Pairwise comparison data consists of information gathered in the form of comparisons made among a given set of items or agents. Many real-world applications, including sports tournaments, consumer preference surveys, and political voting, generate data in the form of pairwise comparisons. Such datasets serve a range of purposes, such as ranking items \cite{BradleyTerry1952, Luce1959, Plackett1975, McFadden1973, Zermelo1929, Hunter2004, GuiverSnelson2009, CaronDoucet2012, NegahbanOhShah2017, Chenetal2019,  HendrickxOlshevskySaligrama2020}, analyzing team performance over time \cite{BongLiShrotriyaRinaldo2020}, studying market or sports competitiveness 
\cite{JadbabaieMakurShah2020, JadbabaieMakurShah2020journal}, and even fine-tuning large language models using reinforcement learning from human feedback \cite{InstructGPT, DPO}.

A popular modeling assumption while performing such learning and inference tasks with pairwise comparison data is to assume that the data conforms to an underlying \emph{Bradley-Terry-Luce} (BTL) model \cite{BradleyTerry1952, Luce1959, Plackett1975, McFadden1973, Zermelo1929} as a generative model for the data. 
Given $n$ items $\{1,\dots,n\}$, the BTL model assigns a latent ``skill score'' $\alpha_i > 0$ to each item $i$, representing its relative merit or utility compared to other items, and posits that the likelihood of $i$ being preferred over an item $j$ in a pairwise comparison is given by 
\begin{equation}
\P(i \text{ is preferred over } j) = \frac{\alpha_i}{\alpha_i + \alpha_j}.
\end{equation}
The BTL model is known to be a natural consequence of the assumption of \textit{independence of irrelevant alternatives} (IIA), which is widely used in economics and social choice theory \cite{Luce1959}. 
However, the IIA assumption underlying the BTL model has been questioned, and it has been shown that the BTL model may not always accurately capture real-world datasets \cite{DavidsonMarschak1959, McLaughlinLuce1965, Tversky1972}. For instance, in the case of sports tournaments, the BTL model is incapable of capturing the home advantage effect, which refers to the possible advantage that a home team may experience when playing against a visiting team. The home advantage effect has been observed in several sports, such as soccer \cite{Clarke1995} and cricket \cite{Morley2005}, and ignoring this effect can lead to biased estimates of the skill scores of teams. Additionally, real-world datasets may exhibit non-transitive behaviors that violate the BTL model. 
To address these limitations, modifications of the BTL model that incorporate the home advantage effect \cite[Chapter 10]{Agresti2002}, Thurstonian models \cite{Thurstone1927}, and other generalizations of the BTL model \cite{Hunter2004} have been explored in the literature. Furthermore, to capture more complex behaviors within pairwise comparison datasets, various other models have been introduced, such as models based on Borda scores \cite{ShahWainwright2018, HeckelShahRamchandranWainwright2019} and non-parametric stochastically transitive models \cite{Chatterjee2015, Shahetal2016, Mania2018}.

Nevertheless, primarily because of its simplicity and interpretability, the BTL framework remains one of the most widely used models. A large fraction of the associated results in the literature focuses on estimation of the skill score parameters of the BTL model. Some popular approaches include maximum likelihood estimation \cite{Zermelo1929,Hunter2004,YCOM}, rank centrality (or Markov-chain-based) methods \cite{NegahbanOhShah2012,NegahbanOhShah2017}, least-squares methods \cite{HendrickxOlshevskySaligrama2020}, and non-parametric methods \cite{Chatterjee2015,BongLiShrotriyaRinaldo2020} (also see \cite{GuiverSnelson2009,CaronDoucet2012} for Bayesian inference for BTL models). Once the parameters are estimated, they are then used for inference tasks, such as ranking items and learning skill distributions \cite{JadbabaieMakurShah2020, JadbabaieMakurShah2020journal}. As alluded to above, an inherent assumption for such statistical analysis to be meaningful in real-world scenarios is that the BTL model accurately represents the given pairwise comparison data. If the underlying data-generating process violates the assumptions made by the BTL model, then the model's predictions and guarantees cannot be trusted.  Hence, it is important to develop a systematic method to test the validity of the BTL assumption on real-world datasets. Such a method can help identify scenarios where a BTL model can provide a useful approximation to the underlying preferences or skills, thereby providing essential guidance for downstream applications. 

In this work, our broad objective is to develop and address key questions concerning hypothesis testing for the BTL model given pairwise comparison data (also see \cite{MakurSinghBTL}). Specifically, we aim to formulate and rigorously perform minimax analysis of a test that determines whether a given pairwise comparison dataset conforms to an underlying BTL model. Ideally, we would like to develop a test that does not require additional observations beyond the usual $k=O(1)$ comparisons per pair of items, which is known to be sufficient for consistent parameter estimation (which means the relative $\ell^2$-estimation error vanishes as $n \to \infty$) \cite{Chenetal2019}. 
As we will see, the critical threshold for our testing problem as well as its type \Romannum{1} and type \Romannum{2} probabilities of error converge to zero as $n\to\infty$ with $k=O(1)$ observations. Hence, in this sense, testing for BTL models can be performed under similar conditions on $k$ and $n$ as estimation of BTL model parameters. We next delineate some related literature and then present the main contributions of this work. 

\subsection{Related Literature}
This work lies at the confluence of two fields of study: preference learning and hypothesis testing. On the preference learning front, the problem of analyzing ranked preference data has become increasingly important with the rise in the availability of data on consumer preferences and web surveys. In a complementary vein, the analysis of preference data, such as pairwise comparisons, has a rich history starting from the seminal works \cite{BradleyTerry1952, Plackett1975, Luce1959, McFadden1973, Thurstone1927, Zermelo1929}. Among the various models proposed for ranking and analyzing pairwise comparison data, the BTL model \cite{BradleyTerry1952} has emerged as one of the most widely used and well-studied approaches. Initially introduced in \cite{Zermelo1929} for estimating participants' skill levels in chess tournaments, the BTL model is a special case of the more general Plackett-Luce model \cite{Luce1959,Plackett1975}, originally developed in social choice theory, mathematical psychology, and statistics contexts. The BTL model finds applications in diverse domains, such as sports \cite{Cattelan2012,cattelan2013dynamic}, psychology \cite{Matthews2018pain}, ranking of journals \cite{Stigler1994}, fine-tuning large language models \cite{InstructGPT}, etc. We refer the readers to \cite{Diaconis1988, Shah2019} for a comprehensive overview of different models of rankings. 

In the literature, many studies have focused on estimating parameters of the BTL model and characterizing the corresponding error bounds, cf. \cite{Yellott1977, SimonsYao1999, Hunter2004, NegahbanOhShah2012, NegahbanOhShah2017, Negahbanetal2018, Shahetal2016, Chenetal2019}. For example, \cite{NegahbanOhShah2017} introduces and analyzes spectral estimation methods, \cite{Chenetal2019} presents non-asymptotic bounds for relative $\ell^\infty$-estimation errors of normalized vectors of skill scores for both spectral and maximum likelihood estimators, and \cite{Shahetal2016} provides graph-dependent $\ell^2$-estimation error bounds for the maximum likelihood estimator. 
Another emerging line of work includes uncertainty quantification for the estimated parameters, cf. \cite{UncertaintyQuantificationBTL, LagrangianInference,UncertaintyQuantificationMultiway,UncertaintyQuantificationCovariates}. For example, \cite{UncertaintyQuantificationBTL} proves the asymptotic normality of the maximum likelihood estimator as well as the spectral estimator for BTL skill parameters. In a different vein, \cite{JadbabaieMakurShah2020,JadbabaieMakurShah2020journal} develop a technique for estimating the underlying skill distribution of agents participating in a tournament, introducing a Bayesian flavor to the BTL estimation problem.  
\par
On the other hand, hypothesis testing also has a rich history in statistics, ranging from Pearson's $\chi^2$-test \cite{Lehmann2006} to non-parametric tests \cite{GrettonBorgwardtRaschScholkopf2012}. The classical literature develops some tests for the IIA assumption, e.g., \cite{HausmanMcFadden, McfaddenTyeTrain, ShevFujiiHsiehMcCowan2014}, but these tests are based on asymptotic $\chi^2$-approximations. Moreover, testing for BTL models has been explored far less than the estimation question in recent times. And to the best of our knowledge, no study has developed rigorous hypothesis tests with minimax analysis to determine the validity of the BTL assumption in the literature. The minimax perspective of hypothesis testing, which we specialize in our setting, was initially proposed by \cite{Ingster1994}. Among related works to ours, \cite{RastogiBalakrishnanShahAarti2022} analyzed two-sample testing on pairwise comparison data, and \cite{Seshadri2020} derived lower bounds for testing the IIA assumption given preference data. The formulations in \cite{RastogiBalakrishnanShahAarti2022} and \cite{Seshadri2020} are different to our formulation here and in \cite{MakurSinghBTL}. For example, we quantify the separation distance of a pairwise comparison model from the class of BTL models in terms of the Frobenius norm rather than the variant of total variation (TV) distance used in \cite{Seshadri2020}. For the special case of pairwise comparisons under a complete induced observation graph, the lower bounds in \cite{Seshadri2020} agree with ours in terms of the high-level scaling law of the critical threshold with respect to $n$ and $k$. However, their study does not focus on providing statistical tests and corresponding upper bounds as we do in this work. Finally, in the alternative setting of online preference learning, \cite{Haddenhorst2021} studied the testing problem for various forms of stochastic transitivity and established upper and lower bounds on the sample complexity for testing. 

In this work, we also briefly explore the stability of rankings made under a BTL assumption when the underlying model does not conform to this assumption. Indeed, since the BTL assumption may not perfectly hold in various scenarios, it is crucial to examine the stability of rankings obtained under the BTL assumption. In the literature, \cite{ShahWainwright2018} provided empirical evidence that the BTL assumption is not very robust to changes in the pairwise comparison matrix, and \cite{cucuringu2016sync} proposed the Sync-Rank algorithm, which treated the ranking problem as a group synchronization problem, thereby eliminating the dependence of the algorithm on the BTL assumption.
Furthermore, recent work by \cite{datar2022byzantine} demonstrated the ineffectiveness of standard spectral methods for BTL parameter estimation in the presence of even small fractions of Byzantine voters, and proposed a more robust Byzantine spectral ranking algorithm. In contrast to these works, we introduce and utilize a new measure of separation distance that quantifies deviation from the class of BTL models. As we will see, this measure is instrumental in our analysis to quantify the stability of BTL assumptions in the context of rankings with respect to classical Borda count rankings.

\subsection{Main Contributions}\label{sec:Main Contributions}
In contrast to the majority of the recent literature that focuses on estimation of BTL parameters, our work takes a different approach by developing a systematic method for testing whether a given pairwise comparison dataset conforms to the BTL assumption. Our main contributions are summarized below:
\begin{enumerate}
    \item We devise a notion of ``separation distance'' that allows us to easily quantify the deviation of a general pairwise comparison model from the class of all BTL models. Under some regularity conditions, we show in \cref{Distance to closest BTL model} that this measure is always within constant factors of the Frobenius norm distance between a general pairwise comparison model and the class of BTL models. We then formulate our hypothesis testing problem for BTL models in a minimax sense using this separation distance and also introduce a test statistic in \cref{BoxedTest} based on it. 
    \item We define the critical threshold for our testing problem and establish an upper bound on it in \cref{Main theorem with upper bound} for general induced observation graphs (satisfying mild assumptions). Furthermore, we also derive upper bounds on the type \Romannum{1} and type \Romannum{2} probabilities of error in \cref{thm:Type-1 and Type-2 error}. These bounds provide insights into the influence of various parameters on the error decay rate.
    \item We also provide an information-theoretic lower bound on the critical threshold for complete induced observation graphs in \cref{thm:Lower Bound}, thereby demonstrating the minimax optimal scaling of the critical threshold (up to constant factors). 
    \item Our test requires that the underlying observation graph and the pairwise comparison matrix satisfy certain regularity assumptions, e.g., expansion {and bounded principal ratio}. To substantiate this, we identify different classes of graphs in \cref{sec: Graphs with Bounded Principal Ratio} (see \cref{lem: Height of Perron–Frobenius Eigenvector,lem: Height of Perron–Frobenius Eigenvector for expanders,lem: Height of Perron–Frobenius Eigenvector under sampling}) that fulfill these criteria for all pairwise comparison models.
    \item We also prove several auxiliary results. 
    For example, during our analysis, we obtain $\ell^2$-estimation error bounds for (virtual) BTL parameters in \cref{lem:L2error} when the data is actually generated by a general pairwise comparison model as opposed to a BTL model. These bounds also highlight the robustness of the spectral ranking method under model mismatch. As another example, we utilize the notion of separation distance mentioned above to analyze the stability of BTL assumptions in the context of rankings in \cref{sec: Stability of the BTL Assumption}. More specifically, we investigate the deviation from the BTL condition that is sufficient for the ranking produced under the BTL assumption to differ from the classical Borda count ranking \cite{borda1781}. Our results in \cref{Prop: Stability of BTL Model} show that a deviation of $O(1/\sqrt{n})$ from the BTL condition is sufficient to produce inconsistent BTL and Borda rankings. 
    \item Finally, we perform several experiments on synthetic and real-world datasets to validate some of our theoretical results. Additionally, we propose a non-parametric approach based on permutation testing to determine the non-explicit threshold of our hypothesis test in a data-driven manner in experiments. 
\end{enumerate}

\subsection{Outline}
The paper is organized as follows. Several notational preliminaries are presented in  \cref{Notational Preliminaries}. In \cref{Formal Model and Goal}, we introduce general pairwise comparison models, define the necessary terminology, and present the various regularity assumptions necessary for our analysis. In \cref{sec: Main Results}, we present the main results of our work. Specifically, in \cref{sec:Problem Formulation}, we introduce a notion of separation distance that allows us to measure the deviation of a pairwise comparison model from the class of all BTL models. Additionally, we mathematically formulate the hypothesis testing problem and introduce the associated testing rule. \cref{Upper Bound on Critical Threshold} provides an upper bound on the scaling of the critical threshold for this hypothesis test. In \cref{sec: Type1 and Type2}, we provide upper bounds on type \Romannum{1} and type \Romannum{2} probabilities of error. Moreover, \cref{Lower Bound on Critical Threshold} introduces a matching lower bound on the critical threshold of the test proving minimax optimality (up to constant factors). In \cref{sec: Graphs with Bounded Principal Ratio}, we provide examples of different classes of graphs that meet the required assumptions for all pairwise comparison models. In our paper, formal proofs of various propositions are available in \cref{sec:proofs}, while the upper bound proofs for the critical threshold (\cref{Main theorem with upper bound}) and type \Romannum{1} and type \Romannum{2} probabilities of error (\cref{thm:Type-1 and Type-2 error}) are provided in \cref{Proof:Main theorem with upper bound}. The proof of our lower bound is presented in \cref{sec: Proof of Lower Bound and Stability}. In \cref{sec: Numerical Simulations}, we present numerical simulations for synthetic data that support our theoretical results and also apply our testing rule to both real-world and synthetic datasets. Finally, in \cref{sec: Conclusion}, we reiterate our results and provide some directions for future research. Additionally, we present our remaining proofs and characterize the stability of the BTL model in the Appendices. 

\subsection{Notational Preliminaries} \label{Notational Preliminaries}
        We briefly collect some notation here that is used throughout this work. We let $\R = (-\infty,\infty)$, $\R_+ = [0,\infty)$, and $\N \triangleq \{1,2,3,\dots\}$ denote the sets of real, positive real, and natural numbers, respectively. For any $n \in \N$, we let $\R^n$ and $\R^{n \times n}$ be the sets of all $n$-length vectors and $n \times n$ matrices, respectively, $\1_n \in \R^n$ be the column vector with all entries equal to $1$,  and $[n] \triangleq \{1,\dots,n\}$. For any (index) set $\calS \subseteq [n]$, we let $\1_\calS \in  \R^n$ denote a vector in which the entries are $1$ for elements in the set $\calS$ and $0$ otherwise. 
        Moreover, we let $\calP_n$ denote the $(n-1)$-dimensional probability simplex in $\R^n$, and $\I_{\mathcal{A} }$ denote the indicator function on the set $\mathcal{A}$. Next, for any vector $x \in \R^n$, $\| x\|_2$ and $\|x\|_\infty$ denote its $\ell^2$- and $\ell^\infty$-norms, $x^{p}$ denotes the entrywise $p$th power of $x$, i.e., $x^{p} = [x_1^p , x_2^p,\dots,x_n^p]^\T$, and $\diag(x) \in \R^{n \times n}$ is the diagonal matrix with $x$ along its principal diagonal. For any matrix $A \in \R^{n \times n}$, $\| A\|_2$, $\|A\|_\F$, and $\text{Tr}(A)$ denote the standard spectral norm, Frobenius norm, and trace of $A$, respectively, $A_{ij}$ denotes the $(i,j)$th element of $A$, $A_{:,i}$ denotes the $i$th column of $A$, and $A_{j,:}$ denotes the transpose of the $j$th row of $A$ (i.e., in column vector form). Let $\sigma_1(A) \geq \sigma_2(A) \geq \dots \geq \sigma_n(A)$ represent the (ordered) singular values of $A \in \R^{n \times n}$.
        For an entrywise strictly positive vector $\pi \in \R^n$, we define a Hilbert space $\ell^2(\pi)$ on $\R^n$ with inner product $\ip{x}{y}_{\pi} = \sum_{i=1}^n \pi_i x_i y_i$ and corresponding vector, operator, and weighted Frobenius norms $\|x\|_\pi = \sqrt{\ip{x}{x}_\pi}$, $\|A\|_\pi = \sup_{\|x\|_\pi = 1} \|x^\T A \|_{\pi} $, and $\|A\|_{\pi,\F} = \big(\sum_{i}\sum_j \pi_j A^2_{ij}\big)^{1/2}$, respectively. Finally, we utilize standard Bachmann-Landau asymptotic notation, e.g., $O(\cdot)$, where it is assumed that the parameter $n \rightarrow \infty$ or the parameter $k \rightarrow \infty$, where $n$ represents the number of items and $k$ is the number of pairwise comparisons per pair of items. Throughout this paper, we use ``high probability'' to refer to a probability of at least  $1 - 1/\text{poly} (n)$, where  $\text{poly} (n)$ captures a polynomial function of $n$.  

\section{Formal Setup and Goal}
 \label{Formal Model and Goal}
 
 We begin by introducing a general pairwise comparison model. Consider a scenario where a set of $n$ agents, indexed by $[n]$ with $n \in \N\backslash\!\{1\}$, engage in a tournament consisting of several observed pairwise comparisons. This scenario is ubiquitous in many application domains, such as sports tournaments, discrete choice models in economics, etc. For example, in a sports tournament, $[n]$ represents the teams or players that play pairwise games with each other. Similarly, in the discrete choice models in economics, $[n]$ represents the available alternatives that an individual may choose from. 

    Within this context, we assume that each observation corresponds to a pairwise comparison between agents $i$ and $j$ for distinct $i,j \in [n]$. We focus on a general asymmetric setting, where an ``$i$ vs. $j$'' comparison can have a different interpretation to a ``$j$ vs. $i$'' comparison. 
    For example, in an ``$i$ vs. $j$'' comparison, $i$ and $j$ may represent the home team and the away team, respectively. This type of asymmetry is common in sports like cricket, soccer, etc. \cite{Morley2005}, where it is often observed that the home team has a certain advantage over the away team. Nevertheless, it is worth mentioning that many of our ensuing theorems, insights, and other results can be extended with some effort to the symmetric setting, where  ``$i$ vs. $j$'' and  ``$j$ vs. $i$'' comparisons are considered equivalent. 
    
    In general, comparisons between all pairs $i, j \in [n]$ may not be observed. To model this, we assume that we are given an induced observation graph $\G = ([n], \calE)$ with vertex set $[n]$ and edge set $\calE$, where an edge $(i, j) \in \calE$ (with $i \neq j$) exists if and only if comparisons of the form ``$i$ vs. $j$'' are observed. Additionally, for convenience, self-loops are included in $\calE$ for all vertices (i.e., $(i, i) \in \calE$ for all $i \in [n]$). Let $E \in \{0,1\}^{n \times n}$ be the adjacency matrix of the graph $\G$ with $E_{ij} = 1$ if $(i, j) \in \calE$ and $0$ otherwise. We also define the projection operator $\PE(X) \triangleq X \odot E$ for $X \in \R^{n \times n}$, where $\odot$ denotes the Hadamard product. Moreover, we assume that the edge set $\calE$ is symmetric (i.e., the graph is actually undirected); hence, if  ``$i$ vs. $j$'' comparisons are observed, then we require that  ``$j$ vs. $i$'' comparisons are observed as well. Finally, we assume that the graph is connected and is given a priori, i.e., it does not depend on the outcomes of observed pairwise comparisons. (Connectedness of the graph is a required, standard assumption in the literature, cf. \cite{Ford1957,Hunter2004}.) In the sequel, we will further specify several classes of graphs where we can prove theoretical guarantees on hypothesis testing for the BTL model.
    
    The next definition presents the general pairwise comparison model. Note that in the literature, various well-known probabilistic models have been developed to capture pairwise comparison settings, such as the BTL model \cite{BradleyTerry1952, Luce1959, McFadden1973}, Thurstonian model \cite{Thurstone1927}, non-parametric models \cite{Chatterjee2015, Shahetal2016}, etc. These are all specializations of the general model below.
    
    \begin{definition}[Pairwise Comparison Model] \label{Def: Pairwise Comparison Model} For any pair of distinct agents $i,j \in [n]$, let $p_{ij} \in (0,1)$ denote the probability that agent $j$ beats agent $i$ in a  ``$i$ vs. $j$''  pairwise comparison. We refer to the collection of parameters $\{p_{ij}: (i, j) \in \calE, \, i \neq j\}$ as a \emph{pairwise comparison model} (on $\calE$).
        \end{definition}
    
    Furthermore, a pairwise comparison model can be aptly summarized by a \emph{pairwise comparison matrix} $P \in \R^{n \times n}$ with
    \begin{equation}
        P_{ij} \triangleq  \begin{cases}  p_{ij} , & i \neq j \text{ and } (i,j) \in \calE\\
                                \frac{1}{2} , & i = j \\ 
                                0, & \text{otherwise }
                  \end{cases}, 
    \end{equation} 
    where we have set $ P_{ii} = \frac{1}{2}$ for notational convenience.
    In our analysis, we find it convenient to assign a time-homogeneous Markov chain (or a row stochastic matrix) on the finite state space $[n]$ to any pairwise comparison model. We establish this canonical assignment as follows. 
    \begin{definition}[Canonical Markov Matrix] \label{Def: Canonical Markov Matrix}  
        For any pairwise comparison model $\{p_{ij} \in (0,1) : (i,j) \in \calE, \, i\neq j \}$ with pairwise comparison matrix $P \in \R^{n \times n}$, its \emph{canonical Markov matrix} is the row stochastic matrix $S \in \R^{n \times n}$, where 
        \begin{equation}    \label{eq:  Canonical Markov Matrix}
            S_{ij} \triangleq \begin{cases} \displaystyle{\frac{p_{ij}}{ d }}, & i \neq j  {\text{ and }}  (i, j) \in \calE \\
                \displaystyle{1 - \frac{1}{d }\sum_{\substack{  u \in [n]\backslash\!\{i\}:\\ (i,u)  \in \calE} }{p_{iu}}} , & i = j \\ 
                0, & \text{otherwise}
                    \end{cases} ,
        \end{equation}
    and $d$ is a sufficiently large constant such that the Markov matrix $S$ is at least $\big(\frac{1}{2}\big)$-lazy, i.e., $S_{ii} \geq \frac{1}{2}$ for all $i \in [n]$.
    \end{definition}
     For fixed induced graphs, we set the parameter $d = 2 d_{\max},$ where $d_{\max } \triangleq \max_{i \in [n]} \sum_{j=1}^n E_{ij}$ is the maximum degree of a node in $\G$. {Further discussion regarding this parameter can be found} in \cref{sec: Graphs with Bounded Principal Ratio}. 
     {We note that due to the strict positivity of $p_{ij}$'s and connectedness of $\G$, $S$ possesses a unique, entrywise strictly positive, stationary distribution.} As noted earlier, the most well-known specialization of the pairwise comparison model in \cref{Def: Pairwise Comparison Model} is the BTL model defined below \cite{BradleyTerry1952, Luce1959, McFadden1973}.

    \begin{definition}[BTL Model {\cite{BradleyTerry1952,Luce1959,McFadden1973}}] \label{Def: BTL Model}
        A pairwise comparison model $\{ p_{ij} \in (0,1): (i,j) \in \calE, i \neq j\}$ on $\calE$ is known as a BTL (or multinomial logit) model if there exist skill score parameters $\alpha_i > 0$ for every agent $i \in [n]$ such that:
        \begin{equation*}
            \forall (i,j) \in \calE \textup{ with }  i \neq j, \, \enspace p_{ij} = \frac{\alpha_j}{\alpha_i + \alpha_j} . 
        \end{equation*}
    Hence, we can describe a BTL model entirely using the collection of its $n$ skill score parameters $\{\alpha_i: i \in [n]\}$.
    \end{definition}

    Next, we explain the data generation process. Fix any pairwise comparison model $\{p_{ij} \in (0,1): (i,j) \in \calE, \, i \neq j\}$. For any pair $i \neq j$, define the outcome of the $m$th 
     ``$i$ vs. $j$'' pairwise comparison between them as a Bernoulli random variable
    \begin{equation}
        Z_{m,{ij}} \triangleq \begin{cases} 1 , & \text{ if } j \text{ beats } i \ (\text{with probability } p_{ij}) \\
                                            0 , & \text{ if } i \text{ beats } j \ (\text{with probability } 1-p_{ij})
                              \end{cases},
    \end{equation}
    for $m \in [k_{ij}]$, where $k_{ij}$ denotes the number of observed  ``$i$ vs. $j$'' comparisons. We assume throughout that the observation random variables $\Z \triangleq \{Z_{m,{ij}} : (i,j) \in \calE, \, i \neq j, \, m \in [k_{ij}]\}$ are mutually independent.
    Moreover, let $Z_{ij}  \triangleq \sum_{m=1}^{k_{ij}}Z_{m,{ij}} $. Clearly, it follows that for any $(i,j) \in \calE$ with $i \neq j$, $Z_{ij}$ is a binomial random variable, i.e., $Z_{ij} \sim \text{Bin}(k_{ij},p_{ij})$, and for simplicity, we set $Z_{ii} = k_{ii} = 0$ for all $i \in [n]$. 

    \subsection{Main Goal} 
    \label{subsec: Main Goal}
    
    Given the observations $\Z$ of a tournament as defined above, our objective is to determine whether the underlying pairwise comparison model is a BTL model on the observed comparison set $\calE$. This corresponds to solving a \emph{composite hypothesis testing} problem:
        \begin{equation}
            \begin{aligned}
                H_0 & : \Z \sim \text{BTL model for some $\alpha_1,\dots,\alpha_n > 0$}, \\
                H_1 & : \Z \sim \text{pairwise comparison model that is not BTL},
            \end{aligned}\label{eq:CompositeHypothesisTest}
        \end{equation}
    where the null hypothesis $H_0$ states that $\Z$ is distributed according to a BTL model (on the observed comparison set $\calE$), and the alternative hypothesis $H_1$ states that $\Z$ is distributed according to a general non-BTL pairwise comparison model. 

    Note that in the settings considered in this work, we do not necessarily have pairwise comparisons for all pairs. It is straightforward to see that it is not possible to test whether the unobserved pairs adhere to a BTL model or conform to some general pairwise comparison model. (Indeed, for any unobserved pair $(i,j)$, it is not possible to distinguish between $p_{ij} = \alpha_{j}/(\alpha_i+\alpha_j)$ and $p_{ij} + \Delta$ for any $\Delta\neq 0$.) Therefore, our focus is solely on \emph{testing whether given pairwise comparison data adheres to a BTL model on the set of observed comparisons $\calE$}. 

   To pose the hypothesis testing problem in \cref{eq:CompositeHypothesisTest} more rigorously, we demonstrate an interesting relation between a BTL model and its canonical Markov matrix. Recall that a Markov chain on the state space $[n]$, defined by the row stochastic matrix $W \in \R^{n \times n}$, is said to be \emph{reversible} if it satisfies the \emph{detailed balance conditions} \cite[Section 1.6]{LevinPeresWilmer2009}:
    \begin{equation}
    \forall i,j \in [n], \, i \neq j, \enspace \pi_i W_{ij} = \pi_j W_{ji} ,
    \end{equation} 
    where $W_{ij}$ denotes the probability of transitioning from state $i$ to state $j$, and $\pi = (\pi_1,\dots,\pi_n)$ denotes the stationary (or invariant) distribution of the Markov chain (which always exists). Equivalently, the Markov chain $W$ is reversible if and only if
    \begin{equation}
    \diag(\pi) W = W^{\T} \diag(\pi) . 
    \end{equation} 
    It turns out that there is a tight connection between reversible Markov chains and the BTL model. This is elucidated in the ensuing proposition, which is a slightly more general version of \cite[Lemma 6]{RajkumarAgarwal2014}, \cite[Section 2.2]{NegahbanOhShah2017}.

    \begin{proposition}[BTL Model and Reversibility]
    \label{Prop: BTL Model and Reversibility} 
    For a symmetric comparison set $\calE$, a pairwise comparison model $\{p_{ij} \in (0,1): (i,j) \in \calE, \, i \neq j\}$ is a BTL model if and only if its canonical Markov matrix $S \in \R^{n \times n}$ is reversible and satisfies the translated skew-symmetry condition $p_{ij} + p_{ji} =  1$ for all $(i,j) \in \calE$.
    \end{proposition}

    The proof is provided in \cref{Proof of Prop: BTL Model and Reversibility}. The proof relies on the fact that for a BTL model (with a symmetric edge set $\calE$), we can exactly compute the stationary distribution of its canonical Markov matrix in closed form. In cases where the graph lacks symmetry, the stationary distribution of the canonical Markov matrix may depend on the underlying graph topology. Furthermore, the assumption of a symmetric induced graph made earlier is in some sense necessary because, e.g., if the induced graph is not symmetric for some pair $(i,j)$, then we cannot check translated skew-symmetry for $(i,j)$.

    \subsection{Expansion and Principal Ratio} 
    To carry out our analysis, we need a few more assumptions on the underlying graph $\G$ and the probabilities $p_{ij}$ of the pairwise comparison model. Broadly, these assumptions ensure that the canonical Markov matrix $S$ possesses sufficient connectivity, quantified by edge expansion, and a bounded principal ratio. We first recall the notion of edge expansion of non-negative matrices, which characterizes the extremal ``connectivity'' or ``flow'' from one subset of vertices to another under stationarity.   
    \begin{definition}[Edge Expansion of Non-negative Matrices \cite{MehtaSchulman}]
    Consider any non-negative matrix $R \in \mathbb{R}^{n \times n}$ with a Perron-Frobenius eigenvalue of $1$ and corresponding left and right eigenvectors $u$ and $v$, which we assume are entrywise strictly positive and normalized such that $u^\T v =1$. Let $D_u \triangleq \diag(u)$ and $D_v \triangleq \diag(v)$ be diagonal matrices with $u$ and $v$ on the principal diagonals, respectively. Then, the \emph{edge expansion} of $R$ is defined as
    \begin{align}
        \phi(R) & \triangleq  \min _{\calS \subseteq[n]} \frac{ \mathbf{1}_{\calS}^\T D_{u} R D_{v} \mathbf{1}_{\calS^\complement } }{\min \{ \mathbf{1}_{\calS}^\T D_{u} D_{v} \mathbf{1}_n,\mathbf{1}_{\calS^\complement}^\T D_{u} D_{v} \mathbf{1}_n \}} \nonumber \\ 
        & = \min _{\calS \subseteq[n] } \frac{\sum_{i \in \calS, j \in \calS^\complement } R_{i j}  u_{i}  v_{j}}{\min\{\sum_{i \in \calS} u_{i}  v_{i}, \sum_{i \in \calS^\complement } u_{i}  v_{i}\} }.
        \label{original expression for expansion}   
    \end{align}
    \end{definition}
    Note that in cases where there are no strictly positive left and right eigenvectors $u$ and $v$ corresponding to the eigenvalue $1$, the edge expansion $\phi(R)$ is defined to be $0$. 
    For the row-stochastic matrix $S$ with stationary distribution $\pi$, its edge expansion simplifies to
    \begin{equation}
        \phi(S) = \min_{\calS \subseteq [n]} \frac{\sum_{i \in \calS, j \in \calS^\complement } p_{ij} \pi_i \mathbbm{1}_{(i,j) \in \calE } }{ d \min\{\sum_{i \in \calS} \pi_i, \sum_{i \in \calS^\complement } \pi_i \} }. \label{exact expression for edge expansion of DTM}
    \end{equation}
    We introduce the following assumption on the edge expansion of $S$. 
    \begin{assumption}[Edge Expansion] 
    \label{assm:Large Expansion of DTM}
    For the canonical Markov matrix $S$ corresponding to the pairwise comparison model $P$, we assume that there exists a constant $\xi > 0$ such that
    \begin{equation}
            \phi(S) \geq \xi. 
    \end{equation}
    \end{assumption}

    As mentioned earlier, in addition to edge expansion, we require the canonical Markov matrix $S$ to have a bound on its \emph{principal ratio}, cf. \cite{EigenvectorsIrregularGraphs} (which is the \textit{Birkhoff norm of the Perron-Frobenius vector} \cite{BirkhoffContractionCoefficient}, and is sometimes called the \textit{height of the Perron-Frobenius vector} \cite{Minc1970}). 
   \begin{definition}[Principal Ratio {\cite{EigenvectorsIrregularGraphs}}]
   Given the stationary distribution $\pi$ of a canonical Markov matrix $S$, its \textit{principal ratio}, denoted $h_\pi$, is defined as the ratio of its maximum and minimum entries:
      \begin{equation}
            h_\pi \triangleq \frac{\max_{i \in [n]} \pi_i}{\min_{j \in [n]} \pi_j}. \label{def: principal ratio}
        \end{equation}
    \end{definition} 

    We note that the principal ratio is both a function of the underlying graph $\G$ and the pairwise comparison probabilities $p_{ij}$. Intuitively, it captures a notion of ``regularity'' of Markov matrices like $S$ (akin to the regularity of vanilla graphs like $\G$). Indeed, a principal ratio closer to $1$ signifies a ``more regular'' Markov matrix $S$; for instance, in random walks on graphs, the principal ratio equals $1$ if and only if the graph is regular (i.e., all nodes have identical degrees) \cite{EigenvectorsIrregularGraphs}. This ratio has a long history in matrix theory and the analysis of Markov chains \cite{birkhoff1,birkhoff2}. For example, Birkhoff's geometric ergodicity theory connects the principal ratio to the contraction coefficient in projective metrics, where smaller ratios imply faster convergence to stationarity \cite{BirkhoffContractionCoefficient}. Since the principal ratio impacts the rate of convergence of random walks on graphs, it also arises in the analysis of spectral algorithms, such as PageRank \cite{PagerankPrincipalRatio}, and matrix analysis \cite{ostrowski1952}. Furthermore, several combinatorial studies have derived bounds on the principal ratio under various settings \cite{EigenvectorsIrregularGraphs,AksoyChung2016}.

  To proceed, we make the following assumption on the principal ratio of $S$.
    \begin{assumption}[Bounded Principal Ratio] \label{assm: Bounded Principal Ratio} There exists a constant $h > 0$ such that the principal ratio $h_\pi$ is bounded as:
    \begin{equation}
        h_\pi \leq h. 
    \end{equation}
    \end{assumption}   
    We remark that we only require \cref{assm:Large Expansion of DTM,assm: Bounded Principal Ratio} in the analysis of our hypothesis testing framework.\footnote{We generally think of $\xi$ and $h$ in \cref{assm:Large Expansion of DTM,assm: Bounded Principal Ratio} as constants that do not depend on $n$. However, our results illustrate precise dependencies with respect to these constants in cases where they depend on other problem parameters.} 
     
    In the sequel, we introduce a more palatable simplifying assumption that will be further discussed. This assumption is standard in the analysis of pairwise comparison models under consideration (see, e.g., \cite{NegahbanOhShah2017,Chenetal2019,JadbabaieMakurShah2020journal}). 
    \begin{assumption}[Dynamic Range] \label{Pijbounded}
        We assume that there exists a constant $\delta \in (0,1)$ such that for all $(i,j) \in \calE$,
        \begin{equation}
            \frac{\delta}{1+\delta} \leq p_{ij} \leq \frac{1}{1+\delta} . 
        \end{equation}
    \end{assumption}
    We will highlight that \cref{Pijbounded} implies \cref{assm:Large Expansion of DTM,assm: Bounded Principal Ratio} for various well-known classes of induced graphs.
    For now, note that coupled with \cref{Pijbounded,assm: Bounded Principal Ratio}, \Cref{assm:Large Expansion of DTM} can be expressed in terms of the canonical notion of edge expansion of $\G$. Indeed, observe that
   \begin{align}
        \phi(S) & = \min_{\calS \subseteq [n] } \frac{\sum_{i \in \calS, j \in \calS^\complement } p_{ij} \pi_i \mathbbm{1}_{(i,j) \in \calE } }{ d \min\{\sum_{i \in \calS} \pi_i, \sum_{i \in \calS^\complement } \pi_i \} } \nonumber \\
         & \geq \frac{\delta }{d h_{\pi} (1+\delta) } \min_{\calS \subseteq [n] } \frac{|\calE(\calS,\calS^\complement)|}{\min\{|\calS|, |\calS^\complement|\}} \nonumber\\ 
         & \geq \frac{\delta  }{d h (1+\delta) } \tilde{\phi}(\G) , \label{trivial lower bound of expansion}
    \end{align} 
    where $\calE(\calS,\calS^\complement) \triangleq \{(i,j) \in \calE : i \in \calS, \, j \in \calS^\complement\}$ is the set of edges in $\calE$ from set $\calS$ to $\calS^\complement$, and $\tilde{\phi}(\G)$ is the usual definition of edge expansion of $\G$ \cite{AlonChung1989}: 
    \begin{equation}
        \tilde{\phi}(\G) \triangleq \min_{\calS \subseteq [n]} \frac{|\calE(\calS,\calS^\complement) |}{ \min\{|\calS|, |\calS^\complement| \}}.
    \end{equation}

    Finally, we emphasize that all BTL models with a bounded condition number, 
    \begin{equation}
    \label{Eq: bound on cond num}
       \frac{ \max_{i\in [n]} \alpha_i }{ \min_{j \in [n]} \alpha_j } \leq \frac{1}{\delta} ,
    \end{equation}
    automatically satisfy both \cref{Pijbounded,assm: Bounded Principal Ratio} (with $h = \delta^{-1}$).\footnote{We refer readers to \cref{Proof of Prop: BTL Model and Reversibility} for an explicit expression for the stationary distribution in BTL models.} Furthermore, if the underlying graph $\G$ exhibits significant edge expansion in the canonical sense, i.e., $\tilde{\phi}(\G) \geq \tilde{\epsilon} d$ for some constant $\tilde{\epsilon} > 0$, then \cref{assm:Large Expansion of DTM} is satisfied with $\xi = O(1)$ for all BTL models with bounded condition number. For general pairwise comparison models, we will present several classes of graphs in \cref{sec: Graphs with Bounded Principal Ratio}, e.g., complete graphs, dense regular graphs, and Erd\H{o}s-R\'enyi random graphs, where any pairwise comparison model satisfying \cref{Pijbounded} will also satisfy \cref{assm:Large Expansion of DTM,assm: Bounded Principal Ratio}  with constant $h$ and $\xi$  (with high probability, where appropriate).

    \section{Main Results and Discussion} \label{sec: Main Results}
    In this section, we present our main contributions and results as outlined in \cref{sec:Main Contributions}. 

\subsection{Minimax Formulation and Decision Rule} \label{sec:Problem Formulation}
    We begin by rigorously formalizing the hypothesis testing problem in \cref{subsec: Main Goal}. To this end, we first define a separation distance between a general pairwise comparison matrix $P$ and its closest BTL model, and then define minimax risk and introduce our decision rule.

    Recall that by \cref{Prop: BTL Model and Reversibility}, any pairwise comparison matrix $P$ is BTL if and only if it satisfies the reversibility condition $\Pi P = P^\T \Pi$ and translated skew-symmetry $P + P^\T = \PE(\1_n\1_n^\T)$, where $\Pi = \diag(\pi)$ and $\pi$ is the stationary distribution of the canonical Markov matrix $S$ corresponding to $P$. It turns out that both conditions are elegantly captured by the matrix $\Pi P + P \Pi - \PE(\1_n \pi^\T)$ as illustrated in \cref{EasyProp}. 

    \begin{proposition}[BTL Model Characterization] \label{EasyProp}
        For a symmetric edge set $\calE$, the pairwise comparison matrix $P$ in \cref{Def: Pairwise Comparison Model} corresponds to a BTL model (on the set $\calE$) if and only if $\Pi P + P \Pi = \PE(\1_n \pi^\T)$.
    \end{proposition}

    \cref{EasyProp} is proved in \cref{EasyPropProof}. It suggests that we can use the Frobenius norm of $\Pi P + P \Pi - \PE(\1_n \pi^\T)$ to quantify the deviation of a pairwise comparison matrix from the family of BTL models. To rigorously argue that the (scaled) Frobenius norm of $\Pi P + P \Pi - \PE(\1_n \pi^\T)$ coincides with the usual measure of separation distance in this setting, we require a useful decomposition of weighted Frobenius norm given in \cref{OrthogonalDecomposition}. 

    \begin{proposition}[Decomposition of Weighted Frobenius Norm]
    \label{OrthogonalDecomposition}
    For any pairwise comparison matrix $P \in \R^{n \times n}$ and any vector $\pi \in \R^n$ with strictly positive entries, we have
    \begin{align*}
        \| \Pi P + P \Pi - \PE(\mathbf{1}_n\pi^\T) & \|^2_{\pi^{-1},\F} = \| \Pi P - P^\T \Pi  \|^2_{\pi^{-1},\F}  \\ 
        & + \|  P + P^\T - \PE(\mathbf{1}_n\mathbf{1}_n^\T) \|^2_{\pi,\F} . 
    \end{align*}
    \end{proposition}

    The proof of \cref{OrthogonalDecomposition} is provided in \Cref{OrthogonalDecompositionProof}. (We note that since $\pi_i > 0$ for all $i \in [n]$, the norm $\|\cdot\|_{\pi^{-1},\F}$ is well-defined.) 
    Using \cref{OrthogonalDecomposition}, we show in \cref{Distance to closest BTL model} that if a pairwise comparison model $P$ satisfies \cref{assm:Large Expansion of DTM,assm: Bounded Principal Ratio} {with $h, \xi =\Theta(1)$}, then the quantity $\| \Pi P + P \Pi - \PE(\1_n \pi^{\T}) \|_\F/\|\pi\|_\infty$ is always within constant factors of the Frobenius-norm-distance between $P$ and the set of BTL models (or more precisely, the set $\BTLh$ defined below). Hence, $\| \Pi P + P \Pi - \PE(\1_n \pi^{\T}) \|_\F/\|\pi\|_\infty$ captures a natural notion of separation distance in this setting.  
    \begin{theorem}[Distance to Closest BTL Model] \label{Distance to closest BTL model}
    Suppose the pairwise comparison matrix $P$ and the induced graph $\G$ satisfy \cref{assm:Large Expansion of DTM,assm: Bounded Principal Ratio}. Then, there exist constants $c_1, c_2 > 0$ (independent of $n$) such that 
    $$
    \begin{aligned}
    c_1\frac{\xi^2}{h^{3} } & \frac{ \|\Pi P + P\Pi - \PE( \1_n \pi^\T)\|_\F}{\|\pi\|_\infty}  \leq  \min_{B \in  \BTLh} \|P - \PE(B) \|_\F  \\
    & \ \ \ \ \ \ \ \ \ \ \ \ \ \ \ \ \ \ \ \ \ \  \leq  c_2 h  \frac{\|\Pi P + P\Pi - \PE(\1_n \pi^\T)\|_\F}{\|\pi\|_\infty} ,
    \end{aligned}
    $$
    where $\BTLh$ is the set of all pairwise comparison matrices $B$ corresponding to BTL models whose skill scores $\alpha \in \R^n$ satisfy $ \frac{ \max_{i\in [n]} \alpha_i }{ \min_{j \in [n]} \alpha_j } \leq h$.
    \end{theorem}

    The proof is provided in \cref{Proof of Distance to closest BTL model}.
It utilizes a new Lagrangian and perturbation-based approach to compute the Frobenius-norm-distance between a given pairwise comparison matrix and its closest reversible counterpart. This approach may be of independent utility in other matrix-theoretic scenarios. We also note that our approach of measuring separation distance is quite different to distance measures between Markov chains based on spectral radius introduced in \cite{daskalakisDikkalaGravin2018,friedWolfer2022}. Next, armed with \cref{EasyProp} and \cref{Distance to closest BTL model}, we formally define the hypothesis testing problem for BTL models. 

\subsubsection{Hypothesis Testing Problem} 

For a given tolerance parameter $\epsilon > 0$ and a pairwise comparison model $P$ satisfying \cref{assm:Large Expansion of DTM,assm: Bounded Principal Ratio}, we can formulate the hypothesis testing problem in \eqref{eq:CompositeHypothesisTest} as:   
    \begin{equation} \label{Hypothesis test}
        \begin{aligned}
        & H_0 : \Z \sim P \text{ such that} \ \ \Pi P+ P\Pi = \PE(\1_n\pi^\T) , \\
        & H_1 : \Z \sim P \text{ such that}\ \ \frac{\| \Pi P+ P\Pi - \PE(\1_n\pi^\T) \|_{\F}}{n \|\pi\|_\infty}  \geq \epsilon.
    \end{aligned}    
    \end{equation}

\subsubsection{Minimax Risk and Decision Rule} 
    Let $\Phi$ denote a \emph{decision rule} or \emph{hypothesis test} that maps the consolidated observations $\Z$ to $\{0,1\}$, where $0$ represents the null hypothesis $H_0$ and $1$ represents the alternative hypothesis $H_1$. Let $\P_{H_0}$ and $\P_{H_1}$ denote the probability distributions of the observations $\Z$ under $H_0$ and $H_1$, respectively. Throughout this paper, we will use $\P_{H_0}$ and $\P_{H_1}$ when we need to specify a hypothesis, and $\P$ when a probability expression holds for both hypotheses. Similarly, $\E_{H_0}$, $\E_{H_1}$, and $\E$ will represent the corresponding expectation operators. Furthermore, for a fixed induced graph $\G$, let $\calM_0$ and $\calM(\epsilon)$ denote the sets of pairwise comparison models that satisfy the null and alternative hypotheses, respectively, in \eqref{Hypothesis test} along with \cref{assm:Large Expansion of DTM,assm: Bounded Principal Ratio}:
    \begin{align}
    \calM_0 & \!\triangleq\!\left\{ P \in [0,1)^{n \times n}: \parbox[]{11em}{$P$ is a pairwise comparison matrix with respect to $\mathcal{E}$ satisfying \cref{assm:Large Expansion of DTM,assm: Bounded Principal Ratio}, and $\Pi P+ P\Pi = \PE(\1_n\pi^\T)$} \right\} , \\
    \calM(\epsilon) & \!\triangleq\!\left\{P \in [0,1)^{n \times n}: \parbox[]{11em}{$P$ is a pairwise comparison matrix with respect to $\mathcal{E}$ satisfying \cref{assm:Large Expansion of DTM,assm: Bounded Principal Ratio}, and $\| \Pi P+ P\Pi - \PE(\1_n\pi^\T) \|_{\F}  \geq n \epsilon \|\pi\|_\infty$}\right\} .
    \end{align}
     Now, we can define the \emph{minimax risk} for the graph $\G$ and a given $\epsilon >0$ as 
    \begin{align}
        \mathcal{R}_{\mathsf{m}}[\G, \epsilon] \triangleq \inf_\Phi \bigg\{ & \sup_{P \in \calM_0} \P_{H_0}(\Phi(\Z) = 1) \nonumber \\ 
        & \ \ \ \ \ + \sup_{P \in \calM(\epsilon)} \P_{H_1}(\Phi(\Z) = 0)\bigg\} ,
        \label{Minimax Risk}
    \end{align}
    where the infimum is taken over all (possibly randomized) $\{0,1\}$-valued decision rules $\Phi$ based on $\mathcal{Z}$. 
    We would like to emphasize here that the suprema in the minimax risk are over all pairwise comparison matrices with a fixed graph $\G$, and the probability measures $\P_{H_0}$ and $\P_{H_1}$ are defined by the randomness in the data generation process given a fixed graph $\G$.
    Finally, we define the \emph{critical threshold} of the hypothesis testing problem in \eqref{Hypothesis test} as the smallest value of $\epsilon$ for which the minimax risk is bounded by $\frac{1}{2}$:
    \begin{equation} \label{critical radius}
    \varepsilon_{\mathsf{c}} = \inf\left\{\epsilon > 0: \mathcal{R}_{\mathsf{m}}[\G, \epsilon] \leq \frac{1}{2} \right\} . \end{equation}
    
    The choice of constant $\frac{1}{2}$ is arbitrary and could be replaced by any constant in $(0,{1})$. 
    
    Formally, \emph{one of our primary goals is to provide bounds on the critical threshold} and determine its scaling with problem parameters like $n$. Intuitively, when $\epsilon$ is larger than $\varepsilon_{\mathsf{c}}$ (in scaling), the minimax risk can be made small, but if $\epsilon$ is smaller than $\varepsilon_{\mathsf{c}}$, then the minimax risk cannot be made small. To analyze the critical threshold, we introduce a statistical test that takes the consolidated observations $\Z$ as input and thresholds the following statistic:
    \begin{align}     
        T \triangleq \sum_{\substack{(i,j)\in  \calE:\\ i\neq j}  } \bigg( \left(\hat{\pi}_{i}+\hat{\pi}_{j}\right)^{2} & \frac{Z_{ij}(Z_{ij}- 1)}{k_{ij}(k_{ij}-1)} +\hat{\pi}_{j}^{2} \nonumber \\ 
        & \ \ \ \ \ -2  \hat{\pi}_{j}\left(\hat{\pi}_{i}+\hat{\pi}_{j}\right) \frac{Z_{ij}}{k_{ij}} \bigg) \I_{k_{ij} > 1}, \label{BoxedTest}
    \end{align}
    where $\hat{\pi}$ denotes the stationary distribution (choosing one arbitrarily if there are several) of the empirical canonical Markov matrix $\Shat \in \R^{n\times n}$ defined via 
    \begin{equation}
        \Shat_{ij} \triangleq  \begin{cases}
        \frac{Z_{ij}}{k_{ij} d}, &  (i,j) \in \calE \text{ and } i\neq j\\
        1-\frac{1}{d}\sum\limits_{\substack{u \in [n]\backslash\!\{i\}:\\ (i,u)\in \calE}}\frac{Z_{iu}}{k_{iu}}, & i = j \\
        0, & \text{otherwise } 
      \end{cases}. \label{EmpiricalMarkovChain}
    \end{equation} 
    To understand the test statistic $T$, notice that if were to set $\hat{\pi} = \pi$ and assume that $k_{ij} \geq 2$ for all $(i,j) \in \calE$ with $i\neq j$ in \cref{BoxedTest} and consider the hypothetical statistic
    \begin{equation}
    \bar{T} =\!\! \sum_{\substack{(i,j)\in  \calE:\\ i\neq j}  }\!\!  \left(\pi_{i}+\pi_{j}\right)^{2} \frac{Z_{ij}(Z_{ij}- 1)}{k_{ij}(k_{ij}-1)} +\pi_{j}^{2} -2  \pi_{j}\left({\pi}_{i}+ {\pi}_{j}\right) \frac{Z_{ij}}{k_{ij}},
    \end{equation}
    then $\E[\bar{T}] = \|\Pi P + P \Pi - \PE(\1_n\pi^\T)\|^2_\F$. So, the statistic $T$ is designed by ``plugging in'' $\hat{\pi}$ in place of $\pi$ in an unbiased estimator $\bar{T}$ of $\|\Pi P + P \Pi - \PE(\1_n\pi^\T) \|^2_\F$. The precise decision rule using $T$ is given in the next section in \cref{eq:Decision rule}.
    
    Lastly, we remark that testing for a single BTL model can be conducted using a $\chi^2$-type goodness-of-fit test \cite{WassermanAllofStatistics}, and testing for a class of BTL models could be attempted using a composite $\chi^2$-test \cite{KAKIZAWAcompositechisquared}. However, performing sharp minimax analysis on the corresponding $\chi^2$-statistics poses challenges. Therefore, our focus lies on the proposed test statistic $T$ in \cref{BoxedTest}. Notably, we demonstrate minimax optimality of the critical threshold for complete graphs, thereby underscoring the effectiveness of considering our proposed test.

\subsection{Upper Bound on Critical Threshold} \label{Upper Bound on Critical Threshold}
    In this section, we present an upper bound on the critical threshold of the minimax hypothesis testing problem for BTL models. 
    For simplicity, we assume that $k_{ij} = k$ for all $(i,j) \in \calE$ with $i\neq j$ throughout the analysis for \cref{Upper Bound on Critical Threshold,sec: Type1 and Type2,Lower Bound on Critical Threshold}. 

    \begin{theorem}[Upper Bound on $\varepsilon_{\mathsf{c}}$] \label{Main theorem with upper bound}
         Consider the hypothesis testing problem in \eqref{Hypothesis test} such that \cref{assm:Large Expansion of DTM,assm: Bounded Principal Ratio} hold. Then, there exist constants $c_3, c_4, c_5 >0$ such that if $n \geq c_3$, the number of comparisons per pair satisfies $k \geq \max\!\big\{2,\frac{c_4 h \log n}{d_{\max} \xi^4 } \big\}$, and $d_{\max} \geq (\log n)^4$, the critical threshold defined in \cref{critical radius} is upper bounded by
         $$ \varepsilon_{\mathsf{c}}^2 \leq \left(\frac{c_5 h^3 }{\xi^4 }\right) \frac{1}{nk}. $$
    \end{theorem}
    The proof is provided in \cref{Proof:Main theorem with upper bound}.
    The proof has several essential steps. Among these, the most important step involves computing $\ell^2$-error bounds (see \cref{lem:L2error} and \cite[Theorem 9]{Chenetal2019}) for the estimated BTL parameters $\hat{\pi}$ when the data is generated by a general pairwise comparison model, which is not necessarily BTL (i.e., under $H_1$). The derivation of these error bounds requires us to analyze the second largest singular values of \emph{divergence transition matrices} (DTMs) $R = \Pi^{1/2} S \Pi^{-1/2}$ \cite{makur2019,MakurZheng2020,Huangetal2024}, corresponding to non-reversible Markov chains $S$, particularly in the context of induced graphs that are not complete. Once we have established these $\ell^2$-error bounds, we can compute bounds on $\E[T]$ and $\var(T)$. This step involves mitigating the correlation between the terms $\hat{\pi}$ and $\{Z_{ij}\}_{(i,j) \in \calE}$ in \cref{BoxedTest} as both of them share the same source of randomness. Broadly speaking, this is done by splitting the product of dependent terms into three parts (by utilizing the identity $\hat{a} \hat{b} = (\hat{a} - a)(\hat{b} - b) + (\hat{a} - a)b + a \hat{b}$, where $\hat{a}$ and $\hat{b}$ are correlated estimators of $a$ and $b$) and bounding each part using appropriate concentration inequalities.  
    Moreover, in the case where the induced graph is complete, the analysis becomes considerably simpler. An instance of such simplification in the proofs arises through the utilization of bounds between contraction coefficients of $S$ (see \cref{subsec:Preliminaries} and \cref{lem:SpectralGapLB}).

    In the scenario where $k_{ij} = k$ for all $(i,j) \in \calE$ with $i\neq j$ and $d_{\max} \geq (\log n)^4$, the decision rule that we analyze to establish \cref{Main theorem with upper bound} is
    \begin{equation}
    \label{eq:Decision rule}
    \Phi(\Z) = \I_{T \geq \thr} ,
    \end{equation}
    where the precise threshold $ {\thr = \Theta(h^3/(nk\xi^4))}$ is given in \cref{TestThreshold}. (In other words, this decision rule returns the alternative hypothesis if and only if $T \geq \thr$ for an appropriate threshold $\thr = \Theta(h^3/(nk\xi^4))$.) 
    In \cref{sec: Estimating the Threshold}, we also present a permutation-test-based approach to obtain a non-explicit threshold for our test based on data. This approach can be more readily employed in simulations and works even when the $k_{ij}$'s are not all equal.

    Finally, we remark that the condition $d_{\max} \geq (\log n)^4$ in \cref{Main theorem with upper bound} can be relaxed to the condition $d_{\max} \geq \log n$ at the expense of a poly-logarithmic factor in the scaling of the critical threshold as demonstrated in the following proposition.
     \begin{proposition}[Upper Bound on $\varepsilon_{\mathsf{c}}$ for Sparse Graphs] \label{Main prop with upper bound}
         Consider the hypothesis testing problem in \eqref{Hypothesis test} such that \cref{assm:Large Expansion of DTM,assm: Bounded Principal Ratio} hold. Then, there exist constants ${\tilde{c}_3, \tilde{c}_4, \tilde{c}_5} >0$ such that if $n \geq {\tilde{c}_3}$, the number of comparisons per pair satisfies $k \geq \max\!\big\{2,\frac{{\tilde{c}_4} h \log n}{d_{\max} \xi^4 } \big\}$, and $d_{\max} \geq \log n$, the critical threshold defined in \cref{critical radius} is upper bounded by      
       $$
     \varepsilon_{\mathsf{c}}^2 \leq \left(\frac{ {\tilde{c}_5} h^3}{\xi^4}\right) \frac{(\log n)^{1/2}}{nk}. 
     $$
    \end{proposition}

    The proof of \cref{Main prop with upper bound} can be gleaned from the proof of \cref{Main theorem with upper bound} in \cref{Proof:Main theorem with upper bound}. In essence, the behavior in \cref{Main prop with upper bound} stems from the fact that the proof of \cref{lem: Bounds on T1 and T2} in \cref{Proof of Bounds on T1 and T2} relies on a concentration inequality (see \cref{thm:user friendly bounds}) which, under the stronger assumption $d_{\max} \geq (\log n)^4$, allows us to avoid the poly-logarithmic factor in \cref{Main prop with upper bound} when establishing \cref{Main theorem with upper bound}. However, if a standard matrix Bernstein inequality were employed for concentration, the ``special'' constants $c_{\alpha}$ and $c_{\gamma}$, defined in \cref{lem: Bounds on T1 and T2}, would scale as $(\log n)^{1/2}$. Then, the proof of \cref{Main prop with upper bound} would follow by using essentially the same logic as the proof of \cref{Main theorem with upper bound} and observing that \cref{Polynomial} yields an additional factor of $(\log n)^{1/2}$ in the scaling of $\varepsilon^2_\mathsf{c}$. Moreover, in this regime, our decision rule returns the alternative hypothesis if and only if $T \geq \frac{\gamma h^3 \sqrt{\log n}}{n k \xi^4} $ for some appropriately chosen constant $\gamma$.

We end this subsection with two remarks. Firstly, we can conclude by combining \cref{Distance to closest BTL model,Main theorem with upper bound} that when the separation distance satisfies  
    \begin{equation}
        \min_{B \in  \BTLh} \|P - \PE(B) \|_\F \geq {\Omega}\bigg(\frac{h^{5/2}}{\sqrt{nk \xi^4}} \bigg),
    \end{equation}  
    our test guarantees a minimax risk of at most $\frac{1}{2}$. This result demonstrates that the minimum separation distance required by our test increases with the principal ratio and when the underlying observation graph has ``poor'' expansion properties. Secondly, while it remains unclear whether the $k \geq 2$ condition is strictly necessary for the aforementioned results, we note that this condition was also required in a related two-sample testing problem~\cite{RastogiBalakrishnanShahAarti2022} where it was shown to be necessary. Intuitively, the $k \geq 2$ requirement arises because our test statistic is based on a second-moment (or variance) approximation, which requires at least two samples to obtain a meaningful estimate.

\subsection{Bounds on Type \Romannum{1} and Type \Romannum{2} Probabilities of Error} \label{sec: Type1 and Type2}
In this section, we provide bounds on the type \Romannum{1} and type \Romannum{2} probabilities of error for our proposed test. Here, the probability of type \Romannum{1} error represents the probability of incorrectly rejecting the null hypothesis when it is true, while the probability of type \Romannum{2} error represents the probability of failing to reject the null hypothesis when it is false. 

\begin{theorem}[Bounds on Type \Romannum{1} and Type \Romannum{2} Probabilities of Error]
\label{thm:Type-1 and Type-2 error}
Consider the hypothesis testing problem in \eqref{Hypothesis test} such that \cref{assm:Large Expansion of DTM,assm: Bounded Principal Ratio} hold, and suppose that $d_{\max} \geq (\log n)^4$ and there exists a constant $c_4 >0$ such that the number of comparisons per pair satisfies $k \geq \max\!\big\{2,\frac{c_4 h \log n}{d_{\max} \xi^4 } \big\}$. Then, we have the following bounds: 
\begin{enumerate}
\item There exist constants $c_6, c^\prime_6, \tilde{c}_6>0$ such that for all $t \geq 0$, and any BTL model in $\mathcal{M}_0$, the probability of type \Romannum{1} error is upper bounded by 
$$
\begin{aligned}
    \mathbb{P}_{H_{0}}\bigg(T\geq u \frac{h^2}{nk} & +  c_{6} \frac{h^3}{\xi^4} \frac{1}{nk}\bigg) \leq \\
    &  \exp \bigg(-c_6^\prime \min \bigg\{\frac{  nu^2 }{   d_{\max  } },  nu \bigg\} \bigg) + \frac{ \tilde{c}_6 }{n^3}. 
\end{aligned}
$$
\item There exist constants $c_7, c^\prime_7, c^{\prime \prime}_7, \tilde{c}_7>0$ such that for all $t \geq 0$, and any pairwise comparison model in $\mathcal{M}(\epsilon)$ with $\epsilon \geq c_7\sqrt{h^3} /\sqrt{nk \xi^4}$  the probability of type \Romannum{2} error is upper bounded by 
\begin{align*}
     \mathbb{P}_{H_{1}}&\left(T \leq  \bigg(D -\! {\frac{\sqrt{h^3} }{ \xi^2}} \frac{c_7}{\sqrt{nk}}\bigg)^{2}\!\! - \frac{{u}}{nk} \right) \leq \\ 
    &\ \ \ \ \ \ \ \ \ \ \ \ \ \exp\! \bigg(- c^\prime_7 \min \bigg\{\frac{  n u^2}{ d_{\max  } },\! nu  \bigg\}
    \bigg) + \\ 
    &\ \ \ \ \ \ \ \ \ \ \ \ \   \exp \left(-c^{\prime \prime}_7\frac{ n^2 \|\pi\|_\infty^2 u^{2} }{ D^2k }\right) + \frac{ \tilde{c}_7 }{n^3},  
\end{align*}
where $D = \|\Pi P+ P\Pi - \PE(\1_n\pi^\T)\|_\F $.
\end{enumerate}
\end{theorem}
\cref{thm:Type-1 and Type-2 error} is established in \cref{Proof of Type-1 error}.
\subsection{Lower Bound on Critical Threshold} \label{Lower Bound on Critical Threshold}
In this section, we present an information-theoretic lower bound on the critical threshold for the hypothesis testing problem in \cref{Hypothesis test} for the special case of a complete induced graph $\G$. Our lower bound demonstrates the minimax optimality of the scaling of the critical threshold provided in the upper bound in \cref{Main theorem with upper bound} for complete induced graphs. (We will show in \cref{sec: Graphs with Bounded Principal Ratio} that pairwise comparison models on complete graphs can satisfy \cref{assm:Large Expansion of DTM,assm: Bounded Principal Ratio}.) 
As noted earlier, we assume that $k_{ij} = k$ for all $i,j \in [n]$ with $i\neq j$. 

\begin{theorem}[Lower Bound on $\varepsilon_{\mathsf{c}}$] \label{thm:Lower Bound}
Consider the hypothesis testing problem in \cref{Hypothesis test} and assume that the corresponding induced graph $\G$ is a complete graph. Then, there exists a constant $c_8>0$ such that the critical threshold defined in \cref{critical radius} is lower bounded by
$$ \varepsilon_{\mathsf{c}}^2 \geq \frac{c_8}{n k} . $$
\end{theorem}
The proof of \cref{thm:Lower Bound} is provided in \cref{sec: Proof of Lower Bound and Stability}. The proof uses the \emph{Ingster-Suslina method} for constructing a lower bound on the critical threshold \cite{IngsterSuslina2003,RastogiBalakrishnanShahAarti2022}. This method is similar to the well-known \emph{Le Cam's method}, but it establishes a minimax lower bound by considering a point and a mixture on the parameter space instead of just two points. (Although Le Cam's method could also be used for this proof in principle, the Ingster-Suslina method greatly simplifies the calculations to bound TV distance in our setting.) Our proof constructs a perturbed family of pairwise comparison models from a fixed BTL model and utilizes the complete graph structure to compute both the stationary distribution and the separation metric $\|\Pi P + P\Pi - \PE(\1_n \pi^\T)\|_\F/\|\pi\|_\infty$ in closed form under $H_1$. We remark that due to our problem setting, our proof here is much simpler than and different to the technique developed in \cite{Seshadri2020}, where the separation distance is quantified for Eulerian graphs in terms of sums of TV distances. We leave the problem of determining minimax lower bounds for more general graph topologies satisfying \cref{assm:Large Expansion of DTM,assm: Bounded Principal Ratio} for future work. We also note that the bounds in \cref{Main theorem with upper bound,thm:Lower Bound} portray that the $\Theta(1/\sqrt{nk})$ scaling of the critical threshold is minimax optimal even if we took suprema over all induced graphs satisfying \cref{assm:Large Expansion of DTM,assm: Bounded Principal Ratio} in the definition of minimax risk in \cref{Minimax Risk}. Finally, we remark that the construction of our lower bound inherently satisfies weak stochastic transitivity. Consequently, when the class of non-BTL models is restricted to the class of weakly stochastically transitive matrices, our lower bound result continues to hold. So, even under the restricted class of models satisfying weak stochastic transitivity, we obtain the same lower bound, and thus the testing problem is not simpler (in terms of critical threshold).

\subsection{Graphs with Bounded Principal Ratio} \label{sec: Graphs with Bounded Principal Ratio}

In this section, we establish bounds on both the principal ratio and edge expansion for three distinct classes of graphs: complete graphs, $\tilde{d}$-regular graphs, and random graphs generated from Erd\H{o}s-R\'enyi models. These three classes represent a few examples of graphs for which \cref{assm:Large Expansion of DTM,assm: Bounded Principal Ratio} hold and the theoretical guarantees of our testing framework are valid for any pairwise comparison matrix $P$ consistent with \cref{Pijbounded}. In the first case, we assume that the induced graph is complete and pairwise comparisons among all pairs are observed. The second scenario involves $\tilde{d}$-regular graphs that are sufficiently dense (as explained later) and possess some degree of edge expansion. And the third case assumes the existence of a complete underlying pairwise comparison model consistent with \cref{Pijbounded}, where comparisons between any pair $(i,j)$ (and $(j,i)$), for $i>j$, are observed independently with probability $p$.\footnote{Throughout this work, 
$p$ without subscripts denotes the Erdős-Rényi edge probability, while $p_{ij}$ refers to pairwise comparison probabilities.}  We show that there exists a constant $\mathsf{c}_p > 0$ such that as long as $n p \geq \mathsf{c}_p \log n $, the Erd\H{o}s-R\'enyi pairwise comparison model satisfies \cref{assm:Large Expansion of DTM,assm: Bounded Principal Ratio} with high probability. We also note that much of our technical analysis in each of the three scenarios lies in analyzing the principal ratio.

\subsubsection{Complete Graphs}
We begin by deriving bounds on the principal ratio and the edge expansion of a canonical Markov matrix $S$ corresponding to a complete graph. To this end, we consider a pairwise comparison matrix $P$ corresponding to a complete graph on $n$ vertices and construct $S$ using $d=2n$. (Note that $d = n$ would also work for the case of a complete graph.) In the following proposition, we show that the principal ratio is always upper bounded by $1 / \delta^{2}$ for all pairwise comparison matrices consistent with \cref{Pijbounded}. 

    \begin{proposition}[Principal Ratio for Complete Graphs] \label{lem: Height of Perron–Frobenius Eigenvector}
        Let $S$ be a canonical Markov matrix in \cref{Def: Canonical Markov Matrix} corresponding to a complete graph with $d = 2n$ and stationary distribution $\pi$. Suppose further that \cref{Pijbounded} holds. Then, we have 
            $$ h_\pi = \frac{\max_{i \in [n]} \pi_i}{\min_{j \in [n]} \pi_j} \leq  \frac{1}{\delta^2} . $$
    \end{proposition}
    The proof is provided in \cref{Proof of lem: Height of Perron–Frobenius Eigenvector}; our argument is a modification of that in \cite[Theorem 3.1]{Minc1988}.
    \cref{lem: Height of Perron–Frobenius Eigenvector} illustrates that \cref{assm: Bounded Principal Ratio} holds with $h = 1/\delta^2$. It also implies that the pairwise comparison model satisfies \cref{assm:Large Expansion of DTM} with $\xi = \frac{\delta^3}{4(1+\delta)}$, because we obtain the following lower bound on the edge expansion $\phi(S)$ using \cref{trivial lower bound of expansion} by substituting $d = 2n$:
    \begin{equation}
           \phi(S) \geq \frac{\delta^3}{2(1+\delta)n} \tilde{\phi}(\G) = \frac{\delta^3}{4(1+\delta)},
       \end{equation}   
    where we have utilized the fact that for complete graphs, $\tilde{\phi}(\G) = \frac{n}{2}$. 
    Alternatively, for complete graphs, we can obtain tighter upper bounds on the critical threshold (i.e., with better implicit dependence on $\delta$) by directly bounding the second largest singular value $\sigma_2(R)$ of the DTM $R = \Pi^{1/2} S \Pi^{-1/2}$ (see \cref{lem:SpectralGapLB} in \cref{subsec:Preliminaries}) instead of relying on expansion properties and \emph{Cheeger inequalities} (i.e., using \cref{assm:Large Expansion of DTM} and \cref{lem:SpectralGapLB for a general graph} in \cref{subsec:Preliminaries}). This alternative approach leverages bounds between contraction coefficients, specifically, in terms of the \emph{Dobrushin contraction coefficient} for TV distance (cf. \cite{makur2019, MakurZheng2020}); see \cref{subsec:Preliminaries} for details. Thus, for the case of complete graphs, we have shown that any pairwise comparison model satisfying \cref{Pijbounded} also satisfies \cref{assm:Large Expansion of DTM,assm: Bounded Principal Ratio}. This allows us to test whether data generated from any pairwise comparison matrix $P$ corresponding to a complete induced graph satisfying \cref{Pijbounded} conforms to an underlying BTL model. However, the testing procedure for a complete graph requires $n(n-1)k$ samples.

    \subsubsection{Dense $\tilde{d}$-Regular Graphs}
    We next derive a bound on the principal ratio of a canonical Markov matrix $S$ corresponding to a $\tilde{d}$-regular graph under some additional assumptions. In this case, we set the parameter $d = 2\tilde{d}$ to construct $S$. Notably, we illustrate that when the degree of the regular graph satisfies $\tilde{d} = \Theta(n)$, i.e., the graph is dense, the principal ratio can be upper bounded by a constant. This result holds even if the induced graph depends on the complete set of underlying pairwise comparison probabilities. The ensuing proposition bounds the principal ratio for dense $\tilde{d}$-regular graphs under additional assumptions.
    
    \begin{proposition}[Principal Ratio for Dense Regular Graphs] \label{lem: Height of Perron–Frobenius Eigenvector for expanders}
    Let $S$ be a canonical Markov matrix in \cref{Def: Canonical Markov Matrix} corresponding to a $\tilde{d}$-regular graph $\G$ with $d = 2\tilde{d}$ and stationary distribution $\pi$. Suppose further that \cref{Pijbounded} holds, and for some constants $a, b, c > 0$, $\G$ satisfies $|\calE(\mathcal{S}, \mathcal{T})| \geq a|\mathcal{S}||\mathcal{T}|$ for all disjoint subsets $\mathcal{S},\mathcal{T} \subseteq [n]$ with $|\mathcal{S}| \geq bn$ and $|\mathcal{T}| \geq cn$. If $\tilde{d} \geq cn$, then we have 
            $$ h_\pi =  \frac{\max_{i \in [n]} \pi_i }{\min_{j \in [n]} \pi_j}  \leq \frac{\tilde{c}(a)}{\delta^5} , $$
    where $\tilde{c}(a) > 0$ is a constant that depends on $a$.
    \end{proposition}
    The proof is provided in \cref{Proof of lem: Height of Perron–Frobenius Eigenvector for expanders}. This result generalizes the result in \cite[Theorem 3]{AksoyChung2016}, and illustrates that \cref{assm: Bounded Principal Ratio} holds with $h = \tilde{c}(a)/\delta^5$. The assumption, $|\mathcal{E}(\mathcal{S}, \mathcal{T})| \geq a|\mathcal{S}||\mathcal{T}|$ for all disjoint subsets $\mathcal{S}, \mathcal{T}$ with cardinality $\Theta(n)$, is typically satisfied by \emph{$(n, \tilde{d}, \lambda_2(\G))$-regular expander graphs} $\G$ with $\tilde{d}= \ \nu n$ and $\lambda_2(\G) \leq (1-\tilde{\nu})\tilde{d}$ for some constants $\nu, \tilde{\nu} \in (0,1]$, where $\lambda_2(\G)$ denotes the second largest eigenvalue modulus of the adjacency matrix of $\G$; we refer readers to \cite{AlonChung1989} for the definition of $(n, \tilde{d}, \lambda_2(\G) )$-regular expander graphs. We also refer the readers to \cref{Expansion} for details regarding the existence of $(n, \tilde{d}, \lambda_2(\G))$-regular expander graphs with $\tilde{d} = \nu n$ and $\lambda_2(\G) \leq (1-\tilde{\nu})\tilde{d},$ and why they satisfy the assumption in \cref{lem: Height of Perron–Frobenius Eigenvector for expanders}. 
    Furthermore, pairwise comparison models corresponding to such graphs also satisfy \cref{assm:Large Expansion of DTM} as explained in \cref{Expansion}, cf. \cite[Theorem 9.2.1]{AlonSpencer2008}. Thus, for the case of dense $\tilde{d}$-regular graphs with appropriate expansion properties, we have again established that any pairwise comparison model satisfying \cref{Pijbounded} also satisfies \cref{assm:Large Expansion of DTM,assm: Bounded Principal Ratio}. This permits us to test whether data generated from any pairwise comparison matrix $P$ corresponding to a dense $\tilde{d}$-regular graph satisfying \cref{Pijbounded} and certain expansion properties conforms to an underlying BTL model. Moreover, the denseness requirement $\tilde{d}= \Theta(n)$ implies that the testing procedure requires $\Theta(n^2k)$ samples in order to satisfy the various assumptions. The next setting establishes the utility of our testing procedure for sparse induced graphs.

\subsubsection{Erd\H{o}s-R\'enyi Random Graphs}
In this case, we assume that there exists an underlying (predetermined) pairwise comparison matrix $P^{*}$, corresponding to a complete graph, which satisfies \Cref{Pijbounded}. Subsequently, as in \emph{Erd\H{o}s-R\'enyi random graphs} \cite{BelaBollobas}, we sample each edge $(i, j) \in [n]^2$ with $i > j$ of the undirected induced graph independently with probability $p \in [0,1]$, such that for some sufficiently large constant $\mathsf{c}_p>1$, we have $n p \geq \mathsf{c}_p \log n$.\footnote{Note that the constant $\mathsf{c}_p$ does not depend on $p$.} In essence, we randomly sample the entries of a fixed comparison matrix $P^{*}$ to obtain a pairwise comparison matrix $P = \PE(P^{*})$, where $\calE$ is the random edge set. Next, we construct canonical Markov matrices $S^*$ and $S$ corresponding to $P^*$ and $P$ with $d =3n$ and $d = 3 np$, respectively. The following well-known proposition provides high probability bounds on the vertex degrees for an Erd\H{o}s-R\'enyi random graph, highlighting that $S$ is indeed a $(\frac{1}{2})$-lazy Markov matrix with high probability.
\begin{proposition}[Degree Concentration for Erd\H{o}s-R\'enyi Random Graph {\cite[Lemma 1]{Chenetal2019}}] \label{lem: Degree Concentration} 
Suppose that the induced graph $\mathcal{G}$ is an Erd\H{o}s-R\'enyi random graph with edges selected independently with probability $p$, and let $d_{\min} \triangleq \min_{i \in [n]} \sum_{j \in [n]\backslash\!\{i\}} E_{ij}$ and $d_{\max} \triangleq \max_{i \in [n]} \sum_{j\in[n]} E_{ij}$ be the minimum and maximum degrees, respectively. If $p \geq \frac{c_0 \log n}{n}$ for some sufficiently large constant $c_0>0$, then the event $\mathcal{A}_0$, defined as
$$
\mathcal{A}_0\triangleq \left\{\frac{n p}{2} \leq d_{\min } \leq d_{\max } \leq \frac{3 n p}{2}\right\},
$$
occurs with probability at least $1-O(n^{-5}).$
\end{proposition}
We emphasize that the probability law utilized here differs from that in the minimax risk framework. The definition of minimax risk in \cref{Minimax Risk} assumes a fixed graph $\G$ and considers probabilities associated with the randomness in the data generating mechanism. In contrast, the probability laws in \cref{lem: Degree Concentration} and other results in this section are governed by the graph generation process. 

Let $\pi^*$ and $\pi$ be the stationary distributions of $S^*$ and $S$, respectively. Notably, by our assumption $n p \geq \mathsf{c}_p \log n$ for $\mathsf{c}_p>1$, the underlying random graph $\G$ is connected with high probability \cite{BelaBollobas}. Therefore, $S$ is irreducible and aperiodic with high probability (as $d_{\max} \leq 3np$ implies $S_{ii}>0$ for all $i \in [n]$), and hence, it has a unique stationary distribution $\pi$. 

Utilizing \cref{lem: Height of Perron–Frobenius Eigenvector}, we know that the principal ratio of $\pi^*$ is upper bounded by a constant $1/\delta^2$ (note that \cref{lem: Height of Perron–Frobenius Eigenvector} also holds for $d = 3 n$, since the stationary distribution is independent of $d$ as long as $d \geq n$).
 In the sequel, we will prove that as long as $n p \geq \mathsf{c}_p \log n$ for some large enough constant $\mathsf{c}_p$, the principal ratio of a pairwise comparison model over an Erd\H{o}s-R\'enyi random graph is also upper bounded by a constant with high probability. To bound the principal ratio, we first provide a perturbation bound between $\pi^*$ and $\pi$. 
\begin{theorem}[$\ell^\infty$-Bound under Sub-sampling using Erd\H{o}s-R\'enyi Model] \label{thm: Perturbation Bound on Sparsified Stationary Distribution}
Given an underlying pairwise comparison matrix $P^*$ satisfying \cref{Pijbounded} with $n \geq \tilde{c}_9$ for some constant $\tilde{c}_9 > 1$, suppose we obtain the pairwise comparison matrix $P$ by randomly sampling the entries of $P^*$ according to an Erd\H{o}s-R\'enyi model with parameter $p \in (0,1]$ satisfying $p\geq \frac{{c}_{9}  \log n}{n }$ for some constant $c_9 > 1$. Then, there exists a constant $c_{10} > 0$ such that with probability at least $1-O(n^{-3})$, the stationary distribution $\pi$ satisfies 
$$ \frac{\left\|{\pi}-{\pi}^*\right\|_{\infty}}{\left\|{\pi}^*\right\|_{\infty}} \leq \sqrt{\frac{c_{10} \log n}{n p}} . $$
\end{theorem}
The proof of \cref{thm: Perturbation Bound on Sparsified Stationary Distribution} is provided in \Cref{Proof of Perturbation Bound on Sparsified Stationary Distribution} and utilizes the leave-one-out technique in \cite{Chenetal2019}. The theorem immediately yields the following bound on the principal ratio corresponding to the sub-sampled pairwise comparison matrix $P$.
\begin{proposition}[Principal Ratio for Erd\H{o}s-R\'enyi Random Graphs] \label{lem: Height of Perron–Frobenius Eigenvector under sampling}
Given an underlying pairwise comparison matrix $P^*$ satisfying \cref{Pijbounded} with $n \geq \tilde{c}_9$, suppose we obtain the pairwise comparison matrix $P$ by randomly sampling the entries of $P^*$ according to an Erd\H{o}s-R\'enyi model with parameter $p \in (0,1]$ satisfying $p\geq \frac{\mathsf{c}_p  \log n}{n}$ for a constant $\mathsf{c}_p \geq \max\{c_9, 2c_{10}/\delta^4\}$ (where $c_9, \tilde{c}_9, c_{10}$ are given in \cref{thm: Perturbation Bound on Sparsified Stationary Distribution}). Then, with probability at least $1 - O(n^{-3})$, the principal ratio satisfies the bound
        $$  h_\pi = \frac{\max_{i \in [n]} \pi_i }{\min_{j \in [n]} \pi_j}  \leq  \frac{7}{\delta^2}. $$
\end{proposition}

The proof of \cref{lem: Height of Perron–Frobenius Eigenvector under sampling} is provided in  \cref{Proof of Height of Perron–Frobenius Eigenvector under sampling}. In particular, \cref{lem: Height of Perron–Frobenius Eigenvector under sampling} illustrates that \cref{assm: Bounded Principal Ratio} holds with $h = 7/\delta^2$ with high probability. Moreover, random graphs are known to have nice expansion properties \cite[Theorem 2.8]{BelaBollobas} (also see \cite{SipserSpielman1996}).
Indeed, using \cite[Equation 11]{Han2023}, we have that for an Erd\H{o}s-R\'enyi random graph, $\tilde{\phi}(\G) \geq \frac{1}{4} np$ with high probability, thereby showing that the pairwise comparison matrix $P$ also satisfies \cref{assm:Large Expansion of DTM} with {$\xi = \delta^3/(12 \cdot 7\cdot  (1+\delta))$} with high probability. Thus, for the case of Erd\H{o}s-R\'enyi random graphs, we have again demonstrated that any pairwise comparison model satisfying \cref{Pijbounded} also satisfies \cref{assm:Large Expansion of DTM,assm: Bounded Principal Ratio} with high probability, as required for theoretical guarantees on our testing procedure. Notably, in this case, the testing procedure requires $O(k n \log n)$ pairwise comparisons, which matches the total number of observations needed for consistently estimating the parameters of the BTL model in \cite{Chenetal2019}.

\section{Proofs for Problem Formulation} \label{sec:proofs}
We prove \cref{Prop: BTL Model and Reversibility,OrthogonalDecomposition,EasyProp,Distance to closest BTL model} in this section. Throughout this and future sections, we employ a concise notation by using overlapping labels (e.g., $c, \tilde{c}, c_1, c_2, c^\prime, \hat{c}, \dots $) to denote various constants. To avoid ambiguity, we explicitly reserve the notation $c_\alpha, c_\beta, c_\gamma$ for specially defined constants in \cref{lem: Bounds on T1 and T2} and the subsequent proof (these special constants can depend on parameters like $n$). We also note that some of the bounds used in the subsequent analysis (in this and future sections) are valid only with high probability. For brevity, we do not restate this at every occurrence.

\subsection{Proof of \cref{Prop: BTL Model and Reversibility}} \label{Proof of Prop: BTL Model and Reversibility}
    We provide the proof for completeness. If the pairwise comparison model is BTL, it implies that for some weight vector $\alpha \in \R^n_+$, the pairwise comparison matrix $P$ is given by
    $$p_{ij} = \frac{\alpha_j}{\alpha_i + \alpha_j}, \ \forall (i, j) \in \calE. $$
    It is easy to verify that $\pi \triangleq \big(\sum_{i=1}^n \alpha_i\big)^{-1} [\alpha_1, \dots, \alpha_n \big]^\T $ is the stationary distribution of the canonical Markov matrix $S$ corresponding to $P$. Moreover, since $S$ is reversible, for any $(i,j) \in \calE$, we have 
    $$\pi_i S_{ij} = \frac{\alpha_i}{\sum_{i=1}^n \alpha_i} \times \frac{\alpha_j}{ d (\alpha_i + \alpha_j)}  = \pi_j S_{ji} .  $$
    Note that both the stationarity and reversibility conditions hold only if $\calE$ is symmetric.

    For the converse, since $p_{ij} > 0$ for all $(i,j) \in \calE$, $S$ exhibits irreducibility since the graph $\G$ is {strongly-connected}, and aperiodicity as $S_{ii}>0$. As a result, $S$ has a unique stationary distribution $\pi$. By reversibility of $S$, we have for all $(i,j) \in \calE$ with $i \neq j$,
    $$ \pi_i S_{ij} = \pi_j S_{ji} \implies \pi_i p_{ij} = \pi_j p_{ji} \implies  p_{ij} = \frac{\pi_j}{\pi_i + \pi_j} , $$
    where last step follows from the fact that $p_{ij} + p_{ji} = 1 $. Thus, $P$ corresponds to a BTL model with parameters $\pi$. \qed

\subsection{Proof of \cref{EasyProp}} \label{EasyPropProof} 
If $P$ corresponds to a BTL model, it implies that for some strictly positive weights $\alpha_i, \ i \in [n]$, we have $ p_{ij} = \alpha_j/(\alpha_i + \alpha_j)$ for $(i,j) \in \calE$. It is easy to verify that the stationary distribution of the corresponding canonical Markov matrix is  $$\pi \triangleq \left[\frac{\alpha_1}{\sum_{i=1}^n \alpha_i }, \cdots , \frac{\alpha_n}{\sum_{i=1}^n \alpha_i }\right]^\T, $$
and the stationary distribution satisfies $\Pi P + P \Pi = \PE(\1_n \pi^\T)$. Also, the stationary distribution is independent of graph structure (as long as the adjacency matrix of the graph is symmetric).
Conversely, if $\Pi P + P \Pi = \PE(\1_n \pi^\T)$ is true, this implies that  
$$ \forall \ (i,j) \in \calE, \ p_{ij}\left(\pi_{i}+\pi_{j}\right)=\pi_{j} \implies p_{ij}=\frac{\pi_{j}}{\pi_{i}+\pi_{j}} .  $$
This completes the proof. \qed

\subsection{Proof of \cref{OrthogonalDecomposition}} \label{OrthogonalDecompositionProof}
Observe that
\begin{align*}
    & \| \Pi P + P \Pi -  \PE(\mathbf{1}_n\pi^\T ) \|^2_{\pi^{-1},\F} = \| \Pi P - P^\T \Pi \|^2_{\pi^{-1},\F} \\ 
    & \ \ \ \ \ \ \ \ \ \ \ \ \ \ \ \ \ \ \ \ \ + \|  (P^\T  + P - \PE(\mathbf{1}_n\mathbf{1}_n^\T) )\Pi  \|^2_{\pi^{-1},\F} + \\
    &  2\text{Tr}\left(\Pi^{-1/2}(\Pi P - P^\T \Pi)^\T(P^\T  + P - \PE(\mathbf{1}_n\mathbf{1}_n^\T) )\Pi^{1/2}\right) \\
    & = \| \Pi P - P^\T \Pi \|^2_{\pi^{-1},\F} + \|  P^\T  + P - \PE(\mathbf{1}_n\mathbf{1}_n^\T)  \|^2_{\pi,\F}  + \\
    & \ \ \ \ \ \ \ \ \ \ \ \ \ 2\text{Tr}\left((\Pi P - P^\T \Pi)^\T(P^\T  + P - \PE(\mathbf{1}_n\mathbf{1}_n^\T) ) \right).
\end{align*} \sloppy
Now it remains to show that $\text{Tr}\big((\Pi P - P^\T \Pi)^\T(P^\T  + P - \PE(\mathbf{1}_n\mathbf{1}_n^\T ))\big) = 0$. This follows from the fact that 
 $\text{Tr}(A^\T B) = 0$ when $A$ is anti-symmetric, i.e., $A^\T = -A$, and $B $ is symmetric, i.e., $B = B^\T $:
$$
\begin{aligned}
 \text{Tr}(A^\T B) & = - \text{Tr}(A B) = -  \text{Tr}(A B^\T) =  - \text{Tr}(A^\T B) \\ 
 &\implies \text{Tr}(A^\T B) =0  .
 \end{aligned}
 $$
 Clearly, $\Pi P - P^\T \Pi $ is anti-symmetric, and the symmetry of $P^\T  + P - \PE(\mathbf{1}_n\mathbf{1}_n^\T )$ follows since the set $\calE$ is symmetric by assumption.
This completes the proof. \qed

\subsection{Proof of \cref{Distance to closest BTL model}}\label{Proof of Distance to closest BTL model}
We begin our proof by establishing the upper bound. Assume that for a BTL model $B \in \BTLh$, its skill score parameters are given by $\pi_{B} \in \R^n_+$. We will interchangeably use the notation $\pi_B \in \BTLh$ to mean $B \in \BTLh$. Moreover, without loss of generality, we assume that $\sum_{i=1}^n \pi_{B, i} = 1$, where $\pi_{B,i}$ is the $i$th component of vector $\pi_B$ and moreover, let $\Pi_B \triangleq \diag(\pi_{B})$. 
Thus, we have  
$$
\begin{aligned}
  \min_{B \in \BTLh } &  \|P- \PE(B) \|_{\F}  \\
 & =\min _{\pi_B \in \BTLh} \sqrt{\sum_{(i, j) \in \calE} \frac{((\pi_{B,i}+\pi_{B,j})p_{ij}-\pi_{B,j})^{2}}{(\pi_{B,i}+\pi_{B,j})^{2}}} \\
& \stackrel{}{\leq} \min_{\pi_B \in \BTLh} \frac{\left\|\Pi_{B} P+P \Pi_{B}-\PE(\1_n \pi_{B}^{\T}) \right\|_{\F}}{2 \pi_{B,\min }} \\
& \stackrel{\zeta_1}{\leq}\min_{\pi_B \in \BTLh} h\frac{\|\Pi_B P+P \Pi_B - \PE(\1_n \pi_{B}^{\T})\|_{\F}}{2  \|\pi_B\|_{\infty}} \\
&\stackrel{\zeta_2}{\leq} h \frac{\|\Pi P+P \Pi-\PE(\1_n \pi^{\T})\|_{\F}}{2 \|\pi\|_{\infty}},
\end{aligned}
$$
where $\pi_{B,\min} \triangleq \min_{i \in [n]} \pi_{B,i}$, $\zeta_1$ follows from the fact that $\pi_{B, \min } \geq \|\pi_{B}\|_{\infty}/h $ as $B \in \BTLh$, and $\zeta_2$ follows by substituting $\pi_B = \pi$, which are the skill score parameters obtained from the stationary distribution of the canonical Markov matrix $S$ corresponding to $P$, and by \cref{assm: Bounded Principal Ratio}, $\pi$ belongs to the set $\BTLh$.

Now, we will focus on proving the lower bound. Specifically, we will show that there exists a constant $c^\prime$ such that
\begin{align}
 \min_{B \in \BTLh} \|P-\PE(B)\|_{\F} &\geq \frac{1}{2} \|P+P^{\T}-\PE(\1_n\1_n^{\T})\|_{\F} \label{part1aa},\\
 \min_{B \in \BTLh} \|P-\PE(B)\|_{\F} &\geq c^{\prime}\frac{\xi^2}{h^{3} } \|\Pi^{\frac{1}{2}} P \Pi^{-\frac{1}{2}}\!-\!\Pi^{-\frac{1}{2}} P^{\T} \Pi^{\frac{1}{2}}\|_{\F}. \label{part1bb}
\end{align}
If both \cref{part1aa,part1bb} are true, then we can use the following argument to complete the proof. Letting $c_1 = \min\{0.5 ,c^\prime \frac{\xi^2}{h^{3}} \}$, we have
\begin{align}
    &  2 \min_{B \in \BTLh} \|P-\PE(B)\|_{\F} \nonumber \\
    & \geq  c_1(\|\Pi^{\frac{1}{2}} P \Pi^{-\frac{1}{2}}\!-\!\Pi^{-\frac{1}{2}} P^{\T} \Pi^{\frac{1}{2}}\|_{\F} \!  +\! \|P+\!P^{\T}\!-\!\!\PE(\1_n\1_n^{\T}) \|_{\F}) \nonumber \\
    & \stackrel{\zeta_1}{\geq}  c_1\bigg( \frac{\|\Pi P - P^{\T} \Pi\|_{\pi^{-1},\F}}{\sqrt{\|\pi\|_\infty}} + \frac{\|P+P^{\T}-\PE(\1_n\1_n^{\T}) \|_{\pi,\F}}{\sqrt{\|\pi\|_\infty}}\bigg)\nonumber \\
    & \stackrel{\zeta_2}{\geq} c_1 \frac{\|\Pi P+P \Pi- \PE(\1_n \pi^{\T})\|_{\pi^{-1},\F}}{\sqrt{\|\pi\|_\infty}} \nonumber  \\
    & \geq c_1 \frac{\|\Pi P+P \Pi- \PE(\1_n \pi^{\T}) \|_{\F}}{{{\|\pi\|_\infty}}}, 
\end{align}
where $\zeta_1$ follows since $\|A B\|_\F \geq \sigma_{n}(A) \|B\|_\F$ and $\zeta_2$ follows from \cref{OrthogonalDecomposition} and the sub-additivity of the square root function. 

Next, we will prove \cref{part1aa}. Note that
$$
\begin{aligned}
\min _{B \in \BTLh}\|P-\PE(B)\|_{\F} & \geq \min _{ \substack{B \in [0,1]^{n \times n}:\\ B+B^{\T}= \1_n\1_n^{\T}}}\| P - \PE(B)\|_{\F} \\
& =\frac{1}{2}\left\|P+P^{\T}-\PE(\1_n\1_n^{\T})\right\|_{\F},
\end{aligned}
$$
where the first inequality follows since the minimization is performed over a larger set that contains $\BTLh$, and the last equality follows from the fact that for any $(i,j) \in \calE,$ we have 
$$
\min _{b_{ij} \in [0,1]} (p_{i j}-b_{i j})^{2}+(p_{j i}-(1-b_{i j}))^{2}=\frac{1}{2}(p_{i j}+p_{j i}-1)^{2}.
$$
Now we will prove \cref{part1bb}. Observe that
\begin{align*}
 & \min_{B \in \BTLh}\|P-\PE(B)\|_{\F} \\ 
 & \geq \min _{B \in \BTLh} \sqrt{\frac{\pi_{B,\min } }{\|\pi_{B}\|_{\infty}}} \|\Pi_{B}^{1/2} P \Pi_{B}^{-1/2}-\Pi_{B}^{1/2} \PE(B) \Pi_B^{-1/2}\|_{\F} \\
& \stackrel{\zeta_1}{=} \min _{B \in \BTLh} \sqrt{\frac{\pi_{B,\min } }{\|\pi_{B}\|_{\infty}}} \|\Pi_{B}^{1/2} P \Pi_{B}^{-1/2}-\PE(\Pi_{B}^{1/2} B \Pi_B^{-1/2})\|_{\F} \\
& \stackrel{\zeta_2}{\geq} \min_{\pi_{B} \in \BTLh} \min_{ \substack{C \in \R^{n \times n}:\\ C= C^\T}} \frac{\|\Pi_{B}^{1 / 2} P \Pi_{B}^{-1/2}-C\|_{\F}}{\sqrt{h}} \\
& = \min _{\pi_{B} \in \BTLh} \frac{\|\Pi_{B}^{1 / 2} P \Pi_{B}^{-1 / 2}-\Pi_{B}^{-1 / 2} P^{\T} \Pi_{B}^{1 / 2}\|_{\F}}{2 \sqrt{h}},
\end{align*}
where $\zeta_1$ follows since $\Pi_B$ is a diagonal matrix, and ${\zeta_2}$ follows since $\Pi_{B}^{1/2} B \Pi_B^{-1/2}$ (in $\zeta_1$) is a symmetric matrix and we have enlarged the set over which minimization is being performed. The last equality follows from the observation that for any matrix $A \in \R^{n\times n}$, we have 
$$ \min_{\substack{C \in \R^{n \times n}:\\ C=C^\T} } \|A -C\|_\F =  \frac{1}{2} \|A - A^\T\|_\F.$$
Now we will show that there exists a constant $c$ such that
$$
\begin{aligned}
 \min _{\pi_{B} \in \BTLh}&  \|\Pi_{B}^{1 / 2} P \Pi_{B}^{-1 / 2}  - \Pi_{B}^{-1/2} P^{\T} \Pi_{B}^{1 / 2} \|_{\F} \\
& \ \ \ \ \ \ \ \ \ \geq  \frac{c \xi^2}{h^{5/2}} \| \Pi^{1 / 2} P \Pi^{-1 / 2}-\Pi^{1 / 2} P^{\T} \Pi^{-1/2} \|_{\F}.
\end{aligned}
$$
Note that it is sufficient to show that there exists a constant $\tilde{c}$ such that 
\begin{equation} \label{mainpointa}
\min_{\pi_{B} \in \BTLh} \|\Pi_{B} P-P^{\T} \Pi_{B}\|_{\F} \geq  
 \frac{\tilde{c} \xi^2}{h^{1/2}} \|\Pi P-P^{\T} \Pi\|_{\F}.
\end{equation}
This is because if \cref{mainpointa} is true, then 
\begin{align*}
\min_{\pi_{B} \in \BTLh} & \|\Pi_{B}^{1 / 2} P \Pi_{B}^{-1 / 2}  - \Pi_{B}^{-1 / 2} P^{\T} \Pi_B^{1 / 2}\|_{\F} \\ 
& \geq  \min_{\pi_{B} \in \BTLh}  \frac{\|\Pi_{B} P-P^{\T} \Pi_{B}\|_{\F}}{\|\pi_{B}\|_{\infty}} \\ 
& \stackrel{\zeta_1}{\geq} \min_{\pi_{B} \in \BTLh} \frac{n}{h}  \|\Pi_{B} P-P^{\T} \Pi_{B}\|_{\F} \\
& \stackrel{\zeta_2}{\geq}   \frac{\tilde{c} n \xi^2 }{h^{3/2} } \|\Pi P -P^{\T} \Pi\|_{\F} \\
& \stackrel{\zeta_2}{\geq} \frac{\tilde{c} n \xi^2 \pi_{\min} }{h^{3/2} } \|\Pi^{1 / 2} P \Pi^{-1 / 2}-\Pi^{-1 / 2} P^{\T} \Pi^{1/2}\|_{\F} \\
& \stackrel{\zeta_3}{\geq} \frac{\tilde{c} \xi^2  }{h^{5/2} }  \|\Pi^{1 / 2} P \Pi^{-1 / 2}-\Pi^{-1 / 2} P^{\T} \Pi^{1/2}\|_{\F},
\end{align*}
where $\zeta_1$ follows since $\pi_B \in \BTLh$, hence $\pi_{B, i} \geq \|\pi_{B}\|_{\infty}/h$ which implies that $\|\pi_{B}\|_{\infty} \leq \frac{h}{ n}$, $\zeta_2$ follows from our assumption in \cref{mainpointa} (which we will prove next) and we define $\pi_{\min} \triangleq \min_{i \in [n]} \pi_i$, and  $\zeta_3$ follows since $\pi_{\min} \geq \|\pi\|_\infty/h$ and $n \|\pi\|_\infty \geq 1$.

Now it remains to prove \cref{mainpointa}. To do so, observe that
$$ \min_{\pi_{B} \in \BTLh} \|\Pi_{B} P-P^{\T} \Pi_{B}\|^2_{\F} \geq \min_{ \substack{\pi_{B} \in \R_+^n: \\ \pi_{B}^\T \1_n =1} } \|\Pi_{B} P-P^{\T} \Pi_{B}\|^2_{\F}. 
$$

 Assume that $\pisb$ is a solution to the following minimization problem
$$\min_{\pi_{B} \in \R_+^n: \pi_{B}^\T \1_n =1} \|\Pi_{B} P-P^{\T} \Pi_{B}\|^2_{\F},
$$
and let $\Pisb \triangleq \diag(\pisb)$. 
To prove \cref{mainpointa}, we will show that the following ratio is upper bounded by $O(\sqrt{h}/\xi^2)$:
\begin{equation} \label{ratio}
\frac{\|\Pi P - P^{\T} \Pi\|_{\F} }{\|\Pisb P - P^{\T} \Pisb\|_{\F}} = \sqrt{\frac{\pi^{\T} R \pi}{\pi_{B}^{* \T} R \pi_{B}^{*}}}, 
\end{equation}
where the matrix $R$ is defined as follows
\begin{equation}
R_{i j} \triangleq \begin{cases}-p_{i j} p_{j i}, & i \neq j \text{ and } (i,j) \in \calE \\ \sum\limits_{\substack{l:l\neq i :\\ (i,l) \in \calE}} p_{i l}^{2}, & i=j \\ 0, &  \text{otherwise} \end{cases}.    
\end{equation}
Note that $R$ is symmetric and is in fact a positive semidefinite matrix (as $u^\T R u = \sum_{(i,j) \in \calE}(u_ip_{ij} -u_jp_{ji})^2$). Therefore, all its eigenvalues are non-negative. Moreover, when the canonical Markov matrix corresponding to $P$ is not reversible, $R$ is, in fact, positive definite. The positive definiteness of $R$ follows from the fact that if $u^\T R u = 0$ for some vector $u$, then $u_i p_{ij} = u_j p_{ji}$ for all $(i,j) \in \mathcal{E}$. By the connectedness of the graph $\G$, this implies that $u \geq 0$ entrywise (without loss of generality). Consequently, the canonical Markov matrix $S$ corresponding to $P$ is reversible, which leads to a contradiction. 
Additionally, when the canonical Markov matrix corresponding to $P$ is reversible, then \cref{mainpointa} trivially holds as both sides of the inequality are zero. Therefore, our main focus will be on the case when the matrix $P$ is not reversible. First, to begin, let us focus on the denominator term $\pi_{B}^{*^{\T}} R \pi_{B}^{*}$. Recall that $\pisb$ is the solution of the following optimization problem:
$$
\min _{\substack{\pi_{B} \in \R_+^n: \pi_{B}^\T \1_n =1  }} \frac{1}{2} \pi_{B}^{\T} R \pi_{B}.
$$
For $u\geq \0$ and $\lambda \in \R,$ we obtain the Lagrangian    
$$
L(\pi_B,\lambda, u) \triangleq \frac{1}{2} \pi_{B}^{\T} R {\pi_{B}}-\lambda\left(\pi_{B}^{\T}\1_n-1\right) - u^\T \pi_B.
$$
Using Karush-Kuhn-Tucker conditions, this gives the following optimality conditions (primal feasibility, dual feasibility, stationarity, and complementary slackness): 
$$ 
\begin{aligned}
    \pi_{B}^{*} \geq 0,\  \pi_{B}^{*^\T} \1_n = 1,\ u^* \geq \0,\ R \pi_{B}^{*} - \lambda^{*} \1_n = u^*, \ \\  \text{and} \ \forall i\in [n], \, u_i^* (\pi_{B}^*)_i = 0,   
\end{aligned}
$$
where $\lambda^*, u^*$ is the optimal dual solution. 
We claim that when $R$ is positive definite, the primal and dual solution satisfying the above optimality conditions are given by:
$$\pi_{B}^{*} = \frac{R^{-1} \1_n }{ \1_n^\T R^{-1} \1_n }, \lambda^* = \frac{1 }{ \1_n^\T R^{-1} \1_n }, u^* = \0  .$$
Here, $\pi_{B}^{*}$ is non-negative, because $R^{-1}$ is entrywise non-negative. To see this latter fact, we use the inverse-positivity property of $R$, because $R$ is a special kind of $Z$-matrix (i.e., $R_{ij} \leq 0$ for $i \neq j$), namely, an $M$-matrix \cite{ZmatrixMmatrix}.  
Moreover, the optimality conditions imply that
\begin{equation}
    \pi_{B}^{* \T} R \pi^{*}_B=\lambda^{*}. \label{denominator terms}
\end{equation}
Equivalently, the optimality condition implies that $\pi_{B}^{*}$ is an eigenvector of the matrix $S_{B, \lambda^{*}}$ with eigenvalue $1$, where $S_{B, \lambda}$  is a matrix with parameter $\lambda $ and is defined as
\begin{equation} \label{defnsblambda}
    S_{B, \lambda } \triangleq I_n + \frac{1}{n}(\lambda\1_n\1_n^\T - R), 
\end{equation}
where $I_n$ is the identity matrix of size $n$. This is because the optimality condition implies that  ${\pi_{B}^{*}}^{\T} = {\pi_{B}^{*}}^{\T}\big(I_n+   \frac{\lambda^*\1_n \1_n^\T -R}{n} \big)$. Moreover, note that each entry of $S_{B, \lambda^* }$ is strictly greater than $0$ for $\lambda^* > 0$, and therefore, $\pi_{B}^{*}$ is the Perron vector of $S_{B, \lambda^* }$ with eigenvalue 1. Now, consider the numerator term in \cref{ratio}, which can be expressed as
\begin{align}
& \pi^{\T} R \pi =  \pi^{\T} R \pi_{B}^{*}+\left(\pi-\pi_{B}^{*}\right)^{\T} R\left(\pi-\pi_{B}^{*}\right)+\pi_{B}^{* \T} R\left(\pi-\pi_{B}^{*}\right) \nonumber\\
& =  \lambda^{*} \pi^{\T}\1_n + (\pi-\pi^{*}_{B})^{\T} R(\pi-\pi_{B}^{*}) + \lambda^{*}\1_n^\T (\pi-\pi_{B}^{*})\nonumber\\
& =  \lambda^{*}+\left(\pi-\pi^{*}_{B}\right)^{\T} R\left(\pi-\pi_{B}^{*}\right)  \leq  \lambda^{*} + \lambda_{\max }(R)\|\pi-\pi^{*}_B\|_2^{2}, \label{combining3}
\end{align}
where $\lambda_{\max}(R)$ is maximum eigenvalue of $R$. Note that by Gershgorin circle theorem \cite{horn2012matrix}, we have $\lambda_{\max}(R) \leq \max_{i} \sum_{j=1}^n |R_{ij}| \leq 2d_{\max}$. Now, we will upper bound $\|\pi - \pisb\|_2$. Since, $\pi$ and $\pisb$ are Perron vectors of $S$ and $S_{B, \lambda^*}$ with eigenvalues 1, therefore we have
$$
\begin{aligned}
\pi^{\T}-\pi_{B}^{*^\T}= & \pi^{\T} S - \pi_{B}^{*^\T} S_{B, \lambda^{*}} \\ 
& =  \left(\pi-\pi_{B}^{*}\right)^{\T}\left(S-\1_n\pi^{\T}\right)+\pi_{B}^{*^\T}(S-S_{B, \lambda^{*}}) 
\end{aligned}
$$
{Taking norm $\|\cdot\|_{\pi^{-1}}$ on both sides and using the triangle inequality gives}
        \begin{align} 
            \|\pi -\pi_{B}^{*}\|_{\pi^{-1}} & \leq \|\pi -\pi_{B}^{*}\|_{\pi^{-1}} \| S - \1_n \pi^\T\|_{\pi^{-1}} \nonumber\\ 
            & \ \ \ \ \ \ \ \ \ \ \ \ \ \ \ \ \ + \| \pi_B^{*^\T}(S - S_{B,\lambda^*}) \|_{\pi^{-1}}.  
        \end{align}
It is straightforward to verify that (see \cref{subsec:Preliminaries} for more details)
        $$\|S- \mathbf{1}_n\pi^\T\|_{\pi^{-1}} = \| \Pi^{1/2} S\Pi^{-1/2}  -\sqrt{\pi} \sqrt{\pi}^\T \|_2 . $$ 
Rearranging the terms and utilizing \cref{lem:SpectralGapLB for a general graph} we get the following bound 
\begin{align}
\|\pi\!-\!\pisb \|_2 & \leq 4 \sqrt{ \frac{\|\pi\|_{\infty}}{\pi_{\min}} } \frac{\|\pisb^{\T}(S-S_{B, \lambda^{*}})\|_2}{ \xi^2 } \nonumber  \\ 
& \leq \frac{4\sqrt{h} }{\xi^2} \bigg( \!\|\pisb^{\T}(S-S_{B, 0})\|_2 + \frac{\lambda^*}{n} \|\pisb^\T \1_n\1_n^\T\|_2\! \bigg) \nonumber\\
& \leq \frac{4 \sqrt{h}}{\xi^2 } \bigg(\|\pisb^{\T}(S-S_{B, 0})\|_2 + \frac{\lambda^*}{\sqrt{n}} \bigg),  \label{combining1}
\end{align}
where $S_{B, 0}$ is the matrix in \cref{defnsblambda} with $\lambda = 0$. Now consider the $i$th term of the $\pisb^{\T}(S-S_{B, 0})$ as 
\begin{align*}
& (\pisb^{\T}( S-S_{B, 0}))_i  =\frac{1}{d} \bigg( \pi_{B,i}^* \sum_{\substack{ j: j\neq i,\\ (i,j) \in \calE}} \left(\frac{d}{n} p_{ij}^2 - p_{ij}\right)  \\ 
& \ \ \ \ \ \ \ \ \ \ \ \ \ \ \ \ \ \ \ \ \ \ \ + \sum_{\substack{ j: j\neq i,\\ (i,j) \in \calE}} \pi_{B,j}^*\left(p_{ji} - \frac{d}{n}p_{ij} p_{ji}\right) \bigg) \\
& = - \frac{1}{d} \sum_{\substack{ j: j\neq i,\\ (i,j) \in \calE}} (\pi_{B,i}^* p_{ij} - \pi_{B,j}^*p_{ji})\left(1-\frac{d}{n}p_{ij} \right).
\end{align*}
Since, for any graph we have $d \leq 2n$, we obtain $|1-\frac{d}{n}p_{ij}| \leq 1$. Therefore, using the above bound $\|\pisb^{\T}(S-S_{B, 0})\|_2^2$ can be bounded as
\begin{align}
\|\pisb^{\T}(S-S_{B, 0})\|_2^2 & \leq \frac{1}{d^2} \sum_{i=1}^n \bigg(\sum_{\substack{ j: j\neq i,\\ (i,j) \in \calE}} |\pi_{B,i}^* p_{ij} - \pi_{B,j}^*p_{ji}|\bigg)^2 \nonumber \\ 
& \leq \frac{d_{\max}}{d^2} \sum_{(i,j) \in \calE}  (\pi_{B,i}^* p_{ij} - \pi_{B,j}^*p_{ji})^2\nonumber \\
& = \frac{1}{2d} \|\Pisb P - P^\T \Pisb\|_\F^2 = \frac{\lambda^{*}}{2d}. \label{combining2}
\end{align}
Combining \cref{combining1,combining2}, we get
\begin{equation*}
\|\pi-\pisb \|_2 \leq \frac{4\sqrt{h} }{\xi^2} \bigg(\frac{\lambda^*}{\sqrt{n}} + \sqrt{\frac{\lambda^*}{2d}} \bigg ).    
\end{equation*}
Thus, from \cref{combining3} and the fact that $\lambda_{\max}(R) \leq 2d_{\max}$, there exists a constant $c$ such that we have
\begin{align}
    \pi^{\T}R\pi & \leq \lambda^* + \frac{c h}{\xi^4} \lambda^{*} d_{\max} \bigg(\frac{1}{\sqrt{d}} + \sqrt{\frac{\lambda^*}{n}}\bigg)^2 \nonumber\\
    & \leq \lambda^* + \frac{c h}{\xi^4}\lambda^* \left(1+\sqrt{\lambda^{*}}\right)^2. \label{numerator terms}
\end{align}
Therefore, using \cref{denominator terms,numerator terms}, the ratio in \cref{ratio} is upper bounded as 
\begin{equation} 
\frac{\|\Pi P - P^{\T} \Pi\|_{\F} }{\|\Pisb P - P^{\T} \Pisb\|_{\F}} \leq 1 + O\bigg(\frac{ \sqrt{h} }{\xi^2}\sqrt{\lambda^*} \bigg).
\end{equation} 
Finally, it remains to show that $\lambda^* = 1/\1_n^\T R^{-1} \1_n$ is upper bounded by a constant. {We can show this using a spectral bound $\lambda^* \leq \lambda_{\max}(R)/n \leq 2d_{\max}/n \leq 2$.} 

Alternatively, we can also show this through the following argument. Note that $S_{B, 0}$  is a symmetric matrix with absolute eigenvalues strictly less than $1$ (when the canonical Markov matrix corresponding to $P$ is not reversible). This is because for any vector $u \in \R^n$ with $\|u\|_{2}=1$, we have
$$
u^{\T} S_{B, 0} u =u^{\T} u-\frac{1}{n} \sum_{(i, j) \in \calE}\left(u_{i} P_{i j}-u_{j} P_{j i}\right)^{2}<u^{\T} u.
$$
This implies that when the canonical Markov matrix of $P$ is reversible, then the Perron-Frobenius eigenvalue of $S_{B,0}$ is $1$ (otherwise, it is strictly less than one). Also, recall from the Perron-Frobenius theorem that if $0 \leq \mathrm{A} < \mathrm{B}$ entrywise, then $\lambda_{\max}(A) \leq \lambda_{\max}(B)$. Moreover, if $\mathrm{B}$ is irreducible, then the inequality is strict: $\lambda_{\max}(A)<\lambda_{\max}(B)$. Observe that since the induced graph $\G$ is strongly connected, therefore irreducibility holds. Moreover, since for any $\lambda > 0$, $S_{B,\lambda} > S_{B,0}$ entrywise, therefore we have $\lambda_{\max}(S_{B,\lambda}) > \lambda_{\max}(S_{B,0}) $. Thus, $\lambda^{*}$ in $S_{B,\lambda^{*}} $ is the smallest constant such that $\lambda_{\max}(S_{B, \lambda^{*}} )=1$ (as the spectral radius of $S_{B,\lambda^*}$ is 1). Recall that for an entrywise positive matrix $A$ its Perron-Frobenius eigenvalue is lower bounded as $|\lambda_{\max}(\mathrm{A})| \geq \min_{i \in [n]} \sum_{j=1}^n a_{ij}$. Hence, we utilize this condition to show the existence of a constant $\lambda_0 \geq \lambda^*$, such that $\lambda_{\max}(S_{B,\lambda_0}) \geq 1$, as
$$ 
\begin{aligned}
    \lambda_{\max}(S_{B,\lambda_0}) & \geq  \min_{i \in [n]} \bigg(1 + \lambda_0 +  \frac{1}{n} \sum_{\substack{ j: j\neq i,\\ (i,j) \in \calE}} (p_{ij}p_{ji} - p_{ij}^2) \bigg) \\ 
    & \geq \lambda_0 + 1 - \frac{d_{\max}}{n}.
\end{aligned}
$$
Setting $\lambda_0 = \frac{d_{\max}}{n}$ ensures $\lambda_{\max}(S_{B,\lambda_0}) \geq 1$. Thus, $\lambda^* \leq \frac{d_{\max}}{n} \leq 1$,  proving the theorem. \qed

\section{Proofs of Upper Bounds} \label{Proof:Main theorem with upper bound}
    This section is devoted to the proofs of various lemmata and existing results needed to prove \cref{Main theorem with upper bound,thm:Type-1 and Type-2 error}. The main portion of the proof of \cref{Main theorem with upper bound} is presented in \cref{sec:Upper Bound on Critical Threshold} and proof of \cref{thm:Type-1 and Type-2 error} is presented in \cref{Proof of Type-1 error}. But first, we need to establish some key results discussed in the following section.

\subsection{Preliminaries} \label{subsec:Preliminaries}
    In this section, we will prove key lemmata that will be used quite frequently to develop the proof of the main result in \cref{Main theorem with upper bound}.         
        The following lemma is similar to \cite[Theorem 8]{Chenetal2019} but also holds when the canonical Markov matrix of $P$ is not reversible. 
        \begin{lemma}[Eigenvector Perturbation] \label{lem:PertubationOfPi}
          \sloppy Let $\pi, \pihat, \tilde{\pi}$ be the stationary distributions of the row stochastic matrices $S, \Shat, \tilde{S}$, respectively such that $\pi >0$ entrywise. Then, if $\| \Pi^{\frac{1}{2}} S \Pi^{-\frac{1}{2}}  -\sqrt{\pi} \sqrt{\pi}^\T \|_2 + \|S - \hat{S} \|_{\pi^{-1}} < 1$, we have
         $$
              \|\pihat - \tilde{\pi} \|_{\pi^{-1}} \leq \frac{\| \tilde{\pi}^\T(\tilde{S} - \Shat )\|_{\pi^{-1}}}{1- \| \Pi^{\frac{1}{2}} S \Pi^{-\frac{1}{2}}  -\sqrt{\pi} \sqrt{\pi}^\T \|_2 - \|S - \hat{S} \|_{\pi^{-1}}} . $$
        \end{lemma}

        \begin{proof}
        By stationarity of $\tilde{S}$ and $\Shat$, we have
        \begin{align}
            & \tilde{\pi}^\T - \pihat^\T  = \tilde{\pi}^\T \tilde{S} -  \pihat^\T\Shat  \nonumber\\
            & = \tilde{\pi}^\T(\tilde{S} - \Shat)  + (\tilde{\pi} - \pihat )^\T \hat{S}- (\tilde{\pi} - \pihat)^\T \mathbf{1}_n \pi^\T \nonumber\\
            & = \tilde{\pi}^\T(\tilde{S} - \Shat)  + (\tilde{\pi} - \pihat )^\T(\Shat - \S) + (\tilde{\pi} - \pihat)^\T(S- \mathbf{1}_n\pi^\T) .\label{onetotwo}
        \end{align}
        Taking $\ell^2(\pi^{-1})$ norm on both sides and using the triangle inequality gives
        \begin{align}
            \| \tilde{\pi} - \pihat \|_{\pi^{-1}} & \leq \| \tilde{\pi}^\T(\tilde{S} - \Shat) \|_{\pi^{-1}} + \|\tilde{\pi} - \pihat \|_{\pi^{-1}} \| \hat{S} - S\|_{\pi^{-1}} \nonumber\\
            & \ \ \ \ \ + \|\tilde{\pi}  - \pihat \|_{\pi^{-1}} \|S- \mathbf{1}_n\pi^\T\|_{\pi^{-1}} .  \label{similar1}
        \end{align}
        It is straightforward to verify that 
        $$\|S- \mathbf{1}_n\pi^\T\|_{\pi^{-1}}\! =\!\| \Pi^{\frac{1}{2}} S\Pi^{-\frac{1}{2}}  -\sqrt{\pi} \sqrt{\pi}^\T \|_2 = \sigma_2( \Pi^{\frac{1}{2}} S\Pi^{-\frac{1}{2}}),$$
        where $\sigma_2(M)$ denotes the second-largest singular value of $M$, and the last equality follows since $\sqrt{\pi}$ is both the left and right top singular vector of the DTM $\Pi^{1/2} S\Pi^{-1/2}$ \cite[Proposition 2.2]{makur2019}. 
        Thus, rearranging the terms in the above inequality establishes the statement of \cref{lem:PertubationOfPi}. 
    \end{proof}

    We would like to emphasize the distinction between our proof above and the approach adopted in \cite[Theorem 8]{Chenetal2019} for proving a similar result, and elucidate the proof in \cite{Chenetal2019}. In transitioning from \cref{onetotwo} to \cref{similar1}, the authors in \cite{Chenetal2019} utilize the $\ell^2(\pi)$ norm, as opposed to the $\ell^2(\pi^{-1})$ norm. Specifically, they bound the final term on the right-hand side as follows 
    \begin{equation}\label{authors claim} 
    \left\|(\tilde{\pi}  - \pihat)^{\T}\left(S-\1_{n} \pi^{\T}\right)\right\|_{\pi} \leq \|\tilde{\pi}  - \pihat\|_{\pi} \lambda_2(S),
    \end{equation}
    where $\lambda_l(S)$ is the $l$th largest eigenvalue (in magnitude) of $S$ for $l\in [n]$.
    This bound is rather subtle and a detailed reasoning for \cref{authors claim} is missing. Therefore, we provide it below. Define the function: 
    \begin{equation}
        \rho(S) \triangleq \max_{\substack{ e: \|e\|_\pi \leq 1\\ e^\T\1_n = 0} } e^\T (S-\1_n \pi^\T) \Pi (S-\1_n \pi^\T)^\T e. 
    \end{equation}
    Observe that the maximization is the same as
    \begin{equation}
    \rho(S)=\max _{\substack{e: \|e\|_2 \leq 1\\ e^\T\pi^{-1/2} = 0}} e^{\T} \Pi^{-1/2} (S-\1_n \pi^\T)  \Pi (S-\1_n \pi^\T) ^{\T} \Pi^{-1/2} e.
    \end{equation}
    Also, observe that the optimal vector $e^*$ achieving the maximum in the above problem is orthogonal to $\pi^{3/2}$. This is because any component of $e^*$ in the direction of $\pi^{3/2}$ lies in the left nullspace of $S - \1_n \pi^\T$. Therefore, $\rho(S)$ can be simplified as   
    \begin{equation}
    \rho(S)=\max _{\substack{e: \|e\|_2 \leq 1\\ e^\T\pi^{-1/2} = 0,\ e^\T \pi^{3/2} = 0 }} e^{\T} \Pi^{-1/2} S \Pi S^{\T} \Pi^{-1/2} e.
    \end{equation}
    Finally, $\rho(S) = \lambda_2(\Pi^{-1 / 2} S \Pi^{1 / 2})$ follows by observing that $\pi^{-1 / 2}$ and $\pi^{3 / 2}$ are the corresponding right and left eigenvectors of $\Pi^{-1 / 2} S \Pi^{1 / 2}$. And since for the BTL model (under hypothesis $H_0$) $\Pi^{1 / 2} S \Pi^{-1 / 2}$ is a symmetric matrix, it has real eigenvalues. This implies that, by a similarity transform, $S, \Pi^{1 / 2} S \Pi^{-1 / 2}, \Pi^{-1 / 2} S \Pi^{1 / 2}$ all have the same eigenvalues. This proves the bound in \cref{authors claim}. However, this technique does not work for general pairwise comparison models, and therefore we have to resort to $\ell^2(\pi^{-1})$ norm.

    One of our main goals for deriving \cref{lem:PertubationOfPi} is to find upper bounds on $\|\pihat - \pi\|_2$, where $\pihat, \pi$ are stationary distributions of $\Shat, S$, respectively. To achieve this, we will employ \cref{lem:PertubationOfPi} (with the choice $\tilde{S} = S$ and thus, we have $\tilde{\pi} = \pi$). However, to apply this lemma, it is essential to demonstrate that the condition outlined in \cref{lem:PertubationOfPi} is satisfied. In other words, we need to find upper bounds on the terms $\| \Pi^{1/2} S\Pi^{-1/2}  -  \sqrt{\pi} \sqrt{\pi}^\T \|_2 $ and $ \|\Shat - S\|_{\pi^{-1}}$, ensuring that their sum is less than $1$. When the Markov chain corresponding to $S$ is reversible, the DTM matrix $R = \Pi^{1/2} S\Pi^{-1/2}$ is symmetric, and hence, 
    
    \begin{equation}
    \|R -\sqrt{\pi} \sqrt{\pi}^\T \|_2 =  \lambda_2(S) .
    \end{equation}    
    Moreover, since the Markov chain is irreducible by the Perron-Frobenius theorem, we have $ \lambda_2(S) < 1$, and the corresponding upper bounds have been derived in \cite{NegahbanOhShah2012,Chenetal2019}. However, we are interested in the general case where $S$ need not be reversible (as the underlying pairwise comparison matrix may not be BTL). Hence, we bound each of these terms using the following two lemmata.

    \begin{lemma}[Spectral Norm of Noise] 
    \label{lem:NoiseLB}
        For $\Shat$ constructed as in \cref{EmpiricalMarkovChain}, we have
        \begin{equation*}
            \| \Shat -S\|_{\pi^{-1}} \leq  \sqrt{ \frac{\| \pi\|_\infty }{\pi_{{\min}}} } \|\Shat - \S\|_2 \leq 3 \sqrt{\frac{h \log n}{k \dmax } }, 
        \end{equation*}    
        with probability at least $1 - O(n^{-3} )$ and where $\pi_{{\min}} \triangleq \min_{i \in [n]} \pi_i$.
    \end{lemma}
    The proof of \cref{lem:NoiseLB} is quite similar to \cite{NegahbanOhShah2012} but requires slight modifications and is provided in \cref{Proof of NoiseLB} for completeness.  

    Next we upper bound the quantity $\|R -\sqrt{\pi} \sqrt{\pi}^\T \|_2$. Observe that $\|  R -\sqrt{\pi} \sqrt{\pi}^\T \|_2$ is the second largest singular value of the DTM $R$ (as $\sqrt{\pi}$ are both the left and right singular vectors of $R$ corresponding to largest singular value of $1$ \cite[Proposition 2.2]{makur2019}). The second largest singular value of DTM $R$ is also equal to the square root of the contraction coefficient for $\chi^2$-divergence $\eta_{\chi^2}(\pi, \S)$ of the source-channel pair $(\pi, S)$ (see \cite{sarmanov1958, makur2019, MakurZheng2020} for definitions and details). For the case of the complete graph, we can upper bound this quantity by upper bounding the Dobrushin contraction coefficient for TV distance, as demonstrated in the following lemma.
    \begin{lemma}[Spectral Norm Bound for Complete Graph]        \label{lem:SpectralGapLB}
        In the case of a complete graph, let $d = 2n$. Then the following bounds hold on the second largest singular value of the DTM, $R = \Pi^{1/2} S\Pi^{-1/2}$:
            $$ \left\|\Pi^{1/2} S\Pi^{-1/2} -\sqrt{\pi} \sqrt{\pi}^{\T}\right\|_2 =  
            \sqrt{\eta_{\chi^2}(\pi, \S)} \leq  1 - \frac{\delta}{4(1+\delta)}. $$ 
    \end{lemma}
    
    \begin{proof} 
        We will find an upper bound on $\sqrt{\eta_{\chi^2}(\pi, \S)}$ using \cite[Proposition 2.5]{makur2019}:
        $$
        \eta_{\chi^2}(\pi, S) \leq \eta_{\mathsf{TV}}(S) 
         = \max_{i, j \in [n] } \  \left\|S_{i, :}- S_{j, :}\right\|_{\mathsf{TV}} ,
       $$
        where 
        $\eta_{\mathsf{TV}}(S)$ is the Dobrushin contraction coefficient for TV distance $\|\cdot\|_{\mathsf{TV}}$. Note that $S_{ii} \geq 1/2$ and $S_{ij} \leq 1/(2n)$ for $i,j \in [n]$ with $i \neq j$. We can use \cref{Pijbounded} to bound as $\|S_{i, :}- S_{j, :}\|_{\mathsf{TV}}$ between any pair $i$ and $j$ as
        \begin{align*}
              \left\|\S_{i, :}-\S_{j, :}\right\|_{\mathsf{TV}} 
              &= 1 - \sum_{k=1}^n \min\{ S_{ik}, S_{jk} \} \\ 
             & = 1- \left( \sum_{k:k\neq i, j} \frac{ \min\{p_{ik} , p_{jk}\}}{2n} +  \frac{p_{ij} + p_{ji}}{2n} \right)\\
            & \leq 1 - \frac{\delta}{2(1+\delta)}. 
        \end{align*}
        Hence, the lemma holds since for $0 \leq x \leq 1$, we have $\sqrt{1-x}  \leq 1- \frac{x}{2}$.
    \end{proof}

    In the more general case of an arbitrary graph (consistent with our assumptions), we upper-bound the second largest singular value of $R$ by leveraging edge expansion properties and the Cheeger inequality for non-negative matrices, as demonstrated in the following lemma.
    
    \begin{lemma}[Spectral Norm Bound for a General Graph]     \label{lem:SpectralGapLB for a general graph}
        Let $\G$ be the induced graph corresponding to the canonical Markov matrix $S$. Consider the DTM $R = \Pi^{1/2} S\Pi^{-1/2}$, with edge expansion of $S$ lower bounded by $\xi$, i.e., $\phi(S) \geq \xi$, then the following bounds hold on the second largest singular value of $R$:
            $$ \left\|\Pi^{1/2} S\Pi^{-1/2} -\sqrt{\pi} \sqrt{\pi}^{\T}\right\|_2 \leq 1 -\frac{1}{4} \xi^2. $$ 
    \end{lemma}
    \begin{proof}
        Observe that 
        \begin{align*}
            \sigma_2(R) & =\sqrt{\lambda_2\left(R R^{\T}\right)} \stackrel{\zeta_1}{\leq} \sqrt{\lambda_2\left(\frac{R+R^{\T}}{2}\right)} \\
            & \stackrel{\zeta_2}{\leq} \sqrt{1-\frac{1}{2} \phi^2\left(\frac{R+R^{\T}}{2}\right)} \\
            & \stackrel{\zeta_3}{\leq} 1-\frac{1}{4} \phi^2(R)  \stackrel{\zeta_4}{\leq} 1-\frac{1}{4} \xi^2,
        \end{align*}
        where $\zeta_1$ follows by a standard argument and the explanation is provided below. $\zeta_2$ is a consequence of the Cheeger inequality for non-negative matrices in \cref{lem:Cheeger's Inequality for Nonnegative Matrices Satisfying Detailed Balance} (see below). $\zeta_3$ follows since $\sqrt{1- x} \leq 1-x/2 $ for $x\geq 0$, and since $\phi(R) = \phi( (R+R^\T)/2 ) = \phi(R^\T)$ (the proof is provided below for completeness). Finally, $\zeta_4 $ follows because $\phi(R) = \phi(S) \geq \xi,$ by \Cref{assm:Large Expansion of DTM}. Regarding $\zeta_1$, observe that since the matrix $R$ is $\frac{1}{2}$-lazy, i.e., $R_{ii} \geq \frac{1}{2}$ for all $i \in [n]$ (as $S_{ii} \geq \frac{1}{2}$), therefore we have
        \begin{equation*}
            R R^\T  =\frac{R+R^\T}{2}+\frac{(2 R - I)\left(2 R^\T -I\right)}{4}-\frac{I}{4}. 
        \end{equation*}
        Using the above relation, we obtain
            \begin{align} 
                 \nonumber x^\T R R^\T x & \leq \frac{x^\T(R+R^\T)x}{2}  \\ 
                & \nonumber \ \ \ \ \ \ \ \ + \|x\|^2_2 \frac{\|(2 R-I)(2 R^\T-I)\|_2}{4}- \frac{\|x\|_2^2}{4} \\ 
                &  \stackrel{\zeta}{\leq}  \frac{x^\T(R+R^\T) x}{2} \nonumber, 
            \end{align}
        where $\zeta$ follows since $(2R - I)(2R^{\T} - I)$ is a doubly non-negative matrix (as $R$ is $\frac{1}{2}$-lazy) with largest eigenvalue (and, hence, singular value) of $1$ with eigenvector $\sqrt{\pi}$, which gives $\|(2 R-I)(2 R^\T-I)\|_2 = 1$. 
        This gives
        \begin{equation}
             \max_{x: \|x\|_2 \leq 1, x \perp \sqrt{\pi}}  x^\T R R^\T x   \leq \max_{x: \|x\|_2 \leq 1, x \perp \sqrt{\pi}} \frac{x^\T(R+R^\T)x}{2}. \label{standard argument}
        \end{equation}
        Hence, by \cref{standard argument} and variational characterization of eigenvalues, we have 
        $$
            \lambda_2\left(R R^\T\right) \leq \lambda_2\left(\frac{R+R^\T}{2}\right).
        $$
        Finally, $\phi(R) =\phi( (R+R^\T)/2)$ follows by a little algebra. Observe that both $R$ and $(R+R^\T)/2$ share the same left and right singular vectors $\sqrt{\pi}$. Moreover, the numerator term in \cref{original expression for expansion} can be simplified as 
        $$
        \begin{aligned}
            \1_\calS^\T D_u R D_v \1_{\calS^\complement}  & = \1_\calS^\T \Pi^{1/2} (\Pi^{1/2} S \Pi^{-1/2}) \Pi^{1/2} \1_{\calS^\complement} \\
            & = \1_\calS^\T \Pi S  \1_{\calS^\complement} = \1_\calS^\T \Pi S (\1_n - \1_\calS) \\
            & = \1_\calS^\T \pi - \1_\calS^\T \Pi S \1_\calS , 
        \end{aligned}
        $$ 
        which is the same as 
        $$
        \begin{aligned}        
        \1_\calS^\T D_u R^\T D_v \1_{\calS^\complement} & = \1_\calS^\T \Pi^{1/2} (\Pi^{-1/2} S^\T \Pi^{1/2}) \Pi^{1/2} \1_{\calS^\complement} \\ 
        & = \1_\calS^\T  S^\T \Pi \1_{\calS^\complement} = \1_\calS^\T  S^\T \Pi (\1_n - \1_\calS) \\
        & = \1_\calS^\T \pi - \1_\calS^\T  S^\T\Pi \1_\calS \\
        & = \1_\calS^\T \pi - \1_\calS^\T  \Pi S \1_\calS = \1_\calS^\T D_u R D_v \1_{\calS^\complement}. 
        \end{aligned}
        $$
        The equivalence of the denominator terms in \cref{original expression for expansion} follows from \cref{exact expression for edge expansion of DTM}, thus proving $\phi(R) =\phi( (R+R^\T)/2) = \phi(R^\T)$. 
        \end{proof}

        \begin{lemma}[Cheeger Inequalities for Non-negative Matrices Satisfying Detailed Balance {\cite[Theorem 15]{MehtaSchulman}}] \label{lem:Cheeger's Inequality for Nonnegative Matrices Satisfying Detailed Balance}
        Consider a non-negative matrix $M$ with a Perron-Frobenius eigenvalue of $1$ and positive left and right eigenvectors $u$ and $v$. Assume that $M$ satisfies the condition of detailed balance, i.e., $D_u M D_v = D_v M^\T D_u$ where $D_u = \diag(u)$ and $D_v = \diag(v)$. Then, the following inequalities hold:
         \begin{equation*}
                    1 - \lambda_2(M) \leq  \phi(M)  \leq \sqrt{2(1 - \lambda_2(M))}.
         \end{equation*}
         \end{lemma}  
    Combining \cref{lem:NoiseLB} and \cref{lem:SpectralGapLB} we obtain the following corollary.
    
    \begin{corollary}[Spectral Gap] \label{combinedupperbound}
        For $k \geq \max\!\big\{1,\frac{c h \log n}{d_{\max} \xi^4 } \big\}$ for some constant $c$, the following bound holds with probability at least $1 - O(n^{-3})$:   
        $$
            \begin{aligned}
                1 - \|\Pi^{1/2} S\Pi^{-1/2}  -\sqrt{\pi} \sqrt{\pi}^\T \|_2 & -  \|\Shat -  S \|_\pi \geq \frac{\xi^2}{8} . 
            \end{aligned}
        $$
    \end{corollary}
    \begin{proof}
    Observe that
    $$
    \begin{aligned}
        1 - \| \Pi^{1/2} S \Pi^{-1/2}  - &\sqrt{\pi} \sqrt{\pi}^\T \|_2 -  \|\Shat -  S \|_\pi  \\
& \geq  \frac{1}{4}\xi^2 -    5 \sqrt{ \frac{h \log n}{ k d_{\max} } } \geq \frac{\xi^2}{8}, 
    \end{aligned}
    $$
    where last inequality follows since $k \geq \frac{c h \log n}{d_{\max} \xi^4 }$ for some large enough constant $c$. 
        \end{proof}
    Now, using \cref{combinedupperbound} {to bound the denominator term in} \cref{lem:PertubationOfPi} and following the same procedure in \cite[Theorem 9]{Chenetal2019}, we obtain the following $\ell^2$-error bound.  
    
    \begin{lemma}[$\ell^2$-Error Bound for Pairwise Comparison Model] \label{lem:L2error}
    Under the pairwise comparison model discussed in \cref{Formal Model and Goal} such that \cref{assm:Large Expansion of DTM,assm: Bounded Principal Ratio} holds and for 
    $k \geq \max \big\{1,\frac{c h \log n}{d_{\max} \xi^4 } \big\}$ and $d_{\max} \geq \log n$, the following bound holds:  
    \begin{equation*}
        \|\pi - \pihat\|_2 \leq \frac{C}{\xi^2 } \sqrt{ \frac{h}{k d_{\max} } } \sqrt{n} \|\pi\|_\infty   
    \end{equation*} 
    with probability at least $1 - O(n^{-3})$ for some constants $c, C$ independent of $n, d_{\max}, k$.
    \end{lemma}
    Utilizing the results developed above, the rest of the proof follows using similar arguments in \cite[Theorem 9]{Chenetal2019} and is provided in \cref{Proof of L2 error empirical} for completeness.

\subsection{Proof of \cref{Main theorem with upper bound}} \label{sec:Upper Bound on Critical Threshold}

    In this section, we will utilize the lemmata developed above to prove \cref{Main theorem with upper bound}.
    
    \begin{proof}[Proof of \cref{Main theorem with upper bound}] 
     For every $(i,j) \in \calE$, define $\hat{Y}_{ij} \triangleq \frac{Z_{ij}(Z_{ij}-1)}{k_{ij}(k_{ij}-1)}$ and $\hat{p}_{ij} \triangleq \frac{Z_{ij}}{k_{ij}}$. Since $Z_{ij} \sim \text{Bin}(k_{ij},p_{ij})$, it is easy to verify that 
        \begin{align}
        &\E[\hat{p}_{ij}] = p_{ij}  \  \text{  and } \ \var(\hat{p}_{ij}) = \frac{p_{ij}(1-p_{ij})}{k_{ij}},\nonumber \\
        & \E[\hat{Y}_{ij}]  = p_{ij}^2 \  \text{  and } \nonumber \\
        &  \var(\hat{Y}_{ij}) = \frac{-2\left(2k_{ij}-3\right)p_{ij}^4+4(k_{ij}-2)p_{ij}^3+2p_{ij}^2}{k_{ij}(k_{ij}-1)}    . \label{Yijproperties}
        \end{align}
    Now, the test statistic $T$ in terms of $\hat{Y}_{ij}$ and $\hat{p}_{ij}$ is
    $$T = \sum_{\substack{(i,j) \in \calE: \\ i \neq j} } \left(\hat{\pi}_{i}+\hat{\pi}_{j}\right)^{2} \hat{Y}_{ij} +\hat{\pi}_{j}^{2} - 2  \hat{\pi}_{j}\left(\hat{\pi}_{i}+\hat{\pi}_{j}\right) \hat{p}_{ij}. $$

    We split $T$ as $T=T_{1}+T_{2}+T_{3}$ where
    \begin{align}
      T_{1} & \triangleq \sum_{\substack{(i,j) \in \calE: \\ i \neq j} } \left(\left(\hat{\pi}_{i}+\hat{\pi}_{j}\right)^{2}-\left(\pi_{i}+\pi_{j}\right)^{2}\right)\left(\hat{Y}_{i j}-p_{ij}^{2}\right) \nonumber \\ 
      &\ \ \ \ \ \ \ \ \  -2\left(\hat{\pi}_{j}\left(\hat{\pi}_{i}+\hat{\pi}_{j}\right)-\pi_{j}\left(\pi_{i}+\pi_{ j}\right)\right)\left(\hat{p}_{ij}-p_{ij}\right) , \label{definiton of T1}\\ 
      T_{2}& \triangleq\sum_{\substack{(i,j) \in \calE: \\ i \neq j} } \left(\left(\hat{\pi}_{i}+\hat{\pi}_{j}\right)^{2}-\left(\pi_{i}+\pi_{j}\right)^{2}\right) p_{ij}^{2}+\hat{\pi}_{j}^{2}-\pi_{j}^{2} \nonumber\\ 
      &\ \ \ \ \ \ \ \ \  -2\left(\hat{\pi}_{j}\left(\hat{\pi}_{i}+\hat{\pi}_{j}\right)-\pi_{j}\left(\pi_{i}+\pi_{j}\right)\right) p_{ij}, \label{definiton of T2}\\ 
    T_{3}& \triangleq\sum_{\substack{(i,j) \in \calE: \\ i \neq j} } \left(\left(\pi_{i}+\pi_{j}\right)^{2} \hat{Y}_{i j}+\pi_{j}^{2}-2 \pi_{j}\left(\pi_{i}+\pi_{j}\right) \hat{p}_{ij}\right). \label{definiton of T3} 
    \end{align}
    The following lemma bounds the terms $T_1$ and $T_2$.
    \begin{lemma}[Bounds on $T_1$ and $T_2$] Under the assumptions of \cref{Main theorem with upper bound}, the following bounds hold on $T_1, T_2$: \label{lem: Bounds on T1 and T2}
        \begin{enumerate}
            \item There exist constants \(c_0, \tilde{c}_0, \hat{c}_0\) such that the following tail bound holds for \( |T_1| \):
        \begin{align} 
                & \mathbb{P}\left( |T_1| \geq  c_0\frac{ n \|\pi\|_\infty^2 \sqrt{h} }{k\xi^2 }  +\frac{\tilde{c}_0  n \|\pi\|_\infty^2\sqrt{h}  } {k \xi^2 } \times   \vphantom{\left(\frac{\log^3 n }{ d_{\max} } \right)^{\frac{1}{4} } } \right. \notag \\[-0.3em] 
                & \left.
   \vphantom{\frac{n \|\pi\|_\infty^2 \sqrt{h}}{k\xi^2}}
    \ \ \left( \!\left(\frac{\log^3 n }{ d_{\max} } \right)^{\frac{1}{4} } \!\! +\!\sqrt{\frac{t}{ d_{\max} }} \!+ \!\left( \frac{ t^4 }{ k d_{\max } } \right)^{\frac{ 1}{6} }\!\!  +\sqrt{\frac{t^2}{ k d_{\max}}} \right)\!\right) \nonumber \\ 
     &  \  \ \ \ \ \ \ \ \ \ \ \ \ \ \ \ \ \ \ \leq 16 n e^{-t} + { \frac{ \hat{c}_0 }{n^3} } .\label{tail bound on T1}
         \end{align}
            \item If \( d_{\max} \geq (\log n)^4 \), there exists a constant \( c_\alpha \) such that, with probability at least \( 1 - O(n^{-3}) \), we have
             \begin{equation} \label{finalized bound on T1}
                 |T_{1}| \leq \frac{c_\alpha n \|\pi\|_\infty ^2 \sqrt{h} }{ k\xi^2 }.
             \end{equation}
            \item If \(  d_{\max} \geq (\log n)^4 \), there exist constants \( c_{\beta} \) and \( c_{\gamma} \) such that the following bound holds for \( |T_2| \), with probability at least \( 1 - O(n^{-3}) \):   
        \begin{align} 
                 |T_{2}| & \leq \frac{c_\beta \sqrt{n} \|\pi\|_\infty \sqrt{h} }{ \sqrt{k} \xi^2 } \|\Pi P + P \Pi - \PE(\mathbf{1}_n\pi^\T)   \|_\F \nonumber\\ 
                 & \quad \, + c_\gamma \frac{n\|\pi\|^2_\infty h}{ k \xi^4 } .\label{tail bound on T2}
        \end{align}
        \item If $d_{\max} \geq \log n$, the constants $c_\alpha$ and $c_\gamma$ defined above scale as $O(\sqrt{\log n})$.
        \end{enumerate}
    \end{lemma}
    
    The proof is provided in \cref{Proof of Bounds on T1 and T2}. The following lemma characterizes the mean and the variance of $T_3$.
    
    \begin{lemma}[Mean and Variance of $T_3$] \label{lem:H0proof}
        The following bounds hold for the mean and variance of $T_3$ as defined in \cref{definiton of T3}:
        \begin{enumerate}
            \item Mean of $T_3:$ \[ \E[T_3] =  \|\Pi P + P\Pi - \PE(\mathbf{1}_n\pi^\T) \|^2_\F .\]
            \item Variance of $T_3:$
            \[
            \begin{aligned}                
                \var(T_3) & \leq \frac{4\|\pi \|^2_{\infty}}{k}\|\Pi P+P \Pi-\PE(\mathbf{1}_n \pi^\T)\|_{\F}^{2} \\
                & \quad \, +\frac{4n d_{\max}}{k^2}\|\pi\|_{\infty}^{4} . 
            \end{aligned} 
            \]
        \end{enumerate}
    \end{lemma}
    The proof is provided in \cref{Proof of Mean and Variance of T3}. 
    Now under hypothesis $H_0$, by \cref{lem:H0proof}, $\E_{H_0}[T_3] = 0 $ and $\var_{H_0}(T_3) \leq \frac{4n d_{\max}\|\pi\|^4_\infty}{k^2}$. Moreover, the event $\{T \geq t\}$ can be written as 
    $$
    \begin{aligned}
        \{T \geq t\} & = \left(  \{T\geq t \} \cap \{T \leq T'\} \right ) \cup \left( \{ T \geq t\} \cap \{T > T'\} \right)\\
        & = \left(  \{t \leq T \leq T'\} \right ) \cup \left( \{ T \geq t\} \cap \{T > T'\} \right).
    \end{aligned}
    $$
    Define $\tilde{c}_\alpha = c_\alpha \sqrt{h}/\xi^2$, $\tilde{c}_\beta = c_\beta \sqrt{h}/\xi^2$ and $\tilde{c}_\gamma = c_\gamma h/\xi^4$.     Let $T' =  (\tilde{c}_\alpha + \tilde{c}_\gamma) \frac{\|\pi\|^2_\infty n}{k } + T_3. $ Then, we have (cf. \cite{RastogiBalakrishnanShahAarti2022})
    \begin{align}
        & \P_{H_0}(T  \geq t)   \leq \P_{H_0}(T' \geq t) + \P_{H_0}(T> T') \nonumber\\ 
        & \nonumber \leq \P_{H_0}\left( T_3 \geq t - (\tilde{c}_\alpha+\tilde{c}_\gamma) n\|\pi\|^2_\infty/{k} \right) + O\bigg(\frac{1}{n^3}\bigg) \\
        & \nonumber \stackrel{\zeta_1}{\leq}  \frac{\var_{H_0}(T_3)}{ \var_{H_0}(T_3) + (t - (\tilde{c}_\alpha+\tilde{c}_\gamma) n\|\pi\|^2_\infty/{k})^2} + O\bigg(\frac{1}{n^3}\bigg) \\
        & \nonumber \stackrel{\zeta_2}{\leq}  \frac{4nd_{\max}\|\pi\|^4_\infty/k^2}{4n d_{\max}\|\pi\|^4_\infty/k^2 + 16nd_{\max} \|\pi\|^4_\infty/k^2 } + O\bigg(\frac{1}{n^3}\bigg) \\
        & \stackrel{\zeta_3}{\leq} \frac{1}{4} \nonumber,
    \end{align}
    where $\zeta_1$ follows from one-sided Chebyshev's inequality \cite{HDP}, $\zeta_3$ follows for $n$ large enough, and in $\zeta_2$ we have substituted 
    \begin{equation*}    
      t =  4 \frac{\sqrt{nd_{\max}} \|\pi\|^2_\infty}{k}  + (\tilde{c}_\alpha+\tilde{c}_\gamma)\frac{n\|\pi\|^2_\infty}{k},
    \end{equation*}    
    which is clearly upper bounded by the theoretical threshold $\thr$ (independent of $\pi$) of our test given by
    \begin{align}
      \thr & \triangleq 4 \frac{\sqrt{d_{\max}} h^2}{n^{3/2}k} + (\tilde{c}_\alpha +\tilde{c}_\gamma )\frac{h^2}{n k } \nonumber \\
      &  = 4 \frac{\sqrt{d_{\max}} h^2}{n^{3/2}k} + \left( \frac{{c}_\alpha \sqrt{h}}{\xi^2} + \frac{{c}_\gamma h}{\xi^4}\right)\frac{h^2}{n k } .  \label{TestThreshold}
    \end{align}
    Similarly, under hypothesis $H_1$, for some random variable $T'$ and $t$, we have
        $$
        \begin{aligned}
            \{T < t\} & = \left(  \{T< t \} \cap \{T < T'\} \right ) \cup \left( \{ T <t\} \cap \{T \geq T'\} \right)\\
            & = \left( \{ T < t\} \cap \{T < T'\} \right) \cup \left(  \{T' \leq T < t\} \right ).
        \end{aligned}
        $$
    In this case, define $T' = T_3 -\Delta_T$, where 
        $$
        \begin{aligned}
            \Delta_T = \frac{\tilde{c}_\beta \sqrt{n}\|\pi\|_\infty}{\sqrt{k}} & \|\Pi P + P \Pi - \PE(\mathbf{1}_n\pi^\T) \|_\F \\
            & + (\tilde{c}_\alpha+\tilde{c}_\gamma) \frac{n\|\pi\|^2_\infty }{ k } .
        \end{aligned}
        $$
        Therefore, we can bound $\P_{H_1}(T< \thr)$ as
    \begin{align}
         &\P_{H_1}(T< \thr)  \leq \P_{H_1}(T' < \thr) + \P_{H_1}(T' > T)  \nonumber\\
         &  \leq \P_{H_1}\left(T_3 < \thr + \Delta_T  \right) + O\bigg(\frac{1}{n^3}\bigg) \nonumber\\
            & \nonumber \leq  \frac{\var_{H_1}(T_3)}{ \var_{H_1}(T_3) + (\E_{H_1}[T_3] - \thr - \Delta_T)^2} +  O\bigg(\frac{1}{n^3}\bigg)  \stackrel{\zeta}{\leq} \frac{1}{4},
    \end{align}
    where $\zeta$ is true if 
    $$ 4 \var_{H_1}(T_3) \leq  (\E_{H_1}[T_3] - \thr - \Delta_T)^2 . $$
    Let $D = \|\Pi P + P \Pi - \PE(\mathbf{1}_n\pi^\T)   \|_\F$. The above equation is true if 
    $$
        \begin{aligned}
            & 2\left(\frac{2h}{n\sqrt{k}} D +\frac{2 \sqrt{n d_{\max} } h^2 }{n^2 {k} }\right) \leq \\
            & \ \ \ \ \ \ \ \ \ \ \ \ \bigg ( D^{2} - \frac{ \tilde{c}_{\beta} D  h }{\sqrt{nk }}    - \frac{4h^2 \sqrt{ d_{\max}} }{n^{3/2}{k}} - 2(\tilde{c}_\alpha+\tilde{c}_\gamma) \frac{h^2}{nk}\bigg ) .
        \end{aligned}
    $$
    Substituting $D=\epsilon n \|\pi\|_\infty =  a/\sqrt{nk}$, we obtain the following equivalent condition:
    \begin{align}
     \frac{4 a h }{n^{3/2} k } + \frac{4h^2 \sqrt{d_{\max} } }{ n^{3/2} k } & \leq  \frac{a^2 n \|\pi\|_\infty^2}{k} - \frac{\tilde{c}_\beta a h }{nk } \nonumber \\ 
     & - 4 \frac{\sqrt{d_{\max}} h^2}{n^{3/2} k} - 2(\tilde{c}_\alpha+\tilde{c}_\gamma) \frac{h^2}{nk} . 
    \end{align}
    Using that $n\|\pi\|_\infty \leq 1$, This is equivalent to the following condition:
    \begin{equation}    
        \begin{aligned}
            \frac{4 a }{h n^{1/2} } + 4\sqrt{ \frac{ d_{\max} }{n} }  \leq  \frac{a^2}{h^2} -  \tilde{c}_\beta \frac{a}{h} - 4 \sqrt{ \frac{d_{\max}}{n} } - 2(\tilde{c}_\alpha+\tilde{c}_\gamma) .
        \end{aligned}
    \end{equation}
    Now, we can rewrite the above equation as a polynomial in $a$ as:
        \begin{align}
            \frac{a^2}{h^2} - &\left( \frac{\sqrt{h}}{\xi^2} c_\beta + \frac{4}{\sqrt{n}}\right) \frac{a}{h} \nonumber \\ 
            & \ \  \ \ \ \ \ \  - \left( 8\sqrt{ \frac{d_{\max}}{n} } + 2\left({c}_\alpha \frac{\sqrt{h}}{\xi^2}  + {c}_\gamma  \frac{h}{\xi^4} \right) \right) \geq 0. \label{Polynomial}  
        \end{align}
    Utilizing the fact that $\xi \leq 1$ and $h \geq 1$, the condition in \eqref{Polynomial} is clearly satisfied for a sufficiently large constant $a_0$ such that $a \geq a_0 h^{3/2}/\xi^2 $, which gives 
    $$
    \begin{aligned}
    n \|\pi\|_\infty \epsilon  \geq \frac{a_{0} h^{3/2}}{\xi^2 \sqrt{nk} } \implies \epsilon \geq \frac{a_{0} \sqrt{h^3} }{\sqrt{nk \xi^4}}.
    \end{aligned}
    $$
    Thus, for sufficiently large constant $a_0$, we have demonstrated that the sum of type \Romannum{1} and type \Romannum{2} errors is bounded by $\frac{1}{2}$. Combining this result with the definition of the critical threshold in \cref{critical radius}, we obtain the following bound on $\varepsilon_{\mathsf{c}}$:
$$
 \varepsilon_{\mathsf{c}}^2 \leq O\left( \frac{ h^3 }{ n k\xi^4 } \right). 
$$
This completes the proof.
\end{proof}
We remark that if a standard matrix Bernstein inequality \cite{HDP} were used in the proof of \cref{lem: Bounds on T1 and T2}, the constants $c_{\alpha}$ and $c_{\gamma}$ would scale as $(\log n)^{1/2}$. From \cref{Polynomial}, we would get an additional factor of $(\log n)^{1/2}$ in the scaling of $\varepsilon^2_\mathsf{c}$, thus proving \cref{Main prop with upper bound}.

\subsection{Proofs of Lemmata}

\subsubsection{Proof of \cref{lem: Bounds on T1 and T2}} \label{Proof of Bounds on T1 and T2} 
    For bounding $T_1$ we split $T_1$ as $T_{1}=T_{1 a}+T_{1 b}$, where
        $$
            \begin{aligned}
            &T_{1 a}\triangleq \sum_{\substack{(i,j) \in \calE: \\ i \neq j} }\left(\left(\hat{\pi}_{i}+\hat{\pi}_{j}\right)^{2}-\left(\pi_{i}+\pi_{j}\right)^{2}\right)\left(\hat{Y}_{i j}-p_{ij}^{2}\right) ,\\
            &T_{1 b}\triangleq \sum_{\substack{(i,j) \in \calE: \\ i \neq j} }-2\left(\hat{\pi}_{j}\left(\hat{\pi}_{i}+\hat{\pi}_{j}\right)-\pi_{j}\left(\pi_{i}+\pi_{j}\right)\right)\left(\hat{p}_{ij}-p_{ij}\right).
            \end{aligned}
        $$
    To establish an upper bound for $\P(T_{1a} +T_{1b} \geq t_1 + t_2)$, we utilize the following property: 
     \[ 
     \begin{aligned}
     \forall t_1, t_2 >0, \ \P(T_{1a} + T_{1b} & \geq t_1 + t_2) \\
     & \leq \P(T_{1a} \geq t_1) + \P(T_{1b} \geq t_2).         
     \end{aligned}
     \]
    Hence, we proceed by bounding tail bounds on $T_{1a}$ and $T_{1b}$ separately as below. These bounds will be derived when the event $\mathcal{A}_2$ holds, where we define
    \begin{equation}
        \mathcal{A}_2 \triangleq \bigg\{ \|\pihat - \pi\|_2 \leq c_2 \frac{ \sqrt{n} \|\pi\|_\infty }{\sqrt{k d_{\max}} } \bigg\} \label{eventA2},
    \end{equation}
    where we have absorbed the parameters $h, \xi$ in a special constant $c_2$ for conciseness and we have $c_2 = c \sqrt{h}/\xi^2$. Moreover, by \cref{lem:L2error}, we know that $\P(\mathcal{A}_2) \geq 1 - O(n^{-3})$.
    
    \textbf{Tail Bounds for $T_{1a}$:} Define a matrix $\Q$ as 
    $$Q_{i j}=\left\{\begin{array}{cc}\hat{Y}_{i j}-p_{ij}^{2} & \text{ if } i \neq j \text{ and } (i,j) \in \calE \\ 0 & \text{ otherwise.} \end{array}\right. .$$
    From \Cref{Yijproperties}, we have $\E[Q_{ij}] = 0$. Now we re-write $T_{1a}$ in terms of the matrix $Q$ as    
    \begin{align}
    T_{1 a} &=\sum_{\substack{(i,j) \in \calE: \\ i \neq j} }\left(\hat{\pi}_{i}-\pi_{i}+\hat{\pi}_{j}-\pi_{j}\right) Q_{i j}\left(\hat{\pi}_{i}+\pi_{i}+\hat{\pi}_{j}+\pi_{j}\right) \nonumber\\
            &=(\hat{\pi}-\pi)^\T\Q(\hat{\pi}+\pi)+(\hat{\pi}-\pi)^\T\Q^\T(\hat{\pi}+\pi) \nonumber\\ 
            & \ \ \ \ \ \ \ + \underbrace{\left(\pihat^{2}-\pi^{2}\right)^\T\Q \mathbf{1}_n +\mathbf{1}_n^\T\Q\left(\hat{\pi}^{2}-\pi^{2}\right)}_{\zeta_0}\nonumber \\
            &\stackrel{\zeta_1}{\leq}  2 c_2  \frac{\sqrt{n} \|\pi\|_{\infty}}{\sqrt{k d_{\max} }} \|Q\|_2 (2 \|\pi\|_{2} + \|\pihat - \pi\|_2) \nonumber \\ 
            & \ \ \ \ \ \ \ + 2 {c}_2  \frac{n\|\pi\|_{\infty}^{2}\|Q\|_{2}}{\sqrt{k d_{\max }}} \nonumber  + \tilde{c} {c}_2  \frac{ n\|\pi\|_\infty^2 }{ k }  \\
      & \stackrel{\zeta_2}{\leq} 4c_2 n \|\pi\|_{\infty}^{2} \frac{\|Q\|_2}{\sqrt{k d_{\max} } } + 2 c_2^2  n \|\pi\|^2_{\infty}\frac{\|Q\|_2 }{k d_{\max} } \nonumber \\
      & \ \ \ \ \ \ \ \  +  2{c}_2  \frac{n\|\pi\|_{\infty}^{2}\|Q\|_{2}}{\sqrt{k d_{\max }}}  + \tilde{c} {c}_2  \frac{ n\|\pi\|_\infty^2 }{ k }  \nonumber\\ 
      & \leq 6 {c}_2 n \|\pi\|_\infty^2 \frac{\|Q\|_2}{ \sqrt{k d_{\max}}} + \tilde{c} {c}_2  \frac{ n\|\pi\|_\infty^2 }{ k } ,  \label{Bound1} 
    \end{align}    
    where $\zeta_1$ follows from the fact $x^\T A y \leq \| x\|_2 \|y\|_2 \| A\|_2$ and since the event $\mathcal{A}_2$ holds. In $\zeta_2$, we utilize the fact that $\|\pi\|_2 \leq \sqrt{n} \|\pi\|_\infty$, and the last inequality follows since $k \geq O(h/(\dmax \xi^4))$. The term  $\zeta_0$ is upper bounded as
     \begin{equation}\label{that long derivation}
     \zeta_0 \leq 2{c}_2  \frac{n\|\pi\|_{\infty}^{2}\|Q\|_{2}}{\sqrt{k d_{\max }}} + \tilde{c} {c}_2  \frac{ n\|\pi\|_\infty^2 }{ k } .
     \end{equation}
    The procedure for deriving this bound is somewhat involved and is presented below. 
    
    \textbf{Bounding $\zeta_0$:} Note that
\begin{align}
 & \nonumber \left(\hat{\pi}^{2}-\pi^{2}\right)^{\T} Q \1_{n} \\ 
 & =(\hat{\pi}-\pi)^{\T}(\hat{\Pi}+\Pi) Q \1_{n} = (\hat{\pi}-\pi)^{\T}(\hat{\Pi}- \Pi + 2\Pi) Q \1_{n} \nonumber\\
& =(\hat{\pi}-\pi)^{2^\T} Q \1_{n}+2(\hat{\pi}-\pi)^{\T} \Pi Q \1_{n} \nonumber\\
& \leq\|\hat{\pi}-\pi\|_{2}^{2}\left\|Q \1_{n}\right\|_{\infty}+2\|\hat{\pi}-\pi\|_{2}\|\Pi\|_{2}\|Q\|_{2}\left\|\1_{n}\right\|_{2} \nonumber\\
& \leq \frac{c_{2}^{2} n  \|\pi\|_{\infty}^2 }{k d_{ {\max }}}\left\|Q \1_{n}\right\|_{\infty}+\frac{2 c_{2} \sqrt{n} \|\pi\|_{\infty} }{\sqrt{k d_{ {\max }}}}\|\pi\|_{\infty}\|Q\|_{2} \sqrt{n} \nonumber\\
& \leq \frac{c_{2}^{2}n\|\pi\|_{\infty}^{2}}{k d_{ {\max }}}\left\|Q \1_{n}\right\|_{\infty}+\frac{2 c_{2} n\|\pi\|_{\infty}^2}{\sqrt{k d_{ {\max }}}}\|Q\|_{2}. 
\label{ first simplified bound Q}
\end{align}
Now we will establish concentration bounds for $\|Q\1_n\|_\infty$. To accomplish this, we will utilize McDiarmid's inequality \cite{HDP}. However, first note that
$$
\begin{aligned}
    (Q & \1_{n})_{i} =\sum_{j:(i, j) \in \calE}\left(\hat{Y}_{i j}-p_{i j}^{2}\right) \\ 
    & =\sum_{j:(i, j) \in \calE} \frac{\left(\sum_{m=1}^{k} Z_{m,{i j}}\right)\left(\sum_{m=1}^{k} Z_{m,{i j}}-1\right)}{k(k-1)}-p_{i j}^{2}.
\end{aligned}
$$
Define quantity $V_{ij}$ for $(i,j) \in \calE$ as 
$$
V_{i j}=\frac{\left(\sum_{m=1}^{k} Z_{m,{i j}}\right)\left(\sum_{m=1}^{k} Z_{m, i j}-1\right)}{k(k-1)}.
$$
Let $V_{i j}^{\prime}$ be the value of $V_{i j}$ when one of $Z_{ m, {i j}}$ is replaced by $Z_{m,{i j}}^{\prime}$, i.e., we have
$$ V_{i j}^{\prime}=\frac{\left(Z_{i j}+Z_{m,{i j}}^{\prime}-Z_{m,{i j}}\right)\left(Z_{i j}+Z_{m,{i j}}^{\prime}-Z_{m,{i j}}-1\right)}{k(k-1)}.
$$
Now, the absolute difference $|V^{\prime}_{i j}-V_{i j}|$ is bounded as
$$
\begin{aligned}
\left|V^{\prime}_{i j}-V_{i j}\right| & =\bigg| \frac{2\left(Z_{m,{i j}}^{\prime} - Z_{m,{i j}}\right) Z_{ij}}{k(k-1)} \\ 
& \ \ \ \ \ \ \ \ \ \ +\frac{\left(Z_{m,{i j}}^{\prime}-Z_{m,ij}\right)\left(Z_{m,{i j}}^{\prime}-Z_{m,{i j}-1}\right)}{k(k-1)} \bigg| \\
& \leq \frac{2|Z_{m,{i j}}^{\prime}-Z_{ m,{i j}}|}{k(k-1)}\left|Z_{i j}+\frac{(Z_{m,{i j}}^{\prime}-Z_{m,{i j}}-1)}{2} \right| \\
& \leq \frac{2}{k-1}.
\end{aligned}
$$
An application of McDiarmid's inequality gives
\[ \mathbb{P}\left(\left|\sum_{j:(i, j) \in \calE} \hat{Y}_{i j}-p_{i j}^{2}\right|>t\right) \leq 2 \exp\left(\frac{-2 t^{2}}{k d_{ \max} \left(\frac{2}{k-1}\right)^{2} } \right) .\]
Substituting $t = \tilde{c} \sqrt{\frac{d_{\max} \log n }{k}}$, for some constant $\tilde{c}$, we obtain the following bound:
$$
\forall \ i \in [n], \ \P\bigg( \left(Q \1_{n}\right)_{i} \geq \tilde{c} \sqrt{\frac{d_{\max} \log n }{k}} \bigg) \leq O(n^{-4}).
$$
Therefore, using the union bound, we have
\begin{align} 
\P\bigg(\left\|Q \1_{n}\right\|_{\infty} & \geq \tilde{c} \sqrt{\frac{d_{\max} \log n}{k}}\bigg) \nonumber\\
& \leq \sum_{i=1}^n \P\bigg(\left(Q \1_{n}\right)_{i} \geq  \tilde{c} \sqrt{\frac{d_{\max} \log n}{k}}\bigg) \nonumber\\
& \leq O(n^{-3}). \label{second simplified bound Q} 
\end{align}
Combining \cref{ first simplified bound Q} and \cref{second simplified bound Q}, and utilizing the fact that $c_2 = c\frac{\sqrt{h}}{\xi^2}$ and  $\xi^4 k d_{\max} \geq O(h \log n) $, we have the following bound with high probability:
\begin{align}
\left(\pi^{2}\!-\!\pi^{2}\right)^\T Q \1_{n}\! & \nonumber \leq \tilde{c} c_{2}  \frac{ n\|\pi\|^2_{\infty}}{k } \sqrt{\frac{c h \log n}{ \xi^4 k d_{\max} }}+2 c_{2} \frac{n\|\pi\|_{\infty}^{2}\|Q\|_{2}}{\sqrt{k d_{\max }}} \\ 
& \leq \tilde{c} c_{2}  \frac{ n\|\pi\|^2_{\infty}}{k } + 2 {c}_2  \frac{n\|\pi\|_{\infty}^{2}\|Q\|_{2}}{\sqrt{k d_{\max }}} . \label{that same argument}
\end{align}
In a similar manner, we can bound the term $\1_n^\T Q(\pihat^2 - \pi^2)$, and thus, we obtain the bound in \cref{that long derivation}.

    \textbf{Bounding $\|Q\|_2$:} Now it remains to show that $\|Q\|_2 \leq O(\sqrt{ d_{\max}/k})$ with high probability. In the case of a complete graph, it is much easier to show because each entry $Q_{i j}$ is bounded, and therefore $\Q$ is a random sub-Gaussian matrix, and the variance of each entry is upper bounded by $4/k$ (by \cref{Yijproperties}). Hence, by \cite[Theorem 4.4.5]{HDP}, the spectral norm $\|\Q\|_2 \leq 2c_q  (2\sqrt{d_{\max}} + t)/\sqrt{k}$ for some constant $c_q$ with probability at least $1 - 2e^{-t^2}$. Substituting $t = \sqrt{\log n}$, we get the following bound with a probability at least $1 - O(1/n^3)$ 
    \begin{equation}
        \|\Q\|_2 \leq 6c_q  \sqrt{\frac{ d_{\max} }{k}}. \label{SpectralBoundQ}
    \end{equation}
    For a general graph model (with $d_{\max} \geq \log n$), an application of matrix Bernstein inequality \cite{HDP} yields $\|Q\|_2 \leq O( \sqrt{ {d_{\max} \log n}/{k}} )$ (with high probability). The extra $\log n$ factor becomes a bottleneck later in the analysis. However, using recent advances in concentration inequalities \cite[ Corollary 2.15]{brailovskaya2022} we can indeed show that if $d_{\max} \geq (\log n)^4$, then $\|Q\|_2 \leq O( \sqrt{ d_{\max}/k }) $ with high probability. The tail bounds of $\|Q\|_2$ are computed in \cref{lem:Spectral Norm of Error Improved2}, and therefore, we have
    \begin{align}
    \mathbb{P}\bigg(\|Q\|_2 & \geq \sqrt{\frac{24 d_{\max }}{k}}+c\bigg(\frac{d_{\max }^{1 / 4}}{\sqrt{k}}(\log n)^{3 / 4}+\sqrt{\frac{t}{k}} \nonumber \\ 
    &\ \ \  +\frac{d_{\max }^{1 / 3}t^{2 / 3}}{( k^2 (k-1) )^{1 / 3}} + \frac{t}{k(k-1)}\bigg)\bigg) \leq 4 n e^{-t} .\label{Bound2}
    \end{align}
    Substituting $t = c \log n$, we obtain the bound $ \|Q\|_2 \leq O(\sqrt{d_{\max}/k})$, which by \cref{Bound1} implies the bound on $T_{1a}$ as $T_{1a} \leq O(c_2 n \|\pi\|_\infty^2 /k)$. 
    
    \textbf{Tail Bounds for $T_{1b}$:} Now we bound the quantity $T_{1b}$ as 
    \begin{align}
    T_{1 b} 
    &\nonumber = -2\sum_{\substack{(i,j) \in \calE: \\ i \neq j} }\left(\hat{\pi}_{j}^{2}-\pi_{j}^{2}+\hat{\pi}_{i} \hat{\pi}_{j}-\pi_{i} \pi_{j}\right)\left(\hat{p}_{ij}-p_{ij}\right) \\
    &\nonumber =\sum_{j=1}^n-2\left(\hat{\pi}_{j}^{2}-\pi_{j}^{2}\right)\left(\sum_{\substack{i: (i,j) \in \calE,\\ i \neq j} } (\hat{p}_{ij}-p_{ij})\right) \\
    & \nonumber\ \ \ \ \ \ \ \   -2\sum_{\substack{(i,j) \in \calE: \\ i \neq j} } \left(\hat{\pi}_{i}-\pi_{i}\right)\left(\hat{\pi}_j-\pi_{j}\right)\left(\hat{p}_{ij}-p_{ij}\right) \\
    &\nonumber \ \ \ \ \ \ \ \ \ \  -2\sum_{\substack{(i,j) \in \calE: \\ i \neq j} } \pi_{i}\left(\hat{p}_{ij}-p_{ij}\right)\left(\hat{\pi}_{j}-\pi_{j}\right) \\
    & \nonumber \ \ \ \ \ \ \ \ \ \ -2\sum_{\substack{(i,j) \in \calE: \\ i \neq j} } \pi_{j}\left(\hat{p}_{ij}-p_{ij}\right)\left(\hat{\pi}_{i}-\pi_{i}\right) \\
    &\nonumber =  \underbrace{-2 \mathbf{1}_n^\T(\hat{P}-P)\left(\hat{\pi}^{2}-\pi^{2}\right)}_{\zeta_0} -2(\pihat-\pi)^\T(\hat{P}-P)(\hat{\pi}-\pi) \\
    &\nonumber  \ \ \ \ \ -2  \pi^\T(\hat{P}-P)(\hat{\pi}-\pi)-2(\hat{\pi}-\pi)^\T \left(\hat{P}^\T-P^\T\right) \pi \\
    &\nonumber \leq \frac{4c_2 n\left\|{\pi}\right\|_{\infty}^{2}}{\sqrt{k d_{\max }}}\|\hat{P}-P\|_{2}  + \tilde{c} c_2 \frac{n \|\pi\|_\infty^2}{k } + \\
    & \ \ \ \ \ \ \nonumber 2\|\pihat -\pi\|^2_{2} \|\hat{P}-P\|_2  + 4\|\pi\|_{2} \|\hat{P}-P\|_2\|\hat{\pi}-\pi\|_{2} \\
    &\nonumber \stackrel{\zeta_1}{\leq} \frac{4c_2 n\left\|{\pi}\right\|_{\infty}^{2}}{\sqrt{k d_{\max }}}\|\hat{P}-P\|_{2}+ \tilde{c} c_2 \frac{n \|\pi\|_\infty^2}{k }  \\
    &\nonumber  \ \ \ \ \ + 2 c_2^2 n \|\pi\|^2_{\infty} \frac{ \|\hat{P}-P\|_2}{ k d_{\max} }  + 4  c_2 n \|\pi\|_\infty^2\frac{\|\hat{P}-P\|_2 }{ \sqrt{k d_{\max}} }   \\
    & \stackrel{\zeta_2}{\leq}  \hat{c} c_2  n\|\pi\|_{\infty}^2 \frac{ \|\hat{P}-P\|_2 }{\sqrt{k d_{\max}}} + \tilde{c} c_2 \frac{n \|\pi\|_\infty^2}{k }, \label{Bound3}
    \end{align}
    where $\zeta_1$ holds on the event $\mathcal{A}_2$ and the term $\zeta_0$ is bounded as
    \begin{align}
          |2 \mathbf{1}_n^\T(\hat{P}& -P)\left(\hat{\pi}^{2}- \pi^{2}\right)| \nonumber \\ & \ \ \ \ \ \leq \frac{4c_2 n\left\|{\pi}\right\|_{\infty}^{2}}{\sqrt{k d_{\max }}}\|\hat{P}-P\|_{2} + \tilde{c} c_2 \frac{n \|\pi\|_\infty^2}{k }.\label{that long derivation2}
    \end{align}
   Moreover, ${\zeta_2}$ follows, for some constant $\hat{c}$, since $\sqrt{c_2^2 / (k d_{\max}) } = \sqrt{c h / (k d_{\max} \xi^4) } \leq O(1)$ by our assumption in \cref{Main theorem with upper bound}.
   
    \textbf{Bounding $\zeta_0$:}\footnote{The $\zeta_0$ notation is overloaded. Note that we bound a different $\zeta_0$ term this time to the one we addressed earlier in this proof.} To bound $\1_{n}^{\T}(\hat{P}-P)\left(\hat{\pi}^{2}-\pi^{2}\right)$, we utilize the same technique as used in \cref{that long derivation}. Observe that
\begin{align}
 & \nonumber\1_{n}^{\T}(\hat{P}-P)\left(\hat{\pi}^{2}-\pi^{2}\right)\\ 
 & \nonumber =\1_{n}^{\T}(\hat{P}-P)(\hat{\Pi}+\Pi)(\hat{\pi}-\pi) \\
 & \nonumber =\1_{n}^{\T}(\hat{P}-P)(\hat{\Pi}-\Pi + 2\Pi)(\hat{\pi}-\pi) \\
&\nonumber =\1_{n}^{\T}(\hat{P}-P)((\hat{\pi}-\pi)^{2})+2 \1_{n}^\T\left(\hat{P}-P\right) \Pi(\hat{\pi}-\pi) \\
& \nonumber\leq\|\1_{n}^{\T}(\hat{P}\!-\!P)\|_{\infty}\|\hat{\pi}-\pi\|_{2}^{2}+2 \sqrt{n} \| \hat{P} -P\left\|_{2}\right\| \pi\left\|_{\infty}\right\|\hat{\pi}\!-\!\pi \|_{2} \\
&\nonumber \leq\|\1_{n}(\hat{P}- P)\|_{\infty} \frac{c_{2}^{2} n \|\pi\|_{\infty}^{2}}{k d_{\max }}\!+\!2 \sqrt{n}\|\hat{P}- P\|_{2}\frac{ c_{2} \sqrt{n} \|\pi\|^2_{\infty}}{\sqrt{k d_{\max }}}\\
& \leq\|\1_{n}(\hat{P}-P)\|_{\infty} \frac{c_{2}^{2}n \|\pi\|_{\infty}^{2}}{k d_{\max }}+2 c_{2} \frac{n\|\pi\|_{\infty}^{2}}{\sqrt{k d_{\max }}} \|\hat{P}  -P\|_{2} .\label{second simplified bound P}
\end{align}
Now to find concentration bounds for $\|\1_{n}^{\T}(\hat{P}-P)\|_\infty$, observe that its $i$th component can be expressed as
$$
\begin{aligned}
    (\1_{n}^{\T}(\hat{P}-P))_{i}& =\sum_{j:(i, j) \in \calE} (\hat{p}_{j i}-p_{j i})\\ 
    & =\frac{1}{k} \sum_{j: (i, j) \in \calE}\sum_{m=1}^k\left(Z_{m,{j i}}-p_{j i}\right).
\end{aligned}
$$
Utilizing Hoeffding's inequality \cite{HDP}, we get
$$
\P(|(\1_{n}(\hat{P}-P))_{i}|>t) \leq 2 \exp \bigg(\frac{-2 t^{2}}{k d_{\max } (\frac{1}{k^{2}}) }\bigg).
$$
Substituting, $t = c \sqrt{\frac{d_{\max } \log n}{k}}$ for some constant $c$ and utilizing the union bound as in \cref{second simplified bound Q}, we obtain the bound the following bound with probability at least $1 - O(n^{-4})$.
\begin{equation}
\| \1_{n}^\T (\hat{P} - P )\|_{\infty} \leq c \sqrt{\frac{d_{ \max} \log n}{k}}. \label{conc2inf}    
\end{equation}
Substituting the above bound in \cref{second simplified bound P} and using the same argument as in  \cref{that same argument}, we obtain the bound in \cref{that long derivation2} for some constant $\tilde{c} $.

    Now again we need to show that $\|\hat{P} -P\|_2 \leq O(\sqrt{ d_{\max}/k} )$ with high probability. In case of complete graph, by \cite[Theorem 4.4.5]{HDP}, we have $\|\hat{P} -P \|_2 \leq 6c_p \sqrt{d_{\max}/k}$ with high probability for some constant $c_p$ by the same argument as \eqref{SpectralBoundQ} (since $\var(\hat{p}_{ij}) \leq 1/k$). For a general graph model, and by application of matrix Bernstein inequality, again a $\log n$ factor becomes a bottleneck as we obtain $\|\hat{P}- P\|_2 \leq O\big(\sqrt{ \frac{d_{\max} \log n}{k}} \big)$ (with high probability). Therefore, we again utilize \cite[ Corollary 2.15]{brailovskaya2022} to obtain tighter concentration bounds on $\|\hat{P}- P\|_2 $. Applying, \cref{lem:Spectral Norm of Error Improved} we obtain
\begin{align} 
& \mathbb{P}\left(\|\hat{P} - P\|_2 \geq \sqrt{\frac{d_{\max }}{k}}+c\left(\frac{d_{\max }^{1 / 4}}{\sqrt{4 k}}(\log n)^{3 / 4}+\sqrt{\frac{t}{4 k}} \nonumber \vphantom{ +\frac{d_{\max }{ }^{1 / 3}}{(2 k)^{2 / 3}} t^{2 / 3} } \right. \right.  \\ 
    &\ \ \ \ \ \ \ \ \ \ \ \ \ \ \ \ \ \ \ \ \ \ \ \left. \left. \vphantom{\frac{d_{\max }^{1 / 4}}{\sqrt{4 k}}(\log n)^{3 / 4}}  +\frac{d_{\max }^{1 / 3}}{(2 k)^{2 / 3}} t^{2 / 3}+\frac{t}{k}\right)\right) \leq 4 n e^{-t} .\label{Bound4}
\end{align}
Now combining \cref{Bound1,Bound2,Bound3,Bound4}, we obtain  
    \begin{align}
        & \mathbb{P}\left(T_1 \geq  c_0\frac{ n \|\pi\|_\infty^2 \sqrt{h} }{k \xi^2} +\frac{\tilde{c}_0  n \|\pi\|_\infty^2 \sqrt{h}}{k \xi^2 } \left( \left(\frac{\log^3 n }{ d_{\max} } \right)^{\frac{1}{4}}+  \right. \right.  \nonumber \\
        &\left. \left. \!\sqrt{\frac{t}{ d_{\max} }} + \left( \frac{ t^4 }{ k d_{\max } } \right)^{ \frac{1}{6} } \!\! +\sqrt{\frac{t^2}{ k d_{\max}}} \right)\right) \leq 8 n e^{-t} + O\bigg(\frac{1}{n^3}\bigg), 
    \end{align}
    where the term $O\big(\frac{1}{n^3}\big)$ is added to account for the probability with which the event $\mathcal{A}_2$, \cref{second simplified bound Q}, and \cref{conc2inf} hold.
    Substituting $t = \bar{c} \log n$ for some constant $\bar{c}$, and utilizing the fact that $d_{\max} \geq (\log n)^4$, we obtain that there exists a constant $c_{\alpha}$ such that with probability at least $1 - O(1/n^3)$, we have $T_1 \leq c_\alpha\frac{ n \|\pi\|_\infty^2 \sqrt{h}}{k \xi^2} $. The corresponding lower bounds on $T_1$ can be obtained in a similar manner, and therefore, for the tail bound on $|T_1|$, we get a factor of $2$ on the tail-probability, thus proving \cref{tail bound on T1}. 

    \textbf{Bounding $T_2$: }{ Define matrices $P_{2}$ and $P_{3}$ as follows
    $$
    \begin{aligned}
        P_{2}& \triangleq \left \{\begin{array}{cc} p_{ij}^{2}, & i \neq j \text{ and } (i,j) \in \calE \\ 1/2, & i = j \\
        0, & \text{ otherwise} \end{array} \right. ,    \\   
         P_{3} & \triangleq \left\{\begin{array}{cc}\left(1-p_{ji}\right)^{2} + p_{ij}^{2}, & i \neq j \text{ and } (i,j) \in \calE \\
         0, & \text{ otherwise} \end{array}\right. .
             \end{aligned}
    $$
    Now we will bound $T_{2}$ by simplifying it as
    \begin{align*}
    T_{2}& =\sum_{\substack{(i,j) \in \calE: \\ i \neq j} } \left(\left(\hat{\pi}_{i}+\hat{\pi}_{j}\right)^{2}-\left(\pi_{i}+\pi_{j}\right)^{2}\right) p_{ij}^{2}+\hat{\pi}_{j}^{2} -\pi_{j}^{2} \\
    & \ \ \ \ \ \ \ \ \ \ \ \ \ \  -2\left(\hat{\pi}_{j}\left(\hat{\pi}_{i}+\hat{\pi}_{j}\right)-\pi_{j}\left(\pi_{i}+\pi_{j}\right)\right) p_{ij} \\
    & =\sum_{\substack{(i,j) \in \calE: \\ i \neq j} }\left(\hat{\pi}_{i}^{2}-\pi_{i}^{2}\right) p_{ij}^{2} + \left(\hat{\pi}_{j}^{2}-\pi_{j}^{2}\right)\left(p_{ij}^{2}+1-2 p_{ij}\right) \\
    & \ \ \ \ \ \ \ \ \ \ \ \ \ \  +2 \left(\hat{\pi}_{i} \pihat_{j}-\pi_{i} \pi_{j}\right)\left(p_{ij}^{2}-p_{ij}\right) \\
    & = \sum_{\substack{(i,j) \in \calE: \\ i \neq j} }\left(\hat{\pi}_{i}-\pi_{i}\right) \left(\hat{\pi}_{i}+\pi_{i}\right)  \left( p_{ij}^{2} + \left(1-p_{ji}\right)^{2} \right) \\
    & \ \ \ \ \ \ \ \ \ \ \  + 2 \left(\pihat_{i}-\pi_{i}\right) \left(\hat{\pi}_{j}-\pi_{j}\right)\left(p_{ij}^{2}-p_{ij}\right) \\
    & \ \ \ \ \   +2\pi_{i}\left(\hat{\pi}_{j}-\pi_{j}\right)\left(p_{ij}^{2}-p_{ij}\right) +2\pi_{j}\left(\hat{\pi}_{i}-\pi_{i}\right)\left(p_{ij}^{2}-p_{ij}\right) \\
    & = \sum_{\substack{(i,j) \in \calE: \\ i \neq j} }\left(\hat{\pi}_{i}-\pi_{i}\right)^2 ( p_{ij}^{2} + \left(1-p_{ji}\right)^{2} )  + 2\left(\hat{\pi}_{i}-\pi_{i}\right) \times \\
    &  \left(  {\pi}_{i}  ( p_{ij}^{2} + \left(1-p_{ji}\right)^{2} ) 
    +\pi_{j}\left(p_{j i}^{2}-p_{j i}\right) +\pi_{j}\left(p_{ij}^{2}-p_{ij}\right) \right) \\ 
    & \ \ \ \ \ \ \ + 2 \left(\pihat_{i}-\pi_{i}\right) \left(\hat{\pi}_{j}-\pi_{j}\right)\left(p_{ij}^{2}-p_{ij}\right) \\
    & =(\hat{\pi}-\pi)^{2^\T}P_{3} \mathbf{1}_n + \\ & \ \ \ \ \ \ \ \ \ \ \  \sum_{\substack{(i,j) \in \calE: \\ i \neq j} }2\left( \hat{\pi}_{i}-\pi_{i}\right) p_{ij} \left(  {\pi}_{i} p_{ij} + \pi_{j} p_{ij} - \pi_j  \right) \\
    & \ \ \ \ \ \ \ \ \ \ + 2\left( \hat{\pi}_{i} -\pi_{i}\right)  \left(1-p_{ji}\right) \left(\pi_i - \pi_i p_{ji} -\pi_j p_{ji} \right)\\
    & \ \ \ \ \ \ \ \ \ \ \ \ \ +2 \left(\hat{\pi} -\pi\right)^\T (P_{2} -P)\left (\hat{\pi} -\pi\right) \\
    & \stackrel{\zeta_1}{\leq} \|\hat{\pi}-\pi\|^2_2 \|P_{3}\1_n\|_\infty \\ 
    & \ \ \ \ \ \ \ \ \ + 4\sqrt{d_{\max} }\|\pihat -\pi \|_2 \| \Pi P + P\Pi - \PE(\mathbf{1}_n \pi^\T) \|_\F \\
    &  \ \ \ \ \ \ \ \ \ \ \ + \|\pihat -\pi \|^2_2 \|P - P_2\|_2 \\
    & \stackrel{\zeta_2}{\leq} \frac{4c_2 \sqrt{n} \|\pi\|_\infty }{\sqrt{k}}\| \Pi P + P\Pi - \PE(\mathbf{1}_n \pi^\T) \|_\F\! +\! \frac{3 c^2_2 n \|\pi\|^2_\infty}{k} ,
    \end{align*}
    where the explanation for the middle term of $\zeta_1$ is provided below, and $\zeta_2$ follows since the event $\mathcal{A}_2$ holds and $P - P_2$ is a non-negative matrix, which implies that 
    $$
    \begin{aligned}
        \|P -P_2\|_2 &\leq \sqrt{ \max_{i \in [n]} \|e_i^\T(P - P_2)\|_1  } \sqrt{ \max_{j \in [n]} \|(P - P_2)e_j\|_1} \\
        & \leq  d_{\max}. 
    \end{aligned}
    $$ 
    Moreover, we have $\|P_3 \1_n\|_\infty \leq 2 d_{\max}$. The middle term in $\zeta_1$ is bounded as  
    \[
    \begin{aligned}
        & \sum_{\substack{(i,j) \in \calE: \\ i \neq j} }2\left( \hat{\pi}_{i}-\pi_{i}\right) p_{ij} \left(  {\pi}_{i} p_{ij} + \pi_{j} p_{ij} - \pi_j  \right) \\
         & \ \ \ \ \ \ \  = \sum_{i=1}^n 2\left( \hat{\pi}_{i}-\pi_{i}\right) \sum_{\substack{j: (i,j) \in \calE, \\ j \neq i} } p_{ij} \left(  {\pi}_{i} p_{ij} + \pi_{j} p_{ij} - \pi_j  \right) \\
        &\ \ \ \ \ \ \   \leq 2\sqrt{d_{\max}} \|\hat{\pi}-\pi\|_2 \|\Pi P + P\Pi - \PE(\1_n \pi^\T)  \|_\F ,
    \end{aligned}
    \]
    where last inequality follows by utilizing the fact at most $d_{\max}$ entries are non-zero in any row of the matrix $\Pi P + P\Pi - \PE(\1_n \pi^\T),$ and the fact that $(\sum_{i=1}^n a_i)^2 \leq (\sum_{i=1}^n |a_i|)^2 \leq n \sum_{i=1}^n a_i^2$.  This holds because the absolute value of each entry of $P_2 -P$ and $P_3$ is upper bounded by 1. 
    Similarly, we can obtain the corresponding lower bounds on $T_2$ in the same fashion as above, and this completes the proof. \qed

\subsubsection{Proof of \cref{lem:H0proof}} \label{Proof of Mean and Variance of T3}
~\newline
    \indent  
    \textbf{Part 1:} Since $Z_{ij} \sim \text{Bin}(k_{ij},p_{ij})$, we have
    \begin{align*}
        \E[T_3] &= \sum_{i=1}^n \sum_{j=1 }^n(\pi_i+\pi_j)^2  \frac{\E[Z_{ij}^2]  - \E[Z_{ij}]}{k_{ij}(k_{ij} - 1)} +\pi_j^2 \\ 
        & \ \ \ \ \ \ \ \ \ \ \ \ \ \ - 2(\pi_i + \pi_j) \pi_j \frac{\E[Z_{ij}] }{k_{ij}} \\
        & =  \sum_{i=1}^n \sum_{j=1 }^n (\pi_i+\pi_j)^2p_{ij}^2 + \pi_j^2  - 2(\pi_i + \pi_j) \pi_j p_{ij}\\
        & = \|\Pi P - P\Pi -\PE(\1_n \pi^\T) \|^2_\F .
    \end{align*}

    \textbf{Part 2:} Now, to bound $\var(T_3)$, we assume $k_{ij} = k$ for all $i,j \in [n]$, and thus, we have
     \begin{align*}
    & \var\left(T_{3}\right) \! =\!\sum_{\substack{(i,j) \in \calE: \\ i \neq j} }\! \left(\pi_{i}+\pi_{j}\right)^{4} \var(\hat{Y}_{i j}) + 4 \pi_{j}^{2}\left(\pi_{i}+\pi_{j}\right)^{2} \var(\hat{p}_{ij}) \\
     &  \ \ \ \ \ \ \ \ -4 \left(\pi_{i}+\pi_{j}\right)^{3} \pi_{j}\left(\E[\hat{Y}_{i j} \hat{p}_{ij}]-\E[\hat{Y}_{i j}] \E[\hat{p}_{ij}]\right) \\
    & \stackrel{\zeta_1}{=} \sum_{\substack{(i,j) \in \calE: \\ i \neq j} } \left(\pi_{i}+\pi_{j}\right)^{4}\left(\frac{2 p_{ij}^{2}+4(k-2) p_{ij}^{3}+(6-4 k) p_{ij}^{4}}{k(k-1)}\right) \\ 
    &  \ \ \ \ \ \ \ \ \ \ \ \ \ \ \ +\frac{4 \pi_{j}^{2}\left(\pi_{i}+\pi_{j}\right)^{2} p_{ij}\left(1-p_{ij}\right)}{k} \\
    & \ \ \ \ \ \ \ \ \  -4 \pi_{j}\left(\pi_{i}+\pi_{j}\right)^{3}\left(\frac{ 2 p_{ij}^2 - 2 p_{ij}^3}{k}\right) \\
    & \stackrel{}{=} \sum_{\substack{(i,j) \in \calE: \\ i \neq j} } \frac{2(\pi_{i}+\pi_{j})^{2}}{k} \bigg(\frac{(\pi_{i}+\pi_{j})^{2} (p_{ij}^{2} - 2p_{ij}^3 +p_{ij}^4)}{k-1} \\
    &  \ \ \ \ \ \ \ \ \ \ \ + 2 (\pi_{i}+\pi_{j})^{2}p_{ij}^3(1 - p_{ij}) \\ 
    & \ \ \ \ \ \ \ \ \ \ \  + 2 \pi_j^2 p_{ij}(1-p_{ij}) -  4\pi_j (\pi_{i}+\pi_{j}) p_{ij}^{2}(1-p_{ij}) \bigg)\\ 
    & \stackrel{\zeta_2}{\leq } \sum_{\substack{(i,j) \in \calE: \\ i \neq j} } \frac{2(\pi_{i}+\pi_{j})^{2}}{k} \bigg(\frac{(\pi_{i}+\pi_{j})^{2} p_{ij}^{2}(1 - p_{ij})^2}{k-1} \\ 
    &  \ \ \ \ \ \ \ \ \ \ \ \ \ \ \ \  \ \ \ \ \ \ \ \ \ \ \ \ \ \ \ \ \ \ + \frac{1}{2}( (\pi_i +\pi_j)p_{ij} - \pi_j )^2  \bigg) \\
    & \stackrel{\zeta_3}{\leq}  \frac{4 \| \pi\|_{\infty}^{2}}{k }\|\Pi P+P \Pi-\PE(\1_n \pi^\T)\|_{\F}^{2} \\ 
    & \ \ \ \ \ \ \ \ \ \ \ \ \ \ \ \ \ \ \ \ \ \ \ \ + \frac{32 \|\pi\|_{\infty}^{4}n d_{\max} }{k(k-1)} \times \frac{1}{16} \\
    & \leq  \frac{4\|\pi \|^2_{\infty}}{k}\|\Pi P+P \Pi-\PE(\1_n \pi^\T)\|_{\F}^{2}+\frac{4}{k^2}n d_{\max} \|\pi\|_{\infty}^{4},
    \end{align*}
    where $\zeta_1$ follows from the moments of the binomial random variable as described in \cref{Yijproperties} along with additional calculations provided below, $\zeta_2$ follows by upper bounding $p_{ij}(1-p_{ij})$ by $1/4$ and $\zeta_3$ follows by upper bounding $\pi_i$ by $\|\pi\|_\infty$ and $p_{ij}^2(1-p_{ij})^2 $ by $1/16$. To show $\zeta_1$, note that
    
    \begin{align*}
    & \E[\hat{Y}_{i j} \hat{p}_{ij}]-\E[\hat{Y}_{i j}]  \E[\hat{p}_{ij}]  = 
     \E\bigg[ \frac{Z^2_{ij}(Z_{ij}-1)}{k^2 (k -1)}\bigg]-p_{ij}^3 \\
    & = \frac{ k(k-1)(k-2)p_{ij}^3 + 3k(k-1)p_{ij}^2+ k p_{ij}}{k^2 (k -1)} \\
    & \ \ \ \ \ \ \ \ \ \  - \frac{kp_{ij}(1-p_{ij}) + k^2p_{ij}^2}{k^2 (k -1) } - p_{ij}^3\\
    & = \frac{2p_{ij}^2 - 2p_{ij}^3}{k}. 
    \end{align*}
    This completes the proof. \qed

\subsubsection{Proof of \cref{lem:NoiseLB}} \label{Proof of NoiseLB}
We aim to establish an upper bound on $\| S - \hat{S} \|_2$. Recall that for $(i,j) \in \calE$ and $i \neq j,$ we have 
\begin{align}\label{eq:f1}
\hat{S}_{ij} & = \frac{1}{k d} Z_{ij}, 
\end{align}
where $Z_{ij}$ follows a binomial distribution with parameters $k$ and $p_{ij}$, and $Z_{ij} = 0$ otherwise. Moreover, the random variables $Z_{ij}$ are independent for all $(i, j) \in \calE$ and $i \neq j$.  
Define, a diagonal matrix $D \in \R^{n \times n},$ where $D_{ii} = S_{ii} - \hat{S}_{ii}$ for $i \in [n].$ An application of triangle inequality allows us to separately upper bound the spectral norm due to diagonal and off-diagonal entries as
\begin{equation}
    \| S - \hat{S} \|_2 \leq \|D\|_2 + \| S_0 - \hat{S}_0 \|_2,
\end{equation}
where $S_0$ and $\hat{S}_0$ are the same as matrices $S$ and $\hat{S},$ respectively, but with diagonal elements set to $0$. To bound $\|D\|_2,$ Observe that
\begin{align}
        D_{ii} & = S_{ii} - \hat{S}_{ii} \nonumber \\
        & = \left(\!1-\frac{1}{d} \sum_{\substack{ j: j\neq i,\\ (i,j) \in \calE }} p_{ij}\!\right)\!- \! \left(\!1-\frac{1}{d}\sum_{\substack{ j: j\neq i,\\ (i,j) \in \calE }} \hat{p}_{ij}\!\right) \nonumber \\
        & = \frac{1}{k d}\sum_{\substack{ j: j\neq i,\\ (i,j) \in \calE }} \sum_{m=1}^k (Z_{m,{ij}} - p_{ij}) \, . \label{eq:Dii}
\end{align}
Also, since $D$ is a diagonal matrix, we have $\|D\|_2 = \max_i |D_{ii}|$. And by \cref{eq:Dii}, for any fixed $i$, $ k d D_{ii}$ is a sum of at most $k d_{\max}$ independent, zero-mean random variables and each random variable takes values in $(-1, 1)$. Therefore, by applying Hoeffding's inequality,  we have 
\begin{align} \label{azuma}
\forall i \in [n], \ \P\big(kd |D_{ii}| > t \big) & \leq 2 \exp\bigg(-\frac{2t^2}{ 4 kd_{\max} }\bigg). 
\end{align}
Setting $t = 3 \sqrt{k d_{\max}\log n}$, and by application of union bound we get 
\begin{align}
  \P \bigg (\|D\|_2 \geq & \ 3 \sqrt{\frac{d_{\max} \log n}{k d^2}}  \bigg) \nonumber \\ 
  & \leq \sum_{i=1}^n \P\Big(|D_{ii}| > 3\sqrt{\frac{d_{\max} \log n}{kd^2}} \Big) \leq O( n^{-3}).
\end{align}

\textbf{Bounding $\| S_0 - \hat{S}_0\|_2$:} The bound follows by a direct application of \cref{lem:Spectral Norm of Error} (see below) on $\| kd(S_0 - \hat{S}_0)\|_2$. 
Substituting $t =  2 \sqrt{k d_{\max} \log n }$ utilizing our assumption that $d_{\max}\geq \log n$, we obtain 
\begin{equation*}
    \P \big( \| kd(S_0 - \hat{S}_0)\|_2 \geq 2 \sqrt{k d_{\max} \log n} \big) \leq 2 n \times \frac{1}{n^4}.
\end{equation*}
Thus, the following bound holds with probability at least $1 - O(n^{-3})$.
\begin{equation}
\|S_0 - \hat{S}_0\|_2 \leq 2 \sqrt{\frac{d_{\max} \log n}{ k d^2} }. \label{offdiagonalS} 
\end{equation}
Combining \cref{azuma} and \cref{offdiagonalS} and using the fact that $d \geq 2d_{\max} $ completes the proof. \qed

\begin{lemma}[Spectral Norm of Error] \label{lem:Spectral Norm of Error}
Let $A \in \R^{n \times n}$ be a matrix such that $A_{ij} \sim \textup{Bin}(p_{ij},k)$ for every $(i,j) \in \calE$ with $i \neq j$ and $0$ otherwise. Then, we have
\begin{equation*}
    \P ( \| A - \E[A] \|_2 \geq t ) \leq 2 n \exp \left ( \frac{-t^2/2}{k d_{\max}/4 + t/3 }\right ).
\end{equation*}
\end{lemma}

\begin{proof}
We begin by representing $A_{ij}$ as a sum of $k$ independent Bernoulli random variables for $(i,j) \in \calE$ and $i \neq j$. Specifically, let $A_{ij} = \sum_{m=1}^k Z_{m,{ij}},$ where $Z_{m,{ij}} \sim \Ber(p_{ij})$. Also, define matrices $U^{ij},$ where 
\begin{equation*}
    U^{ij}_m = (Z_{m,{ij}} - p_{ij}) e_i e_j^\T \ \  \text{ for $(i,j)\in \calE$ and $i \neq j$},  
\end{equation*}
and $U^{ij}_m = 0$ otherwise. Now define $\tilde{U}^{ij}$, so as to symmetrize the matrices $U^{ij}$:
\begin{equation*}
  \tilde{U}^{ij} = \left (
    \begin{matrix}
    0 & U^{ij} \\
    (U^{ij})^\T & 0
    \end{matrix}  \right ).
\end{equation*}
Observe that $A - \E[A] = \sum_{(i,j) \in \calE, i \neq j} \sum_{m=1}^k U_m^{ij}$ and $\|A - \E[A]\|_2 = \lambda_{\max}( \sum_{(i,j) \in \calE: i \neq j} \sum_{m=1}^k \tilde{U}_m^{ij} )$. Now, we will employ the matrix Bernstein inequality to analyze the sum of independent, zero-mean self-adjoint random matrices. Clearly $\|\tilde{U}_m^{ij}\|_2 \leq 1$, therefore we have
\[
    \P\left(\lambda_{\max} \left( \sum_{\substack{(i,j) \in \calE:\\ i \neq j}} \sum_{m=1}^k \tilde{U}_m^{ij} \right) \geq t\right) \leq 2n \exp\left (\frac{-t^2/2}{\sigma^2 + t/3}\right),
\] 
where 
$$
\begin{aligned}
    \sigma^2 & = \left\|\sum_{(i,j) \in \calE: i \neq j} \sum_{m=1}^k \E[(\tilde{U}_m^{ij})^2]\right\|_2 \\
    & = \left\|\sum_{(i,j) \in \calE: i \neq j} \sum_{m=1}^k p_{ij}(1-p_{ij}) \left (
        \begin{matrix}
            e_{i}e_{i}^\T & 0 \\
            0 & e_j e_j^\T
        \end{matrix}  \right ) \right\|_2 \\ 
        & \leq \frac{k d_{\max}}{4}. 
    \end{aligned}
$$ 
This completes the proof. 
\end{proof}

\subsection{Proof of \cref{thm:Type-1 and Type-2 error}} \label{Proof of Type-1 error}

First, we begin by deriving bounds on the probability of type \Romannum{1} error. Similar to the proof of \cref{Main theorem with upper bound}, we split the test statistic $T$ as $T = T_1 + T_2 + T_3$, where $T_1, T_2$ and $T_3$ are defined in \cref{definiton of T1,definiton of T2,definiton of T3}. To calculate the $\P(T \geq t)$ for {$t \geq 0$}, we use the following inequality
\begin{align}
     & \P(T \geq t) \leq \nonumber \\
     & \ \quad \, \P\left(T_3 \geq \left(t - \frac{ \tilde{c}_\beta \sqrt{n} \|\pi\|_\infty}{ \sqrt{k}} \|\Pi P + P \Pi - \PE(\1_n \pi^\T)  \|_\F  \right. \right. \nonumber \\
     &\ \ \ \ \ \ \ \ \ \ \ \ \ \ \ \ \ \ \ \ \ \ \ \ \ \ \ \ \ \ \ \ \ \ \ \ \ \ \ \ \ \ \ \ \left. \left.  - (\tilde{c}_\alpha+ \tilde{c}_\gamma) \frac{n\|\pi\|^2_\infty }{ k}\right) \right) \nonumber \\
     &\ + \P\bigg(T_1 \geq \tilde{c}_\alpha \frac{n\|\pi\|^2_\infty }{ k} \bigg) \nonumber\\ 
     &\ + \P\bigg(T_2 \geq \frac{\tilde{c}_\beta \sqrt{n} \|\pi\|_\infty}{ \sqrt{k}} \|\Pi P + P \Pi - \PE(\1_n \pi^\T)  \|_\F \nonumber \\
     &\ \ \ \ \ \ \ \ \ \ \ \ \ \ \ \ \ \ \ \ \ \ \ \ \ \ \ \ \ \ \ \ \ \ \ \ \ \ \ \ \ \ \ \ \ \ \ \ \ +  \tilde{c}_\gamma \frac{n\|\pi\|^2_\infty }{ k} \bigg) \nonumber\\
     & \stackrel{\zeta_1}{\leq} \P(T_3 \geq \tilde{t}) + O\bigg( \frac{1}{n^3}\bigg),  \label{split bound}
\end{align}
 where $\zeta_1$ follows from \cref{lem: Bounds on T1 and T2}.
 Now we will derive tail bounds for $T_3$. To proceed, consider the quantity $T^{i j}_3$ for $(i,j) \in \calE$ and $i \neq j$ defined as the $(i,j)$th term of $T_3$:
$$
T_3^{i j} \triangleq  (\pi_{i}+\pi_{j} )^{2} \frac{Z_{i j} (Z_{i j}-1 )}{k(k-1)}+\pi_{j}^{2}-2 (\pi_{i}+\pi_{j} ) \pi_{j} \frac{Z_{i j}}{k}.
$$
Recall that $Z_{ij} = \sum_{m=1}^k Z_{m,{ij}}$. Clearly, $T_3^{i j}$ is a function of ${Z}_{m,{i j}}$ for $m \in [k]$. Observe that \(T_3^{i j} - \E[T_3^{i j}]\) can be expressed as 
\begin{align*}
 & k(k-1)\big(T_{3}^{i j}-\mathbb{E}[T_{3}^{i j}]  \big) \\
 &=(\pi_{i}+\pi_{j})^{2} Z_{i j}(Z_{i j} -1)+\pi_{j}^{2} k(k-1)\\
 & \ \ \ \ \ \ \ - 2(k-1)(\pi_{i}+\pi_{j}) \pi_{j} Z_{i j}\\
& \ \ \ \ \ \ \  -k(k-1)\left(\left(\pi_{i}+\pi_{j}\right) p_{i j}-\pi_{j}\right)^{2} \\
& =\left(\pi_{i}+\pi_{j}\right)^{2}\left(\sum_{m=1}^{k} Z_{m,{i j}}\right)\left(\sum_{m=1}^{k} Z_{m,{i j}}-1\right)\\
& \ \ \ \ \ \ \ -2(k-1)\left(\pi_{i}+\pi_{j}\right) \pi_{j}\left(\sum_{m=1}^{k} Z_{m,{i j}}\right) \\
& \ \ \ \ \ \ \  -k(k-1)\left(\pi_{i}+\pi_{j}\right)^{2} p_{i j}^{2}+2 k(k-1)\left(\pi_{i}+\pi_{j}\right) \pi_{j} p_{i j} \\
& =\left(\pi_{i}+\pi_{j}\right)^{2}\bigg(\sum_{m=1}^{k} \sum_{\substack{m^{\prime} \neq m, \\ m^\prime=1}}^{k}\left(Z_{m,{i j}} Z_{m^{\prime},{i j}}-p_{i j}^{2}\right)\bigg)\\ 
& \ \ \ \ \ \ \ \ \ \ -2(k-1)\left(\pi_{i}+\pi_{j}\right) \pi_{j}\left(\sum_{m=1}^{k}\left(Z_{m,{i j}}-p_{i j}\right)\right).
\end{align*}
Define a vector \(z_{i j} \in \mathbb{R}^{k}\) such that \(z_{i j}=\left[Z_{1,{i j}}, \dots , Z_{k, { i j}}\right]^\T\) and a matrix \(A \in \mathbb{R}^{k \times k}\) such that \(A= \1_{k} \1_{k}^{\T}-I_k\). Observe that 
\begin{align*}
 & k(k-1) (T_{3}^{i j}- \E[T_{3}^{i j}]) =\left(\pi_{i}+\pi_{j}\right)^{2}\left(z_{i j}^{\T} A z_{i j}-p_{i j}^{2} \1_{k}^{\T} A \1_{k}\right)\\ 
 & \ \ \ \ \ \ \ \ \ -2(\pi_{i}+\pi_{j}) \pi_{j}\left(z_{i j}-p_{i j} \1_{k}\right)^{\T} A \1_{k} . 
 \end{align*}
Utilizing the fact that  
 \begin{align*}
& (z_{i j}-p_{i j} \1_{k})^{\T} A(z_{i j}-p_{i j} \1_{k})\\ 
& \ \ \ \ \ \ \ =z_{i j}^{\T} A z_{i j}-2 p_{i j}(z_{i j}-p_{i j} \1_{k})^{\T} A \1_{k} -p_{i j}^{2} \1_{k}^{\T} A \1_{k} ,
\end{align*}
we obtain 
\begin{align}
& \nonumber k(k-1)(T_3^{ij} -  \mathbb{E}[T_{3}^{ij}]) \\ 
& =(\pi_{i}+\pi_{j})^{2}  (z_{i j}-p_{i j} \1_{k})^{\T} A (z_{i j}-p_{i j} \1_{k}) \nonumber \\ 
& + 2 (\pi_{i}+\pi_{j}) ((\pi_{i}+\pi_{j}) p_{i j}-\pi_{j}) (z_{i j}-p_{ij} \1_{k})^{\T} A \1_{k}. \label{decomposition of T3 around expected value}
\end{align}
Observe that under hypothesis \(H_0\), the last term of the above equation reduces to zero.
We will utilize Hanson-Wright inequality \cite[Theorem 1.1]{rudelson2013hanson} to bound the first term. Therefore, under hypothesis \(H_{0}\) and for some constant $c$ we have
\begin{align}
 & \mathbb{P}\left(k(k-1) \sum_{\substack{ (i,j) \in \calE:\\ i \neq j} } T_{3}^{ij} \geq {k}(k-1) t \right) \nonumber \\
  & = \P\left(\sum_{\substack{ (i,j) \in \calE:\\ i \neq j} } \frac{(\pi_{i}+\pi_{j})^{2} (z_{i j}-p_{i j} \1_{k})^{\T} A (z_{i j}-p_{i j} \1_{k}) }{k(k-1)} \geq  t \right) \nonumber\\
 & \leq \P\left(\sum_{\substack{ (i,j) \in \calE:\\ i \neq j} } \frac{4\|\pi\|_{\infty}^2 (z_{i j}-p_{i j} \1_{k})^{\T} A (z_{i j}-p_{i j} \1_{k}) }{k(k-1)} \geq  t \right) \nonumber\\
& \stackrel{\zeta_1}{\leq} \exp \left(-c\min \left\{\frac{k^2(k-1)^{2} t^2}{\|\tilde{A}\|_\F^2 \cdot 4 \|\pi\|_{\infty}^{4} \cdot \left(\frac{1}{2}\right)^4} \right., \right. \nonumber \\
&  \ \ \ \ \ \ \ \ \ \ \ \ \ \ \ \ \ \ \ \ \ \ \ \ \ \ \ \ \ \ \ \ \ \ \ \left. \left. \frac{k(k-1) t}{\|\tilde{A}\|_{2} \cdot 2\|\pi\|_{\infty}^2\cdot \left(\frac{1}{2}\right)^2 } \right\} \right) \nonumber \\
& \stackrel{\zeta_2}{=} \exp \bigg(-c\min \left\{\frac{ 4 k(k-1) t^2}{ |\calE|  \|\pi\|_{\infty}^{4} }, \frac{2k t}{ \|\pi\|_{\infty}^2 } \right\} \bigg),\label{Hanson}
\end{align}
where $\zeta_1$ follows by applying the Hanson-Wright inequality to the concatenated matrix $\tilde{A} \in \R^{|\calE|k \times |\calE|k }$ with matrix $A$ along its diagonal $|\calE|$ times and utilizing the fact that the entrywise sub-Gaussian norm of $Z_{m,{ij}} -p_{ij} $ is upper bounded by $ \frac{1}{2}$, and $\zeta_2$ follows because $\|\tilde{A}\|^2_\F = |\calE| k (k-1) $ and $\|\tilde{A}\|_2 = \|A\|_2 = k-1$. Finally, utilizing $|\calE| \leq n d_{\max}$, we obtain
\begin{align*}
\mathbb{P}\left(T_{3}>t\right) \leq \exp \bigg(-c\min \bigg\{\frac{ 4 k(k-1) t^2}{  n d_{\max  }\|\pi\|_{\infty}^{4} }, \frac{2k t}{ \|\pi\|_{\infty}^2 } \bigg\} \bigg). 
\end{align*}
To obtain tail bounds under hypothesis \(H_{1}\), we again utilize \cref{decomposition of T3 around expected value} as follows:
\begin{align*}
 & \mathbb{P}\left(\sum_{\substack{ (i,j) \in \calE:\\ i \neq j} } (T_{3}^{i j}-\E[ T_{3}^{i j}]  )  \leq -  t\right) \leq \\
 & \mathbb{P}\left( \sum_{\substack{ (i,j) \in \calE:\\ i \neq j} } (\pi_{i}+\pi_{j})^{2}(z_{i j}-p_{i j} \1_{k})^{\T} A (z_{i j}-p_{i j} \1_{k}) \right. \\ 
 &\left. \vphantom{\sum_{\substack{ (i,j) \in \calE:\\ i \neq j} }} \ \ \ \ \ \ \ \ \ \ \ \ \ \ \ \ \ \ \ \ \ \ \ \ \ \ \ \ \ \ \ \ \ \ \ \ \  \leq -\frac{k(k-1) t}{2} \right) \ +  \\
& \P\left( \sum_{\substack{ (i,j) \in \calE:\\ i \neq j} }2(\pi_{i}+\pi_{j})( (\pi_{i}+\pi_{j}) p_{i j}-\pi_{j}) (z_{i j}- p_{ij} \1_{k})^{\T} A \1_{k} \right. \\ 
& \left. \vphantom{\sum_{\substack{ (i,j) \in \calE:\\ i \neq j} }}\ \ \ \ \ \ \ \ \ \ \ \ \ \ \ \ \ \ \ \ \ \ \ \ \ \ \ \ \ \ \ \ \ \ \ \ \ \leq - \frac{k(k-1) t}{2}\right) \\
& \stackrel{\zeta_1}{\leq}  \exp \bigg(-c\min \bigg\{\frac{ 4 k(k-1) t^2}{  |\calE|\|\pi\|_{\infty}^{4} }, \frac{2k t}{ \|\pi\|_{\infty}^2 } \bigg\} \bigg) \ +\\ 
&  \mathbb{P}\left( \sum_{\substack{ (i,j) \in \calE:\\ i \neq j} } \sum_{m=1}^{k} (\pi_{i}+\pi_{j})( (\pi_{i}+\pi_{j} ) p_{i j}-\pi_{j} )  (Z_{m,{i j}} -  p_{i j}) \right. \\
& \left. \vphantom{\sum_{\substack{ (i,j) \in \calE:\\ i \neq j} }}\ \ \ \ \ \ \ \ \ \ \ \ \ \ \ \ \ \ \ \ \ \ \ \ \ \ \ \ \ \ \ \ \ \ \ \ \ \leq -\frac{k t}{4}\right) \\ 
& \stackrel{\zeta_2}{\leq} \exp \bigg(-c\min \bigg\{\frac{ 4 k(k-1) t^2}{  \|\pi\|_{\infty}^{4} }, \frac{2kt}{ \|\pi\|_{\infty}^2 } \bigg\} \bigg) \\ 
& \ \ \ \ \ \ \ + \exp \left(\frac{-2 k t^{2} / 16}{4\|\pi\|_{\infty}^{2} \|\Pi P + P \Pi  - \PE(\1_n \pi^\T)\|_\F^{2} }\right),
\end{align*}
where $\zeta_1 $ follows from \cref{Hanson} and $\zeta_2$ follows by applying Hoeffding's inequality. Therefore, we obtain 
\[
\begin{aligned}
 \mathbb{P}\bigg(T_{3} - \E[T_{3}]  & \leq - t\bigg) \leq \\
 & \exp \bigg(-c\min \bigg\{\frac{4  k(k-1) t^2}{ n d_{\max} \|\pi\|_{\infty}^{4} }, \frac{2k t}{ \|\pi\|_{\infty}^2 } \bigg\} \bigg) + \\
 &  \exp \left(\frac{- k t^{2} }{32\|\pi\|_{\infty}^{2} \| \Pi P+ P\Pi -\PE(\1_n \pi^\T)\|_\F^2 }\right).
\end{aligned}
\] 
This completes the proof. \qed

\section{Proof of Lower Bound} \label{sec: Proof of Lower Bound and Stability}
In this section, we will provide the proof of \cref{thm:Lower Bound}.

\begin{proof}[Proof of \cref{thm:Lower Bound}] 
We will apply the \emph{Ingster-Suslina method} to establish a lower bound on the critical threshold \cite{IngsterSuslina2003}. Throughout the proof, we will assume that the induced graph is complete.
Moreover, under the null hypothesis, we assume that the pairwise comparison matrix $P$ is fixed to be an all $1/2$ matrix, i.e.,
\begin{equation}
H_{0}: \quad P = P_0 \triangleq \frac{1}{2} \mathbf{1}_n\mathbf{1}_n^\T. \label{Lower bound H0}
\end{equation}
We will denote the distribution corresponding to the pairwise comparison matrix $P_0$ by $\P_0$. Additionally, note that under $H_0$, the stationary distribution of the canonical Markov matrix $\S$ is uniform, i.e., $\pi=\frac{1}{n} \mathbf{1}_n$. Under the alternative hypothesis, we assume that the pairwise comparison matrix $P_\theta$ is generated by sampling the parameter $\theta$ uniformly from the set $\Theta$, i.e., 
\begin{equation}
H_1:  \quad  P = P_\theta \text{ and } \theta \sim \text{Unif}(\Theta),  \label{Lower bound H1}
\end{equation}
and for $\theta \in \Theta$, $P_{\theta}$ is given by
\begin{equation}
\begin{aligned}
& P_\theta = \left[\begin{array}{cccc|cccc}
\frac{1}{2} \1_{n/2}\1_{n/2}^\T  & \frac{1}{2} \1_{n/2}\1_{n/2}^\T + \eta Q_{\theta} \\ 
\frac{1}{2} \1_{n/2}\1_{n/2}^\T -\eta Q_{\theta}  & \frac{1}{2} \1_{n/2}\1_{n/2}^\T \\ 
\end{array}
\right],
\end{aligned} \label{pthetaconstruction}
\end{equation}
where $\Theta$ is set of all permutation matrices, and $\Q_{\theta}$ is the $\frac{n}{2} \times \frac{n}{2}$ permutation matrix corresponding to the permutation $\theta$ and the perturbation $\eta \in (0,\frac{1}{2})$. 
Let $\P_\Theta$ denote the overall mixture distribution and $\P_{\theta}$ denotes the distribution corresponding to the pairwise comparison matrix $P_\theta$. The construction of this mixture was inspired by \cite{RastogiBalakrishnanShahAarti2022}. However, there are two notable differences. Firstly, the problem in \cite{RastogiBalakrishnanShahAarti2022} is distinguishing whether two sets of data samples consisting of pairwise comparisons are coming from the same underlying distribution or two different distributions described by a pairwise comparison model. In contrast, our work tests whether or not a single dataset is sampled from a BTL model. Secondly, the manner in which a notion of distance is used to define the deviation of the given data from the null hypothesis is different in the two works. 

Let $S_\theta$ denote the canonical Markov matrix corresponding to $P_{\theta}$. It is straightforward to verify that the stationary distribution of $S_\theta$ is independent of the permutation $\theta$. Let $\pi$ denote the stationary distribution of $S_\theta$. By the symmetry of $S_\theta$, the set of first $n/2$ elements, and respectively, last $n/2$ elements, of $\pi$ are equal, i.e., $\pi_1 = \dots = \pi_{n/2} \triangleq x$, and $\pi_{(n/2) +1} = \dots = \pi_{n} \triangleq y$. Now $x$ and $y$ can be determined by solving the set of linear equations:
$$
    \pi^\T = \pi^\T S_\theta \text{ and } \sum_{i=1}^n \pi_i = 1 \, .     
$$
Solving these equations gives  
$$ x = \frac{1}{n} \left(1-\frac{4\eta}{n} \right) \text{ and }  y = \frac{1}{n}\left(1+ \frac{4\eta}{n}\right) . $$
It is also easy to verify that the deviation from BTL $\| \Pi P_\theta + P_\theta\Pi - \1_n \pi^\T \|_{\F}$ is also independent of the permutation $\theta$ and is given by
\begin{equation} \label{Scaling of lower bound}
\begin{aligned}
& \left\|\Pi  P_\theta + P_\theta \Pi-\mathbf{1}_n\pi^\T \right\|_{\F}^{2} = \\
& \ \frac{n}{2} \left( (x+y) \left(\frac{1}{2} + \eta \right) -y\right)^2 \!  + \frac{n}{2} \left( (x+y) \left(\frac{1}{2} - \eta \right) - x \right)^2 \\ 
& \  +  \frac{n}{2} \left(\frac{n}{2}-1\right) \left(\frac{x+y}{2} -y\right)^2 + \frac{n}{2} \left(\frac{n}{2}-1\right) \left(\frac{x+y}{2} -x\right)^2 \\
& = \frac{2\eta^2}{n}\left( 1 - \frac{2 }{n} \right)^2 + \frac{2\eta^2}{n^2} \left( 1 -\frac{2}{n} \right) . 
\end{aligned}
\end{equation}
Let $\epsilon = \|\Pi P_\theta + P_\theta \Pi-\mathbf{1}_n\pi^\T \|_\F/(n \|\pi\|_\infty)$ to ensure that the $P_{\theta}$'s satisfy the condition of the alternative hypothesis in \eqref{Hypothesis test}. Substituting the values of $\|\pi\|_\infty = y$ and $\|\Pi P_\theta + P_\theta \Pi-\mathbf{1}_n\pi^\T \|_{\F}$ implies that 
\begin{equation}
    \epsilon^2 \leq \frac{4 \eta^2}{n}.  \label{first condition for lb}    
\end{equation}

Now, the Ingster-Suslina method \cite{IngsterSuslina2003} states that
\begin{equation}
     \mathcal{R}_{\mathsf{m}}[\G, \epsilon] \geq   1 - \sqrt{\frac{\chi^2(\P_\Theta||\P_0)}{2}} ,\label{minimax risk lb}
\end{equation}
where $\chi^2(\cdot||\cdot)$ denotes the $\chi^2$-divergence. We compute $\chi^2(\P_\Theta||\P_0)$ by expressing it as an expectation with respect to two independent pairwise models corresponding to permutations $\theta$ and $\theta^\prime$ drawn independently and uniformly at random from $\Theta$ as
\begin{align}
    1+\chi^2(\P_\Theta||\P_0) &= \E_{\theta, \theta^\prime \sim \text{Unif}(\Theta)} \left[ \int \frac{d\P_{\theta} d\P_{\theta^\prime}}{d \P_0} \right] , \nonumber
\end{align}
where $d\P_\theta$ denotes the measure induced by the pairwise comparison model corresponding to permutation $\theta$.
Let $p_{ij }$ and $p_{ij}^\prime$ be the pairwise probabilities corresponding to permutations $\theta$ and $\theta^\prime$. 
Now we will essentially utilize the tensorization property of $1 +\chi^2(P||Q)$, i.e., we utilize the fact that for distributions $P_1, Q_1, \dots, P_n, Q_n$, we have  $$1 +\chi^2\left(\left. \prod_{i=1}^n P_i \right| \left|\prod_{i=1}^n Q_i\right.\right) = \prod_{i=1}^n (1 +\chi^2( P_i||  Q_i)). $$ Therefore, $\chi^2(\P_\Theta||\P_0)$ can be simplified as
\begin{align}
    & \nonumber 1+\chi^2(\P_\Theta||\P_0) = \E_{\theta, \theta^\prime \sim \text{Unif}(\Theta)} \left[  \prod_{i=1}^n \prod_{ \substack{j=1: \\ j\neq i}}^n \right. \\
    & \left. \vphantom{\prod_{i=1}^n \prod_{ \substack{j=1: \\ j\neq i}}^n  } \left(\sum_{m =0 }^k  \frac{ {k \choose m } (p_{ij})^{m} (1 - p_{ij})^{k - m} {k \choose m } (p_{ij}^\prime)^{m} (1 - p^\prime_{ij})^{k - m} }{ {k \choose m } \left(\frac{1}{2}\right)^k } \right) \right] \nonumber \\ 
    & = \E_{\theta, \theta^\prime \sim \text{Unif}(\Theta)} \left[  \prod_{i=1}^n \prod_{ \substack{j=1: \\ j\neq i}}^n \right. \nonumber \\ 
    & \left. \vphantom{ \prod_{i=1}^n \prod_{ \substack{j=1: \\ j\neq i}}^n} \left(\sum_{m =0 }^k  \frac{ {k \choose m } (p_{ij})^{m} (1 - p_{ij})^{k - m} (p_{ij}^\prime)^{m} (1 - p^\prime_{ij})^{k - m} }{  \left(\frac{1}{2}\right)^k } \right)  \right]. \label{version 1}
\end{align}
This calculation of simplifying the $\chi^2$-divergence follows an approach similar to that in \cite{RastogiBalakrishnanShahAarti2022}, but we include the details here for completeness. We will focus on the $(i,j)$th term of the product in \Cref{version 1} for $i \neq j$ and denote it as $f(p_{ij}, p^\prime_{ij})$:
\begin{equation}\label{def of f}
f(p_{ij}, p^\prime_{ij})\!=\!\sum_{m =0 }^k \!\frac{ {k \choose m } (p_{ij})^{m} (1\! -\! p_{ij})^{k - m} (p_{ij}^\prime)^{m} (1\! -\! p^\prime_{ij})^{k - m} }{  \left(\frac{1}{2}\right)^k }.
\end{equation}
By our construction of pairwise comparison matrices both $p_{ij}, p^\prime_{ij}$ take values in set $\{\frac{1}{2}, \frac{1}{2}+\eta\}$ if $j\geq i$ (and $\{\frac{1}{2}, \frac{1}{2}-\eta\}$ otherwise). Furthermore, whenever either $p_{ij}$ or $p_{ij}^\prime$ equals $\frac{1}{2}$, we have $f(p_{ij}, p^\prime_{ij})=1$. Additionally, by \cref{def of f}, we have $f(\frac{1}{2}-\eta, \frac{1}{2}-\eta)=f(\frac{1}{2}+\eta, \frac{1}{2}+\eta)$. Let a random variable $B$ denote the number of occurrences where $p_{ij} = p_{ij}^\prime = \frac{1}{2}+ \eta$ in randomly drawn permutations $\theta$ and $\theta^\prime$ (or equivalently, $p_{ij} = p_{ij}^\prime = \frac{1}{2}- \eta$). Consequently, we obtain
\begin{equation}
      1+ \chi^2(\P_\Theta||\P_0)  = \E_{\theta, \theta^\prime \sim \text{Unif}(\Theta)} \left[  f\left(\frac{1}{2}+ \eta, \frac{1}{2}+ \eta\right)^{2B} \right]. \label{chi square simplified} 
\end{equation}
This is because, by translated skew-symmetry, the number of occurrences of $p_{ij} = p_{ij}^\prime = \frac{1}{2}+ \eta$ is the same as the occurrences where $p_{ij} = p_{ij}^\prime = \frac{1}{2}- \eta$ and $f(p_{ij},p_{ij}^\prime)$ is same in both the scenarios. Moreover, we can simplify $f(p_{ij},p_{ij}^\prime)$ as 
\begin{equation*}
\begin{aligned}
    & f\left(\frac{1}{2}+ \eta, \frac{1}{2}+ \eta\right) \\
    &= 2^k \sum_{m=0}^k {k \choose m} \left(\frac{1}{2}+ \eta\right)^{2m} \left(\frac{1}{2}- \eta\right)^{2k-2m}    \\
    & = 2^k \sum_{m=0}^k {k \choose m} \left(\frac{1}{4}+ \eta^2 + \eta\right)^{m} \left(\frac{1}{4} + \eta^2 - \eta\right)^{k-m}    \\
    & = 2^k \left(\frac{1}{2} + 2\eta^2\right)^k \times \\
    & \ \ \ \ \ \ \ \left(\sum_{m=0}^k {k \choose m} \left(\frac{1}{2}+ \frac{\eta}{\frac{1}{2}+ 2\eta^2}\right)^{m} \left(\frac{1}{2}- \frac{\eta}{\frac{1}{2}+ 2\eta^2}\right)^{k-m} \right) \\ 
    & = (1 + 4\eta^2)^k.
\end{aligned}
\end{equation*}
The above equation combined with \cref{chi square simplified} gives
\begin{align}
    1+  \chi^2(\P_\Theta||\P_0) &=  \E_{\theta, \theta^\prime \sim \text{Unif}(\Theta)} \bigg[   (1+4\eta^2)^{2kB}  \bigg] \nonumber \\
    &   =  \sum_{b=0}^{n/2}\P(B = b)(1+4\eta^2)^{2kb}. \label{Almost simplified}
    \end{align}
To derive an upper bound on $\P(B = b)$, for a fixed permutation $\theta$, we aim to find another permutation $\theta^\prime$ with exactly $b$ aligned elements. These $b$ matches can be selected in ${ {n/2} \choose b }$ ways. For the remaining $n/2 - b$ elements, we require derangements to avoid alignment with the perturbed elements in $\theta$. Therefore, the number of such permutations is  $(\frac{n}{2} - b)! \sum_{i=0}^{n/2-b} \frac{(-1)^i}{i!}$. Clearly, this quantity is upper bounded by $\frac{1}{2} (\frac{n}{2} - b)!$, yielding an upper bound on $\P(B=b)$ as
\[
   \P(B=b) \leq \frac{\frac{1}{2} { {n/2} \choose b } (\frac{n}{2} - b)!}{(n/2)!} \leq \frac{1}{2(b!)}.
\]
Substituting the above bound in \cref{Almost simplified}, we obtain 
\begin{align}
   \nonumber \chi^2(\P_\Theta||\P_0) &\leq \sum_{b=0}^{n/2}\frac{1}{2 (b!)} \bigg ( (1+4\eta^2)^{2kb} -1 \bigg) \\
   & \nonumber \ \ \ \ \ \ \ \ \ \ \ \ \ \ + \sum_{b=0}^{n/2} \P(B=b) - 1\\ 
  \nonumber  & \stackrel{\zeta}{\leq}  \frac{1}{2} \sum_{b=0}^{n/2}\frac{1}{ b!} ( (1+c^\prime)^{b} -1) \\
    & \leq \frac{1}{2} (e^{1+c^\prime} - e) + \frac{1}{2} \sum_{b=\frac{n}{2}+1}^{\infty}\frac{1}{ b!}, \label{almost lb}
\end{align}
where $\zeta$ holds for some constant $c^\prime$ if $8 \eta^2 k \leq \tilde{c}. $
Moreover, the quantity in \cref{almost lb} is bounded above by $\frac{1}{4}$ if $\tilde{c}$ is small enough. However, by \cref{first condition for lb}, we have $\eta^2 \geq \frac{1}{4} n \epsilon^2$. Therefore, we have shown that if $ 2 nk \epsilon^2 \leq \tilde{c}$, for sufficiently small constant $\tilde{c}$ then $\chi^2(\P_\Theta||\P_0) \leq \frac{1}{2}$, which by \cref{minimax risk lb} implies that the minimax risk  $\mathcal{R}_{\mathsf{m}}[\G, \epsilon] \geq \frac{1}{2}$. Hence, $\varepsilon_{\mathsf{c}}^2 \geq c/(kn)$ as desired. 
\end{proof}

\section{Numerical Simulations} \label{sec: Numerical Simulations}
In this section, we will initially introduce a methodology for selecting the threshold for our proposed test. Subsequently, we will validate several of our theoretical findings through simulations. Furthermore, we will examine the outcomes of our proposed test on both synthetic and real-world datasets, thereby exploring its efficacy in practical scenarios.

\subsection{Estimating the Threshold} \label{sec: Estimating the Threshold}
In this section, we aim to select the critical threshold for our hypothesis testing problem for the BTL model. Recall that \cref{Upper Bound on Critical Threshold} indicates the existence of a constant $\gamma$ such that we threshold our test statistic $T$ at the value {$\gamma/nk$}. However, the precise value of $\gamma$ necessary to decide between hypotheses $H_0$ or $H_1$ is unknown. The worst-case constants calculated from the various constants appearing in the analysis may be too large for a smaller $n$ or $k$ (say $n=10$, etc.). Furthermore, this ambiguity may become more profound when the number of comparisons $k_{ij}$ are not the same across all pairs, a common scenario in practice. To address this challenge, we propose two different techniques for selecting the threshold.
    
    \textbf{Empirical Quantile Approach:} In this technique, we randomly generate multiple BTL models with their skill scores $\alpha \in \R^n_+$ selected at random (such that $\max_{i \in [n]} \alpha_i/\min_{j \in [n]} \alpha_j $ is bounded above by some constant). For each model, we generate an equivalent set of synthetic comparisons $k_{ij}$ by independently sampling binomial random variables $\{Z_{ij} \sim \text{Bin}(k_{ij},\alpha_j/(\alpha_i + \alpha_j) )\}_{(i,j) \in \calE}$, matching the original data, and compute the corresponding test statistic $T$.  By repeating this process a sufficient number of times, we build a distribution of test statistics. From this distribution, we extract the $95\%$ percentile value (or $0.95$ quantile) to determine our empirical threshold.

    \textbf{Permutation-Based Scheme:} This approach is motivated by the permutation test method to obtain a sharper threshold for our test in a data-driven manner. Recall that for a standard two-sample hypothesis testing problem, the permutation technique involves randomly shuffling the labels of the two classes. The test statistic is then recalculated for each permutation, obtaining a distribution of the test statistics under the null hypothesis $H_0$. To assess the significance of our observed test statistic, the $p$-value is then computed as the proportion of permuted test statistics that are more extreme than the observed test statistic on unshuffled data. To adapt this technique for the BTL hypothesis testing problem, we perform the following two types of shuffling motivated from \cref{Prop: BTL Model and Reversibility}: 
    \begin{itemize}
        \item \emph{Translated Skew-symmetry:} For each pair $(i,j) \in \calE$ with $i>j$, we collect all the $(k_{ij} + k_{ji})$ samples together and reassign the $k_{ij}$ samples chosen uniformly at random (without replacement) to ``$i$ vs. $j$'' and the rest to ``$j$ vs. $i$''. This shuffling effectively removes any `home advantage effect' between the players, ensuring that the outcomes of matches between $i$ and $j$ are indistinguishable from those between $j$ and $i$.
        \item \emph{Reversibility:} In this approach, we adopt a permutation scheme based on Kolmogorov's loop criterion. Recall that according to Kolmogorov's loop criterion, a Markov transition matrix is reversible if and only if, for every cycle, the forward loop transition probability product is equal to the backward loop transition probability product. 
        To implement the shuffling process, we interpret the given pairwise comparison data, denoted as $\mathcal{Z}$, as corresponding to different transitions of the Markov chain (with associated transition probabilities represented by the canonical Markov matrix of $P$).  Under the null hypothesis $H_0$, where the underlying model is BTL, we know that the canonical Markov matrix representing the process is reversible.
        Consider any cycle of length $l$ in the induced graph, denoted by $i = i_1 \to i_2 \to  \dots \to i_l = i$ for some nodes $ i_1, \dots, i_{l-1} \in [n]$, starting and ending at $i$. For this cycle, we associate two random variables $Z_{FL}$ and $Z_{BL}$ corresponding to the product of observations in a forward loop and a backward loop, defined as:
        \begin{equation}
        Z_{FL} \triangleq \prod_{j=1}^{l-1} Z_{m^j, {i_j i_{j+1}}} \text{ and } Z_{BL} \triangleq \prod_{j=1}^{l-1} Z_{\tilde{m}^j, {i_{l-j+1} i_{l-j}}}, \label{ZFLZBL}
        \end{equation}
        for any fixed $m^1 \in [k_{i_1 i_{2}}], \dots , m^{l-1} \in [k_{i_{l-1}i_{l}}]$ and $\tilde{m}^1  \in [k_{i_{l} i_{l-1}}], \dots , \tilde{m}^{l-1} \in [k_{i_{2}i_{1}}]$. Note that here $Z_{m^j,{i_j i_{j+1}}} $ denotes the $m^j$th observation ($m^j \in [k_{i_j i_{j+1}}] $) of the form ``$i_j$ vs. $i_{j+1}$'' and hence $Z_{m^j, {i_j i_{j+1}}} \sim \Ber(p_{i_j i_{j+1}})$. Therefore, we have that $Z_{FL}$ follows a Bernoulli distribution with parameter $\prod_{j=1}^{l-1} p_{i_j i_{j+1}}$, and $Z_{BL}$ follows a Bernoulli distribution with parameters $\prod_{j=1}^{l-1} p_{i_{l-j+1} i_{l-j}}$.
        Based on Kolgomorov's loop criteria, under the hypothesis $H_0$, we have $\P(\{Z_{FL} = 1\}) $ is the same as $\P(\{Z_{BL}=1\})$ and moreover both the events $\{Z_{FL} =1\}$ and $\{Z_{BL} =1\}$ occur only when each of the random variables $Z_{m^j,{i_j i_{j+1}}} $ and $Z_{\tilde{m}^j,{i_{l-j+1} i_{l-j}}}$ in the product (in \cref{ZFLZBL}) are $1$. Therefore, this motivates the following shuffling process where we essentially replace the data in $\Z$ corresponding to event $\{Z_{FL} = 1\}$ by the data corresponding to event $\{Z_{BL} =1\}$. To proceed with the shuffling process, we begin by uniformly selecting an item $i_1 = i$ and then randomly choose a comparison of the form ``$i$ vs. $j$'' from the data $\mathcal{Z}$, for any $j \neq i$ such that $(i,j) \in \calE$. If the outcome of the comparison results in item $j$ being preferred over item $i$, we move to item $j$ (i.e., set $i_2 = j$) and continue this process from item $j$. Otherwise, we again repeat this step from item $i$. This iterative procedure continues until we revisit item $i$ after at least one departure, effectively completing a cycle (with $i_l = i$). 
        Next, we proceed to remove the data corresponding to the forward cycle $\{Z_{m^1,{i_1 i_{2}}}, \dots, Z_{m^{l-1}, {i_{l-1} i_{l}}} \}$ while adding new data points to $\mathcal{Z}$ corresponding to the backward cycle $\{Z_{m^1, {i_l i_{l-1}}}, \dots, Z_{m^{l-1}, {i_{2} i_{1}}} \}$. This step is illustrated with the following example. Suppose that a cycle of length $3$ is found in our dataset. Assume that in this cycle, item $j$ is preferred over item $i$ in a ``$i$ vs. $j$'' comparison, followed by item $k$ being preferred over item $j$ in the ``$j$ vs. $k$'' comparison, and finally, item $i$ triumphs over item $k$ in the ``$k$ vs. $i$'' comparison. According to Kolmogorov's loop criterion, the forward loop probability product: $p_{ij} \cdot p_{jk} \cdot p_{ki}$ should be the same as the backward loop probability product $p_{ik} \cdot p_{kj} \cdot p_{ji}$. This corresponds to replacing these three outcomes with the following comparisons: item  $k$ being preferred over item $i$ in a ``$i$ vs. $k$'' comparison, item $j$ preferred over item $k$ in a ``$k$ vs. $j$'' comparison, and item $i$ preferred over $j$ in a ``$j$ vs. $i$'' comparison. This entire process of finding a cycle and replacing the data is repeated for a sufficient number of iterations to ensure robust shuffling. Another example illustrating this data transformation process for a cycle of length $4$ is shown in \cref{fig: Data Transformation}.
\end{itemize}
    We repeat the two shuffling techniques sequentially, one after the other, and recalculate the test statistic at each iteration. By iteratively performing this type of shuffling a sufficient number of times, we construct a distribution of test statistics. To establish our empirical threshold, we extract the $95\%$ percentile value (or $0.95$ quantile) from this distribution. Notably, in the case of symmetric settings, shuffling for translated skew symmetry becomes redundant, and only shuffling for reversibility is necessary.

\begin{figure*}
    \centering
    \begin{tikzpicture}
        \tikzstyle{table} = [rectangle, draw, align=center, minimum width=4cm, minimum height=2cm]        
        \node[table] (before) at (0,0) {Data Before Transformation \\ 
            \vdots \\
             $j$ is preferred over $i$ in ``$i$ vs. $j$'' comparison \\
             $k$ is preferred over $j$ in ``$j$ vs. $k$'' comparison \\
             $l$ is preferred over $k$ in ``$k$ vs. $l$'' comparison \\
             $i$ is preferred over $l$ in ``$l$ vs. $i$'' comparison \\
            \vdots};
        \node[table, right=2.7cm of before] (after) {Data After Transformation \\ 
            \vdots \\
             $l$ is preferred over $i$ in ``$i$ vs. $l$'' comparison \\
             $k$ is preferred over $l$ in ``$l$ vs. $k$'' comparison \\
             $j$ is preferred over $k$ in ``$k$ vs. $j$'' comparison \\
             $i$ is preferred over $j$ in ``$j$ vs. $i$'' comparison \\
            \vdots};
        \draw[->, line width=1.5pt] (before) -- (after) node[midway, above] {Shuffling};
    \end{tikzpicture}
    \caption{Illustration of the data transformation process to induce reversibility for a cycle of length four from $i \to j \to k \to l \to i$. The (forward) transition probability corresponding to data in the left is proportional to $p_{ij}\cdot p_{jk}\cdot p_{kl} \cdot p_{li}$. The (backward)  transition probability corresponding to data in the right is proportional to $p_{il}\cdot p_{lk}\cdot p_{kj} \cdot p_{ji}$.  }
    \label{fig: Data Transformation}
\end{figure*}

\subsection{Synthetic Experiments}
In this section, we will perform several experiments on the synthetically generated datasets to verify our theoretical results and the effect of the shuffling techniques discussed above under hypotheses $H_0$ and $H_1$. For the first experiment, we will use the same construction for the pairwise comparison matrix $P$ that we utilized in the proof of \cref{thm:Lower Bound} under the null and alternate hypothesis, which are presented in \cref{Lower bound H0} and \cref{Lower bound H1}. 
We set the number of pairwise comparisons per pair of agents $k = 12$, the number of agents $n$ is linearly increased from $32$ to $128$ in $12$ equally spaced steps, and the perturbation $\eta$ in \eqref{pthetaconstruction} is increased from $0.16$ to $0.32$ in $12$ equally spaced steps. For each value of $\eta$ and $n$, simulations are performed by generating $Z_{ij} \sim \text{Bin}(k, p_{ij})$ for $i\neq j$ for both hypotheses $H_0$ and $H_1$, and the corresponding values of test statistic $T$ under hypothesis $H_1$ and \emph{empirical minimax risk $\hat{\mathcal{R}}_{\mathsf{m}}$} are computed. The threshold used for the decision rule is set to $\eta^2/n$. \cref{fig:1} plots the empirical average of $n\cdot T $ under hypothesis $H_1$ averaged over $250$ iterations. Note that for a fixed value of $\eta $, the values of $n \cdot T$ (under $H_1$) are roughly constant as $n $ increases, and the values of $n \cdot T$ increase as $\eta$ increases. \cref{fig:BayesRisk2} plots the behavior of $\hat{\mathcal{R}}_\mathsf{m}$ for each value of $\eta$ and $n$. Observe that $\I_{\hat{\mathcal{R}}_\mathsf{m}>1/2}$ is largely independent of $n$, which is consistent with our derivation in \cref{sec: Proof of Lower Bound and Stability}, where we have shown that if $8\eta^2 k \leq \tilde{c}$ then $\mathcal{R}_{\mathsf{m}} \geq \frac{1}{2}$. Thus, once $\eta$ exceeds a particular threshold $\mathcal{R}_{\mathsf{m}}$ exceeds $\frac{1}{2}$ for all values of $n$.
\begin{figure*}
\centering
\begin{subfigure}{0.48\textwidth}
  \centering
  \includegraphics[width=\linewidth]{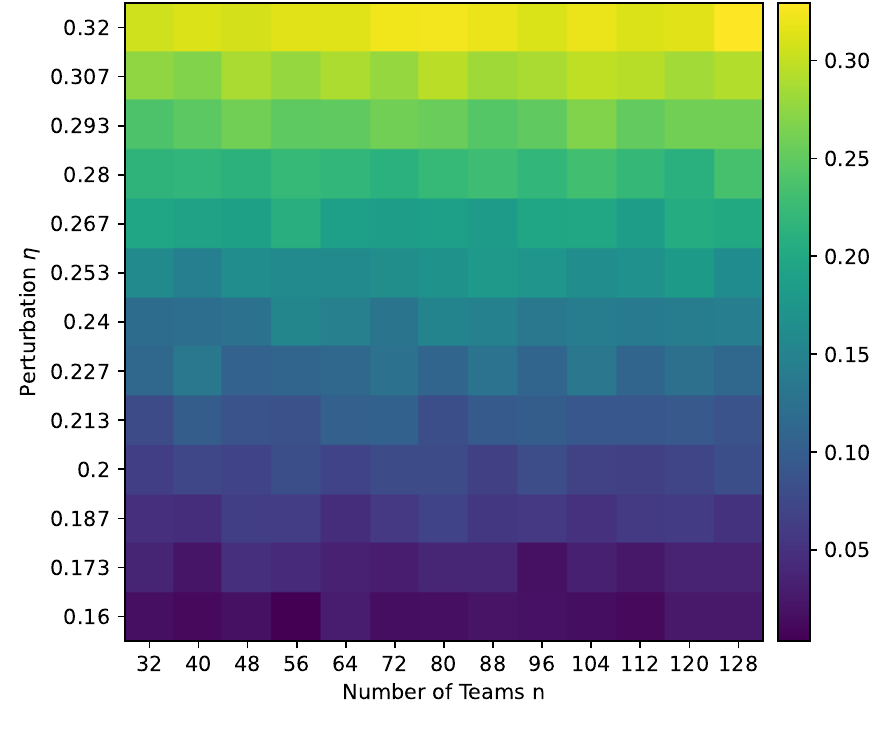} 
  \phantomcaption{(a) Empirical average of $n \cdot T$.}
  \label{fig:1} 
\end{subfigure}
\hfill
\begin{subfigure}{0.48\textwidth}
  \centering
  \includegraphics[width=\linewidth]{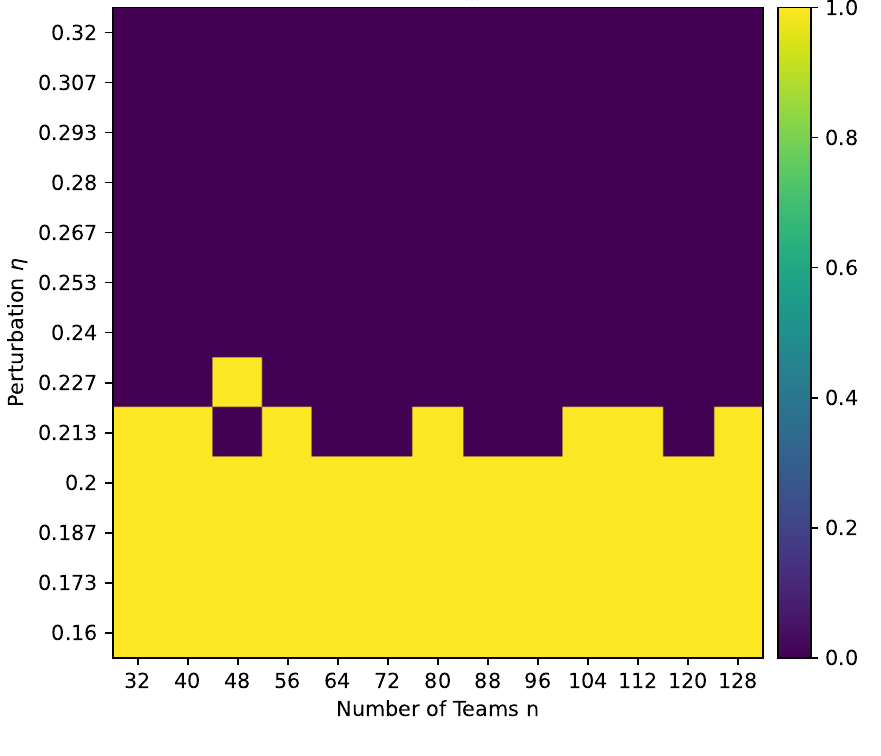} 
  \phantomcaption{(b) Empirical $\mathbbm{1}_{\hat{\mathcal{R}}_\mathsf{m}>1/2}$ }
  \label{fig:BayesRisk2} 
\end{subfigure}
\caption{ Plots 2a and 2b illustrate the empirical average of $n \cdot T$ under hypothesis $H_1$ and $\mathbbm{1}_{\hat{\mathcal{R}}_\mathsf{m}>1/2}$ for various values of $\eta$ and $n$.}
\label{fig:BayesRiskCombined}
\end{figure*}

Now we perform a second set of experiments where we analyze the thresholds derived using the empirical quantile approach and the permutation-based scheme under both hypotheses $H_0$ and $H_1$. We will derive these thresholds independently for various values of $n$ and $k$. Specifically, we consider $n$ values ranging from $10$ to $85$ with intervals of $15$, and $k$ values set to $[12, 24]$. 
For the empirical quantile approach, we begin by generating a collection of $200$ BTL models, where (unscaled) weights are randomly assigned according to $\alpha_i = 0.05 + U[0,1]$, with $U[0,1]$ representing a uniform random variable in $[0,1]$. Next, for each model, we generate synthetic data using the BTL model parameters and the corresponding pairwise comparison probabilities corresponding to a complete graph. The (scaled) test statistic $n\cdot k \cdot T$ is computed for each model, and the $95\%$ percentile value is identified as the threshold. We denote this estimated threshold using the empirical quantile approach as $\gamma_0$, computed separately for each combination of $n$ and $k$. Next, using only one of these models and its synthetic data, we compute an alternative threshold $\gamma_1$ via a permutation-based scheme. This involves two types of shuffling: first, shuffling to induce translated skew-symmetry, followed by $2n$ cyclic shuffles to induce reversibility. The combined shuffling procedure is repeated 200 times for each model, with the 95th percentile of the resulting $n \cdot k \cdot T$ values taken as $\gamma_1$. The entire experiment is then repeated for an observation graph generated from an Erd\H{o}s-R\'enyi model with edge probability $p$ chosen such that $np = \log^2 n$.

\cref{fig:ErdosH0_k} plots the obtained thresholds $\gamma_0$, $\gamma_1$ for various values of $n$ and $k$ for complete and Erd\H{o}s-R\'enyi random graphs. It can be seen that for complete graphs both the scaled thresholds remains roughly constant for various values of $n$ and $k$ thus exhibiting the same high level scaling behavior as the theoretical threshold and moreover, the thresholds, especially $\gamma_0, \gamma_1$, match each other quite closely even though they are obtained from quite different mechanisms. In contrast, the Erd\H{o}s-R\'enyi model shows a gradual decrease in thresholds as $n$ increases, which may reflect the conservative nature of our theoretical bounds.

Now, we will replicate the experimental framework from $H_0$  to both Erd\H{o}s-R\'enyi and complete graphs under the alternative hypothesis $H_1$.  The pairwise comparison model is defined as $P = \text{clip}(B + \Delta, 0.02, 0.98)$, where $B$ is the pairwise comparison matrix generated from a random BTL model, following the same procedure as in the earlier experiment. The clip operation ensures all probabilities remain within the $[0.02, 0.98]$ range to avoid negative values. The perturbation term \( \Delta \) introduces deviations, with \( \Delta_{ij} = 0.5/n^{1/4} \) for \( i > j \), \( \Delta_{ij} = -0.5/n^{1/4} \) for \( j > i \), and \( \Delta_{ij} = 0 \) when \( i = j \). To account for heterogeneity in the data, the number of observations $ k_{ij} = k_{ji}$ and $k_{ij}$ are sampled uniformly and independently from the interval $[12, 24]$. 

\cref{fig:ErdosH1_k} displays the distribution of scaled test statistics $n \cdot k \cdot T$ (with $k=12$) through shaded regions representing the $5\%-95\%$ percentiles, while solid line indicate mean values. For the Erd\H{o}s-R\'enyi graph model, we set the edge probability $np = \log^2 n$, as in the above experiment. The rest of the parameters and method for computing $\gamma_1$ are the same as in the previous experiment. Remarkably, the $\gamma_1$ thresholds computed under $H_1$ show strong agreement with those from $H_0$. This suggests that the combined effects of shuffling corresponding to reversibility and skew-symmetric induce ``BTL-like'' properties in the perturbed data. We emphasize that this permutation-based approach is primarily a conceptual demonstration rather than a scalable solution. While it provides an interesting empirical validation of our theoretical connections, its computational demands grow rapidly with problem size, and formal analysis of its complexity remains an open challenge.

\begin{figure*}
\centering
\begin{subfigure}[b]{0.49\linewidth}
    \centering
    \includegraphics[width=0.99\textwidth]{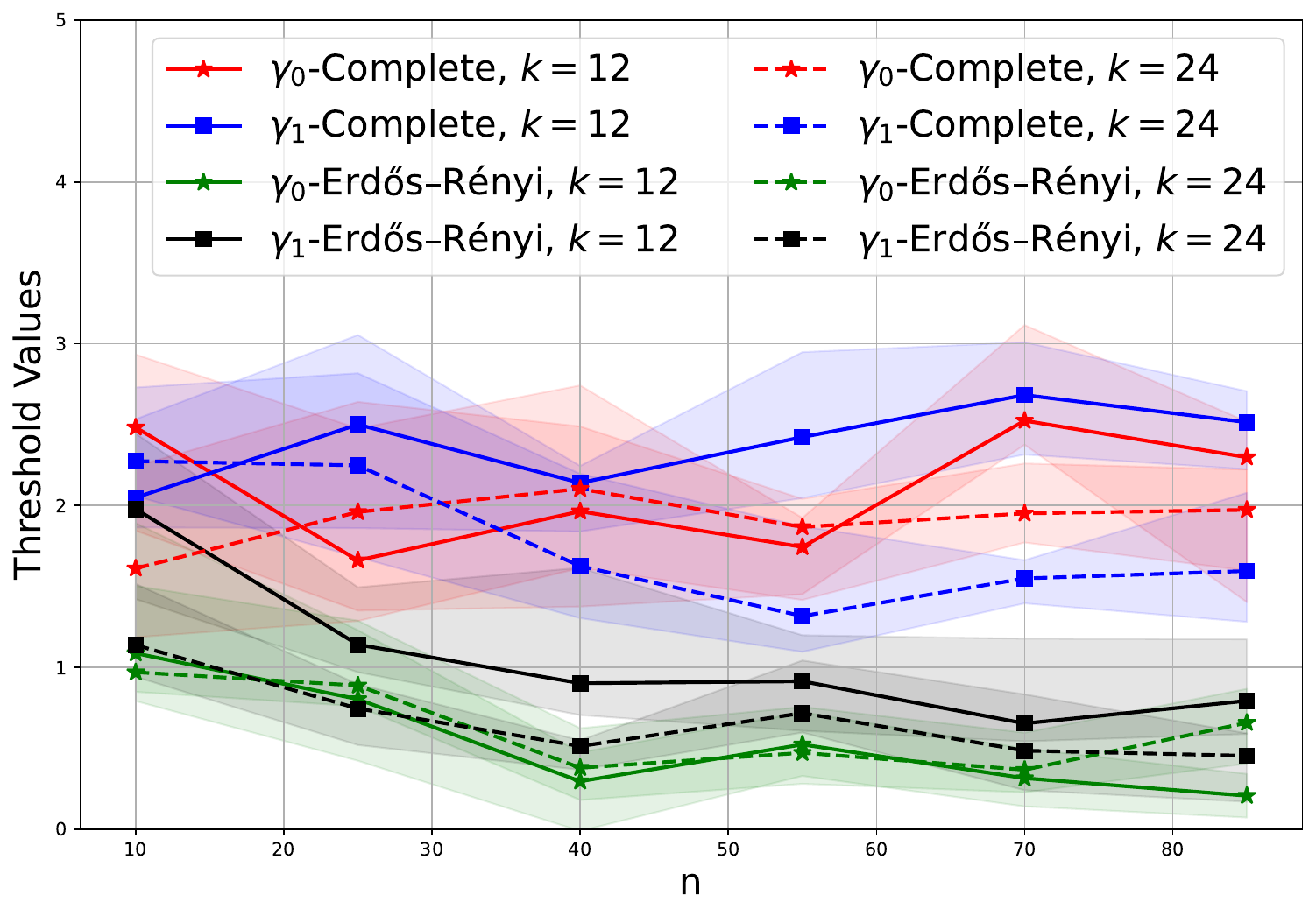}
    \phantomcaption{(a) Estimated thresholds $\gamma_0, \gamma_1$ under   $H_0$. } 
    \label{fig:ErdosH0_k}
\end{subfigure}
\begin{subfigure}[b]{0.47\linewidth}
    \centering
    \includegraphics[width=0.99\textwidth]{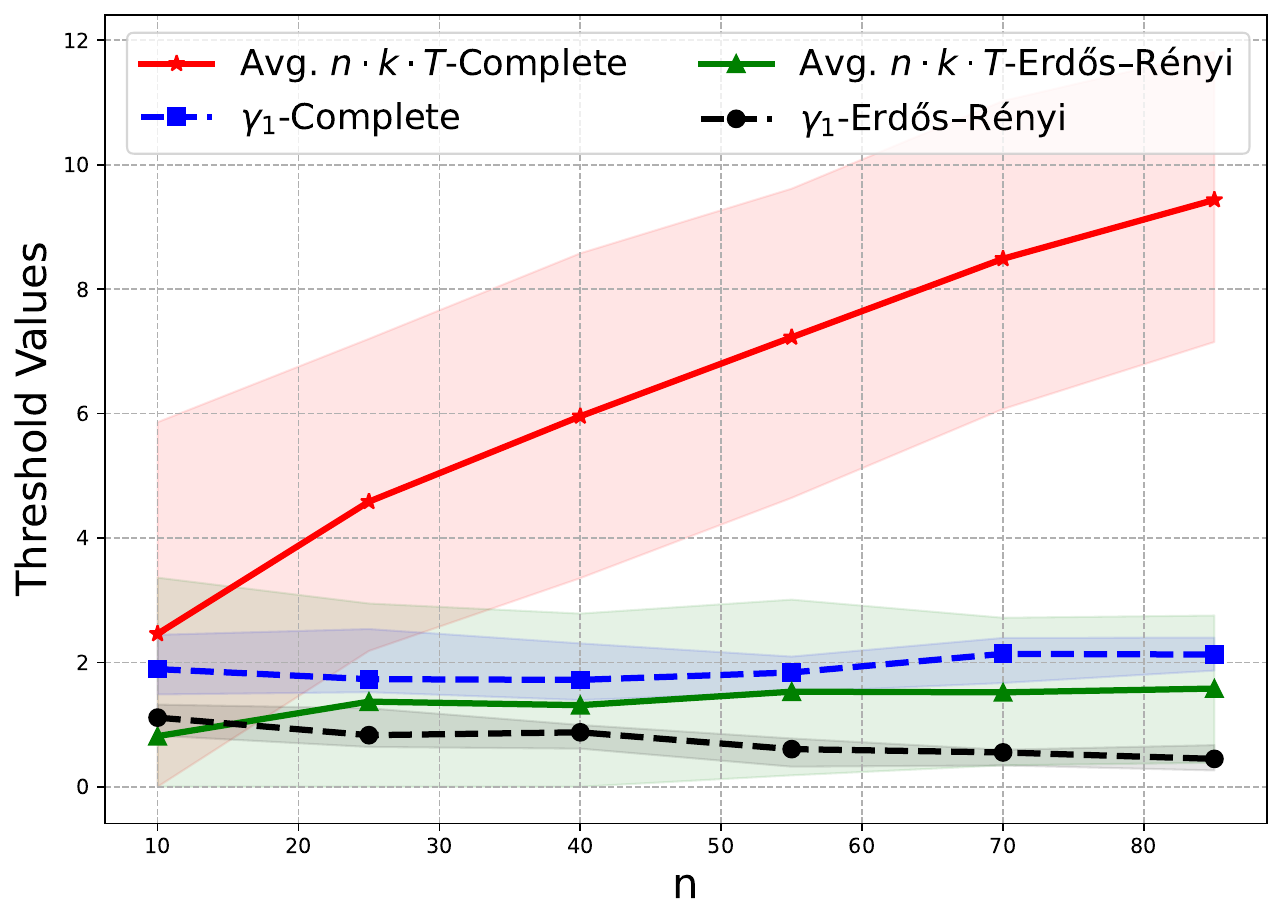}
    \phantomcaption{(b) Average value of $n\cdot k\cdot T$ and $\gamma_1$ under  $H_1$. } 
    \label{fig:ErdosH1_k}
\end{subfigure}
\caption{Plot 3a illustrates estimated thresholds $\gamma_0, \gamma_1$ for various values of $n$ and $k$ for complete and Erd\H{o}s-R\'enyi random graph. Plot 3b illustrates the behavior of estimated thresholds $\gamma_1$ for various values of $n$ and $k$ for complete and Erd\H{o}s-R\'enyi random graph under hypothesis $H_1$. The shaded region highlights 90\% confidence intervals of test statistic $T$.}
\end{figure*}

\subsection{Testing on Real-World Datasets} 

In this section, we apply the testing procedure to real-world datasets, even when these datasets may not be derived from the BTL model. We performed our testing on two distinct datasets representing different sports.

\textbf{Dataset 1:} Our first dataset encompasses public data gathered from cricket One Day International (ODI) matches \cite{CricketDataset} spanning a period of $19$ years, from $1999$ to $2017$. We analyzed matches between the eight most active teams in the dataset: Australia, India, England, West Indies, South Africa, New Zealand, Sri Lanka, and Bangladesh. These teams were selected based on their participation frequency in both home and away matches during the 1999-2017 period. To ensure the robustness of our analysis, we excluded matches resulting in a tie or draw. Furthermore, matches where neither of the teams played as the home team were also removed from consideration. Subsequently, we narrowed our focus to the eight teams that had engaged in the most number of matches against each other in both home and away scenarios. The testing procedure involves calculating the proposed test statistic and the thresholds using the empirical quantile approach and a permutation-based scheme. We calculate the test statistic and the thresholds on the cumulative data of $t$ recent years. \cref{fig:cricket} shows the value of $ n\cdot T$ (with $n=8$) computed on the cumulative data of $t$ most recent years, along with the respective (scaled) thresholds computed using the empirical-quantile approach and permutation-based scheme. As can be seen in the figure, the calculated test statistic consistently exceeded the threshold for cricket matches for most values of $t$. This indicates that the BTL model may not be the most suitable model for accurately characterizing the performance of cricket teams. The marked deviation can be attributed, in part, to a significant observed home advantage in the sport, which could be clearly observed by examining the empirical pairwise probability matrix.

\textbf{Dataset 2:} For our second dataset, we employed a similar process, this time focusing on National Basketball Association (NBA) matches \cite{NBADataset} from the years $2001$ to $2016$. The dataset comprised matches played by fifteen teams that had the most extensive history of total matches. As with the cricket dataset, we performed the same test, calculating the scaled test statistic $n \cdot T$ (with $n= 15$) and comparing it against the threshold computed using the empirical-quantile approach and permutation-based scheme on the cumulative data for $t$ most recent years. Our observations for the NBA dataset are presented in \cref{fig:NBA}, which illustrates that the BTL model has a much better fit, especially for smaller values of $t$. Also, note that the test statistic exceeds the threshold for larger values of $t$. This is reasonable because over such a large period, one could expect significant changes in the skill scores of the teams over time, and therefore, a single BTL model may not be able to accurately capture the data with a single time-independent skill score. Additionally, for both experiments, the thresholds obtained using the two different techniques agree with each other, demonstrating the consistency and effectiveness of both techniques. 

\begin{figure*}
\centering
\begin{subfigure}[b]{0.49\linewidth}
    \centering
    \includegraphics[width=0.99\textwidth]{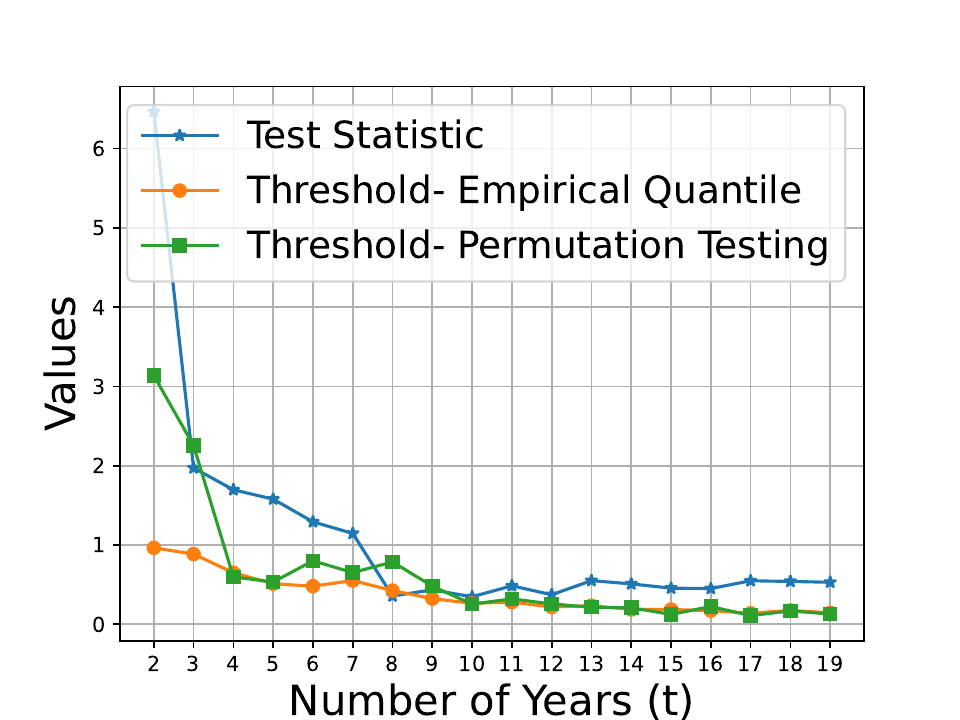}
    \phantomcaption{(a) Cricket ODI dataset. } 
    \label{fig:cricket}
\end{subfigure}
\begin{subfigure}[b]{0.49\linewidth}
    \centering
    \includegraphics[width=0.99\textwidth]{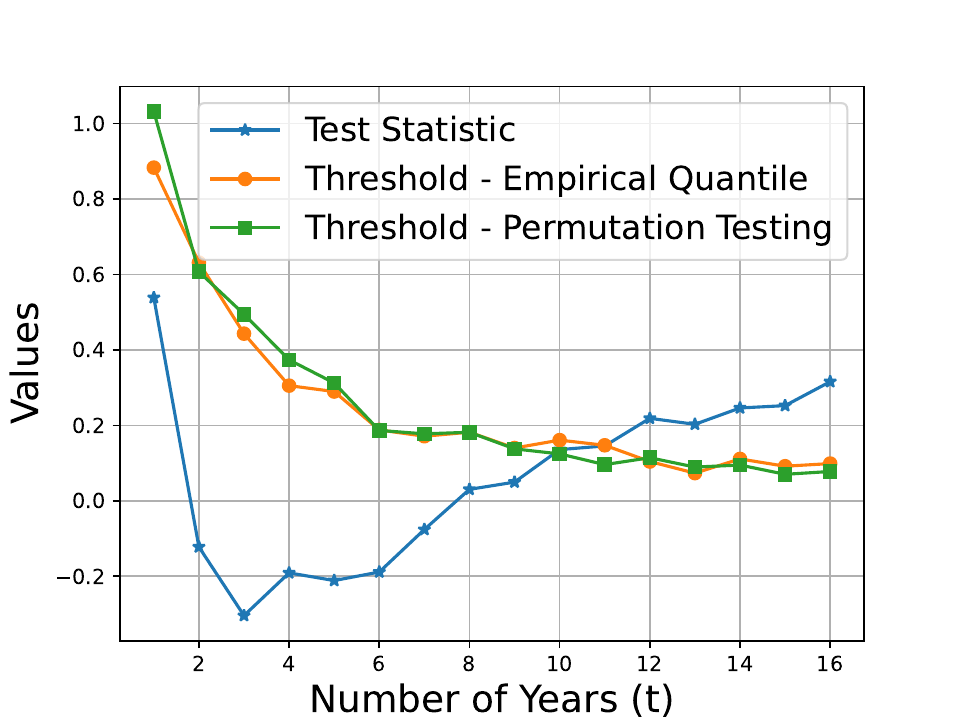}
    \phantomcaption{(b) NBA dataset. } 
    \label{fig:NBA}
\end{subfigure}
\caption{Plots 4a and 4b illustrate the scaled test statistic $n \cdot T$ for the cricket ODI dataset and the NBA dataset. The thresholds computed using both the empirical-quantile approach and the permutation-based scheme are also reported for each dataset.}
\end{figure*}

\section{Conclusion} \label{sec: Conclusion} 
In this work, we studied the problem of testing whether a BTL model accurately represents pairwise comparison data generated by an underlying pairwise comparison model. We introduced a new notion of separation distance to quantify the deviation of a pairwise comparison model from the set of BTL models. This distance measure allowed us to rigorously formulate our minimax hypothesis testing framework. We derived upper bounds on the critical threshold for our hypothesis testing problem under certain fixed induced graph structures and established a matching information-theoretic lower bound for complete induced graphs. 
Furthermore, we established upper bounds on the type \Romannum{1} and type \Romannum{2} probabilities of error for our proposed statistical test. Our work also highlighted the importance of expansion properties and bounded principal ratios for the induced observation graphs in our framework; indeed, our proposed test exhibited theoretical guarantees under such assumptions. In particular, we provided several examples of families of graphs that possessed the desired expansion and regularity properties. Finally, we conducted several experiments on synthetic and real-world datasets to validate our theoretical results. To conduct some of our numerical simulations, we also presented a new non-parametric approach based on the permutation test that constructed a data-driven threshold for our proposed hypothesis test. 

There are several possible future research directions springboarding off of this work. For example, one could extend our hypothesis testing framework to account for general $k$-ary comparison data rather than pairwise comparison data. 
It would also be useful to derive bounds on the principal ratio for more general induced graph structures, even when the graphs depend on the underlying pairwise comparison probabilities. Another interesting direction is to investigate whether bounds on expansion and principal ratio are fundamentally necessary for spectral analysis of the testing problem. Our subsequent work \cite{MakurSingh2025b} presents a maximum-likelihood-based analysis for testing generalized Thurstone models, and circumvents the need for expansion and principal ratio assumptions. Additionally, investigating the impact of other pairwise comparison models, such as those satisfying strong and weak stochastic transitivity, on the critical threshold of the testing problem presents an interesting avenue for future research. Furthermore, extending the minimax lower bounds to more general graph topologies satisfying certain structural conditions presents another important avenue for future work. Lastly, one could delve deeper into the theoretical properties of the new permutation test (based on Kolmogorov's loop criterion) proposed in our work as well.

\appendices

\section{Proofs of Bounded Principal Ratio}
\subsection{Proof of \cref{lem: Height of Perron–Frobenius Eigenvector}}\label[appendix]{Proof of lem: Height of Perron–Frobenius Eigenvector}
        By \cref{Pijbounded} and the Perron-Frobenius theorem, we know that $\pi_i > 0$ for all $i \in [n]$.
        By stationarity of $\pi$, we have 
        $$
            \pi_i = \sum_{k=1}^n \pi_k S_{ki} = \pi_i \left(1 - \frac{1}{d}\sum_{k: k\neq i} p_{ik}\right) + \frac{1}{d } \sum_{k:k\neq i} p_{ki} \pi_k .
        $$
        Rearranging the above equation gives
        $$
            \pi_i = \frac{\sum_{k:k\neq i} p_{ki} \pi_k}{\sum_{k: k\neq i} p_{ik}} .
        $$
        Hence, for any $i \neq j$ such that $\pi_i \leq \pi_j$, we have
        \begin{align*}
            \frac{\pi_i}{\pi_j} &= \frac{\sum_{k:k\neq i} p_{ki} \pi_k}{\sum_{k:k\neq j} p_{kj} \pi_k} \frac{\sum_{k: k\neq j} p_{jk}}{\sum_{k: k\neq i} p_{ik}} \\
            & \geq  \delta \frac{\sum_{k:k\neq i}  \pi_k}{\sum_{k:k\neq j} \pi_k} \frac{\sum_{k: k\neq j} p_{jk}}{\sum_{k: k\neq i} p_{ik}} \\
            & \geq \delta^2 \frac{ 1-\pi_i}{1 - \pi_j} \geq \delta^2 ,
        \end{align*}
        where last inequality holds because $\pi_i \leq \pi_j$ and the other inequalities use \cref{Pijbounded}. The proposition follows by taking  $j = \argmax_k \pi_k$. \qed

    \subsection{Proof of \cref{lem: Height of Perron–Frobenius Eigenvector for expanders}} \label[appendix]{Proof of lem: Height of Perron–Frobenius Eigenvector for expanders}
    We follow the high-level strategy of the argument in \cite[Theorem 3]{AksoyChung2016}. Let the smallest and the largest elements of the stationary distribution be denoted as:
    $$\pi_{s} \triangleq \min_{i \in [n]} \pi_{i} \text{ and } \pi_{\ell} \triangleq \max_{i \in [n]} \pi_{i}.$$
    Let $N(i)$ denotes the neighborhood set of $i$, i.e., $N(i)=\{ j: (i, j) \in \calE\}$. Let $\mathcal{S}$ denote the set $\{v \in N(\ell): \pi_{v} \leq \delta^{2} \pi_{\ell}\}$. By the detailed balance equations, we have
    $$
        \left(\sum_{j \in N(\ell)} p_{\ell j}\right) \pi_{\ell} = \sum_{j \in N(\ell)} \pi_{j} p_{j \ell} =\sum_{ j \in \calS } \pi_{j} p_{j u}+\!\sum_{j \in N(\ell) \setminus \calS} \pi_{j} p_{j \ell} .
    $$
    This gives
    $$
    \begin{aligned}
    \left(\sum_{j \in N(\ell)} p_{\ell j}\right) \pi_{\ell} &\leq  \sum_{ j \in \calS } \delta^{2} \pi_{\ell} p_{j \ell}+\sum_{j \in N(\ell) \setminus \calS}   \pi_{\ell} p_{j \ell} .
    \end{aligned}
    $$
    Utilizing \Cref{Pijbounded}, we obtain
    $$
    \delta \tilde{d} \leq \delta^{2}|\calS|+(\tilde{d}- |\calS|).
    $$
    The above inequality implies $|\calS| \leq \frac{\tilde{d}(1-\delta)}{\left(1-\delta^{2}\right)}$. Define $\mathcal{L}=N(\ell) \setminus \calS$, and therefore, we have
    $$
    \begin{aligned}
    |\mathcal{L}| \geq \tilde{d}-\frac{\tilde{d}(1-\delta)}{1-\delta^{2}}
    & =\tilde{d}\left(1-\frac{1-\delta}{1-\delta^{2}}\right)\\
    & =\tilde{d}\left(\frac{\delta-\delta^{2}}{1-\delta^{2}}\right)=\frac{\delta \tilde{d}}{1+\delta}.
    \end{aligned}
    $$
    Now we consider two cases.
    
    \textbf{Case 1:} Assume $|N(s) \cap \mathcal{L}| \geq  \delta^2 \tilde{d} / (1+\delta)$. In this case, using the detailed balance equations, we obtain 
    $$
    \begin{aligned}
    \left(\sum_{j \in N(s)} p_{s j}\right) \pi_{s}= & \sum_{j \in N(s)} \pi_{j} p_{j s} \geq \sum_{j \in N(s) \cap \mathcal{L}} \pi_{j} p_{j s} \\ 
    & \geq \delta^{2} \pi_{l} \times \frac{\delta}{1+\delta} \times|N(s) \cap  L| \\
    &  \geq \frac{\delta^{5} \pi_{l} \tilde{d}}{(1+\delta)^2}.
    \end{aligned}
    $$
    The left-hand side is upper-bounded by
    $$
    \sum_{j \in N(s)} p_{s j} \pi_{s} \leq \frac{\tilde{d}}{1+\delta} \pi_{s}.
    $$
    Thus, the two equations together give
    $$
    \pi_{s} \geq \frac{ \delta^{5} \pi_{\ell}}{ 1 + \delta}.
    $$
    Thus, we have a bound on the principal ratio as $\frac{1+ \delta}{ \delta^{5}}$.
    
    \textbf{Case 2:} Now we assume $|N(s) \cap \mathcal{L}| \leq \frac{ \delta^2  \tilde{d}}{(1+\delta)}$. Define the set $\mathcal{L}^{\prime}=\mathcal{L} \backslash N(s)$. Then, in this case, we have $\left|\mathcal{L}^{\prime}\right| \geq \frac{(1- \delta)\delta \tilde{d}}{(1+\delta)}$. Now, defining $E_{1} =  \mathcal{E} (\mathcal{L}^{\prime}, N(s) )$ utilizing the property $|\mathcal{E}(\mathcal{T}_1, \mathcal{T}_2)| \geq a|\mathcal{T}_1||\mathcal{T}_2|$ we obtain
    $$
    \begin{aligned}
     |E_{1}| =  |\mathcal{E} (\mathcal{L}^{\prime}, N(s) ) | \geq a \left(\frac{(1-\delta)\delta \tilde{d}}{(1+\delta)}\right) \tilde{d} .
    \end{aligned}
    $$
    Using the detailed balance equations for $\pi_s$, we have
    $$\left(\sum_{j \in N(s)} p_{s j}\right) \pi_{s}=\sum_{j \in N(s)} \pi_{j} p_{j s}. $$
    This gives
    $$
    \begin{aligned}
    \pi_{s} & \geq \frac{\delta}{\tilde{d}} \sum_{j \in N(s)} \pi_{j}  =\frac{\delta}{\tilde{d}} \sum_{j \in N(s)} \sum_{k \in N(j)} \frac{ \pi_{k} p_{ kj }}{\sum_{k \in N(j)} p_{jk}} \\
    &  \geq \frac{\delta}{\tilde{d}} \sum_{(j, k) \in E_1 }   \delta^{2} \pi_{\ell} \times \frac{\delta}{\tilde{d}} = \frac{\delta^4 \pi_{\ell}}{\tilde{d}^2}  |E_1| \geq  a\pi_\ell \delta^5 \left(\frac{1 - \delta}{1+\delta}\right) .
    \end{aligned}
    $$
    Thus, we get a bounded principal ratio of $(1+\delta)/(a(1-\delta)\delta^5)$. \qed     

    \subsection{Additional Details on Properties of $(n, \tilde{d},\lambda_2(\G))$-Expander Graphs}
    \label[appendix]{Expansion}
    This appendix provides supporting details on the expansion properties of $(n, \tilde{d},\lambda_2(\G))$-expanders, including existence results and verification of \cref{assm:Large Expansion of DTM} and the assumptions in \cref{lem: Height of Perron–Frobenius Eigenvector for expanders}.

    For simplicity, we present a single construction of a $(n, \tilde{d} = n/2, \lambda_2(\mathcal{G}))$-expander with $\lambda_2(\mathcal{G}) \leq 0.7 \tilde{d}$. Nonetheless, multiple constructions are expected to exist, analogous to ``optimal'' expanders like Ramanujan graphs, which achieve $\lambda_2(\mathcal{G}) = 2\sqrt{\tilde{d}-1}$. We construct a family of $(n, \tilde{d} = n/2, \lambda_2(\mathcal{G}))$-expanders, where $n = 4m$, $m \in \mathbb{N}$. Let $\mathcal{G}$ be a graph on $n$ vertices, where each vertex $i$ is connected to $m = n/4 $ neighbors on both the left and right sides in a cyclic manner. The adjacency matrix of $\mathcal{G}$ is a circulant matrix with the first row given by:
    $$ (0, \underbrace{1, 1, \ldots, 1}_{m}, \underbrace{0, 0, \ldots, 0}_{2m-1}, \underbrace{1, 1, \ldots, 1}_{m}). $$ 
        The graph $\mathcal{G}$ is $n/2$-regular. The second largest absolute eigenvalue $\lambda_2(\mathcal{G})$ for this graph is known to be (cf. \cite{LeeLuoSaganYeh1992})
    \begin{equation}
    \lambda_2(\mathcal{G}) = \max_{l \in [n-1] } \sum_{j=1}^{m} 2 \cos\left(\frac{2\pi j l}{n} \right).
    \end{equation}
 A trigonometric calculation shows that $\lambda_2(\G) \leq 0.7\tilde{d}$. Thus, $\mathcal{G}$ is an $(n, \tilde{d} = n/2, \lambda_2(\mathcal{G}) \leq 0.7 \tilde{d})$-expander graph.

To show that a $(n, \tilde{d},\lambda_2(\G))$ expander with   $\tilde{d} = \nu n$ and  $\lambda_2(\G) \leq (1- \tilde{\nu}) \tilde{d}$ implies the assumption in \cref{lem: Height of Perron–Frobenius Eigenvector for expanders}, we utilize the expander mixing lemma in \cite[Theorem 1]{abiad2024unified}, which by a simple calculation yields 
\begin{equation}
|\calE(\mathcal{S}, \mathcal{T})| \geq \frac{ (\tilde{d} - \lambda_2(\G))}{n} |\mathcal{S}| |\mathcal{T}|. \label{edge expansion lb}
\end{equation}
This yields the bound $a \geq (1-\nu)\tilde{\nu}$ on $a$ in \cref{lem: Height of Perron–Frobenius Eigenvector for expanders}.
Now to show \cref{assm:Large Expansion of DTM} holds for this expander, we substitute $\mathcal{T} = \mathcal{S}^\complement$ in \cref{edge expansion lb} and which by a simple argument gives 
\begin{equation}
\tilde{\phi}(\G) \geq  \frac{ (\tilde{d} - \lambda_2(\G))}{2}   \geq  \frac{\tilde{\nu} \tilde{d}}{2}.
\end{equation}
Using \cref{trivial lower bound of expansion}, this gives $\xi \geq \frac{\delta^6  }{ \tilde{c}(a) (1+\delta) } \tilde{\nu} $, where $\tilde{c}(a)$ is the constant in \cref{lem: Height of Perron–Frobenius Eigenvector for expanders}.

\subsection{Proof of \cref{lem: Height of Perron–Frobenius Eigenvector under sampling}} \label[appendix]{Proof of Height of Perron–Frobenius Eigenvector under sampling}
The proof immediately follows by utilizing \cref{thm: Perturbation Bound on Sparsified Stationary Distribution} and an application of triangle inequality, giving us the following bounds on $\|\pi\|_\infty$ and $\pi_{\min}$: 
\begin{align}
   \|\pi\|_\infty & \leq   \left(1 + \sqrt{\frac{c_{10}}{\mathsf{c}_p} } \right) \|\pi^*\|_\infty , \label{Bound 1infty} \\
   \pi_{\min} & \geq \pi^*_{\min} - \sqrt{\frac{c_{10} }{\mathsf{c}_p} }  \|\pi^*\|_\infty \stackrel{\zeta}{\geq} \pi^*_{\min} - \sqrt{\frac{c_{10} }{\mathsf{c}_p} }  \frac{\pi^*_{\min}}{\delta^2}, \label{Bound 2infty}
\end{align}
where we define $\pi^*_{\min} \triangleq \min_{i \in [n]} \pi^*_i $, and $\zeta$ follows since the principal ratio for $\pi^{*}$ is upper bounded by $1/\delta^2$ by \cref{lem: Height of Perron–Frobenius Eigenvector}.  
Therefore, we can upper bound the principal ratio for a large enough constant $\mathsf{c}_p$ as
\begin{equation}
     \frac{\|\pi\|_{\infty}}{\pi_{\min}} \leq \frac{(1 + \sqrt{c_{10}/\mathsf{c}_p} ) \|\pi^*\|_\infty }{ (1 - \sqrt{c_{10}/(\delta^4 \mathsf{c}_p) } ) \pi^*_{\min} } \leq \frac{(1 + \sqrt{c_{10}/\mathsf{c}_p} )  }{ (1 - \sqrt{c_{10}/(\delta^4 \mathsf{c}_p) }) \delta^2  }  .
\end{equation} 
Hence,  in the above condition $n p \geq \frac{2c_{10} \log n}{\delta^4}$ suffices to obtain a bounded principal ratio for $\pi$. \qed

\section{$\ell^\infty$-Error Bounds under Sub-sampling}
\label[appendix]{Proof of Perturbation Bound on Sparsified Stationary Distribution}
In this section, we will provide a proof for \cref{thm: Perturbation Bound on Sparsified Stationary Distribution}. For brevity, we will refer to the scenarios in \cref{thm: Perturbation Bound on Sparsified Stationary Distribution} as the \textit{sub-sampled} case. 
Recall that, in the sub-sampled case, $S^*$ and $S$ are defined as 
\[  
\begin{aligned}
    S^*_{ij} & \triangleq  \begin{cases}
        \frac{p^*_{ij} }{3n}, & i\neq j \\
        1-\frac{1}{3n}\sum\limits_{u=1}^n p^*_{iu}, & i = j 
    \end{cases},        \\  
    S_{ij}    &  \triangleq  \begin{cases}
        \frac{p^*_{ij} \I_{(i,j) \in \calE} }{d}, & i\neq j \\
        1 - \frac{1}{d}\sum\limits_{ u=1}^n p^*_{iu} \I_{(i,u) \in \calE }, & i = j 
    \end{cases}.
\end{aligned}
\]
Moreover, in the sub-sampled case, we use $d= 3np$ and we use the notation $d_{\max} = \frac{3np}{2}$ and $d_{\min} = \frac{np}{2}$ to denote the maximum and minimum degrees in the high probability sense (see \cref{lem: Degree Concentration}). Observe that, on the event $\mathcal{A}_0$ defined in \cref{lem: Degree Concentration}, $S$ is a valid row stochastic matrix, and therefore, $\pi$ exists. We remark that our argument in the proof below uses the leave-one-out technique from \cite{Chenetal2019}.

\subsection{Proof of \cref{thm: Perturbation Bound on Sparsified Stationary Distribution}}

In the following section, we will provide a unified proof of \Cref{thm: Perturbation Bound on Sparsified Stationary Distribution}. 
Since ${\pi}$ and $\pi^*$ are stationary distribution of ${S}$ and $S^*$, we can write the $i$th entry as
$$
\begin{aligned}
{\pi}_{i}  -\pi_{i}^{* } & =({\pi}^{\T} {S})_{i}-(\pi^{*^\T} S^*)_{i} \\
&  = (\pi^{*^\T}({S}-S^*))_{i}+(({\pi}-\pi^*)^{\T} {S})_{i} \\
& = \underbrace{\sum_{j=1}^{n} \pi^*_j({S}_{ j i}-S^*_{j i}) }_{L_1}+ \underbrace{\left({\pi}_{i}-\pi^*_{i}\right) {S}_{i i}}_{L_2} \\
& \ \ \ \ \ \ \ \ \ \ \ \  + \underbrace{\sum_{j: j \neq i}\left({\pi}_{ j}-\pi^*_{j}\right) {S}_{ j i }}_{L_3}.
\end{aligned}
$$
In the following two lemmata, we will show that in sub-sampled models the $L_1$ term is bounded as $O\big(\|\pi^*\|_{\infty} \sqrt{\frac{\log n}{n p}}\big) $, with high probability. The precise statement is provided below.
\begin{lemma}[$L_1$ in Sub-sampled Model] \label{lemS: Handling the first term} For the setting of \cref{thm: Perturbation Bound on Sparsified Stationary Distribution}, assume that the sampling probability satisfies $np \geq \log n $, then there exists a constant $c$ such that
\begin{align*}
    \mathbb{P}& \left(\exists i \in [n], \sum_{j=1}^{n} \pi^*_j({S}_{ j i}-S^*_{j i}) \geq  c\|\pi^*\|_{\infty} \sqrt{\frac{\log n}{np }} \right) \\ 
    & \ \ \ \ \ \ \ \ \ \ \ \ \ \ \ \ \ \ \ \ \ \ \ \ \ \ \ \ \ \ \ \ \ \ \ \ \ \ \ \ \ \ \ \ \ \ \ \ \ \ \ \  \leq  O\bigg(\frac{1}{n^3}\bigg).
\end{align*}
\end{lemma} 
The proof of \cref{lemS: Handling the first term} is provided in \cref{Proof of lemS: Handling the first term}. 
Now we focus on the second term $L_2 = \left({\pi}_{i}-\pi^*_{i}\right) {S}_{i i}$. Observe that,  
$$
\begin{aligned}
{S}_{i i} & =1-\frac{1}{d} \sum_{\substack{j: j \neq i \\ (i, j) \in \calE}} p_{i j} \leq  1-\left(\frac{\delta}{1+\delta}\right) \frac{d_{\min} }{ d} .
\end{aligned}
$$
By \cref{lem: Degree Concentration}, we know that $d_{ \min} = \frac{n p}{2}$. Therefore, on the event $\mathcal{A}_0$, we have
\begin{equation} \label{Bound on L2s}
    |({\pi}_{i}-\pi^*_{i}) S_{i i}| \leq \left(1-\frac{\delta}{6(1+\delta)} \right) \left|{\pi}_{i}-\pi^*_{i}\right| .
\end{equation}
Next, we aim to establish an upper bound for the third term $L_3$. However, due to the interdependence between ${\pi}$ and ${S}$, finding a tight bound on this term becomes challenging. To address this, we employ a leave-one-out strategy as in \cite{Chenetal2019} to disentangle the dependence so as to obtain tighter bounds. Therefore, we introduce a new canonical Markov matrix ${S}^{(m)}$ for $m \in [n]$ with its $(i,j)$th entry for $i \neq j$ defined as 
$$({S}^{(m)} )_{i j} \triangleq \begin{cases} {S}_{i j},&  i \neq m  \text { and } j \neq m \\ \frac{p^*_{i j}}{3 n}, & i=m  \text { or } j=m  \end{cases}, $$
i.e., $(S^{(m)})_{ij}$ is the same as $(S^*)_{ij}$ when $i=m$ or $j=m$. Also, we define $({S}^{(m)})_{i i} = 1-\sum_{j: j \neq i} (S^{(m)})_{i j}$. Note that on the event $\mathcal{A}_0$, ${S}^{(m)}$ is also a row-stochastic matrix. Let ${\pi}^{(m)}$ be the stationary distribution of ${S}^{(m)}$. Now, we decompose the term $L_3$ as
$$
\begin{aligned}
    \sum_{j: j \neq m}\left({\pi}_{j }-\pi^*_{j}\right) {S}_{jm} & = \underbrace{\sum_{j: j \neq m} \left({\pi}_{j}-{\pi}^{(m)}_{j}\right) {S}_{jm}}_{L_{3,1}} \\ 
    &\ \ \ \ \ \ \  + \underbrace{\sum_{j: j \neq m}\left({\pi}^{(m)}_{j}-\pi^*_{j}\right) {S}_{ jm}}_{L_{3,2}}.   
\end{aligned}
$$

\textbf{Bounding $L_{3,1}$:} First, we will bound the term $L_{3,1}$ using the Cauchy-Schwarz inequality as
$$
\begin{aligned}
\sum_{j : j \neq i}\left({\pi}_{j} -{\pi}^{(m)}_{j}\right) {S}_{j m} & \leq \frac{1}{d} \sqrt{d_{\max} }\left\|{\pi} - {\pi}^{(m)}\right\|_{2} \\
&  = \frac{1}{\sqrt{6np } } \left\|{\pi}- {\pi}^{(m)}\right\|_{2}.
\end{aligned}
$$
The following two lemmata provide a bound on the $\ell^2$-error for the leave-one-out version of the stationary distribution.
\begin{lemma}[Error Bound for Sub-sampled Case] \label{lem: Error bound on leave-out} There exists constant $c_0>1 $ such that for $n p \geq  c_0 \log n$, the following bound holds with probability at least $1 - O(n^{-3})$, 
\begin{equation}
    \|{\pi}^{(m)}  - {\pi}  \|_2 \leq \|{\pi} - \pi^*\|_\infty + \|\pi^*\|_\infty.
\end{equation}
\end{lemma}
The proof of \Cref{lem: Error bound on leave-out} is provided in \cref{Proof of Error bound on leave-out}. 

\textbf{Bounding $L_{3,2}$:}  To bound the term $L_{3,2}$, we define the graph $\G_{-m}$ as the graph $\G$ without the $m$th node. We will use the following identity to bound $|L_3|$: 
$$
    |L_{3,2}| \leq |L_{3,2} - \E\left[L_{3,2} \mid \G_{-m}\right]| + |\E\left[L_{3,2} \mid \G_{-m}\right]|.
$$
Observe that $\E[{S}^{(m)}_{jm}|\G_{-m}] = S^*_{ij} $. Therefore, we get the following bound on $|L_{3,2}|$
\begin{align}
        |\E\left[L_{3,2} \mid \G_{-m}\right]| & =\sum_{j: j \neq m}\left({\pi}^{(m)}_{j}-\pi^*_{j}\right) S^*_{ij} \nonumber \\ 
        & \leq \left\| {\pi}^{(m)}-\pi^*\right\|_{2} \frac{\sqrt{n}}{3n} \nonumber\\
        & \leq \frac{1}{3\sqrt{n}}\left(\left\|{\pi}^{(m)}-{\pi}\right\|_{2}+\left\|{\pi}-\pi^*\right\|_{2}\right) .\label{Bound on EL32s} 
\end{align}
The bounds for the first term $\|{\pi}^{(m)}-{\pi}\|_2$ have been derived in \cref{lem: Error bound on leave-out} and the error bounds for second term $\|{\pi}-\pi^*\|_{2},$ are provided below in the ensuing lemma.
\begin{lemma}[$\ell^2$-Error Bounds for Sub-sampled Distribution] \label{lem: L2 error subsampling} There exists constants $\tilde{c}, c_0>1 $ such that for $n p \geq  c_0 \log n$, the following bound holds with probability at least $1 - O(n^{-3})$, 
\begin{equation}
    \|{\pi } - \pi^* \|_2 \leq \tilde{c} \sqrt{n}\|\pi^*\|_\infty \sqrt{\frac{\log n}{n p }} .
\end{equation}
\end{lemma}
The proof is provided in \cref{Proof of L2 error subsampling}. Now, it remains to bound the term  $|L_{3,2}-\E[L_{3,2}|\G_{-m}]|$. To bound this term, we employ the Bernstein inequality \cite{HDP,tropp2012} as follows:
\begin{equation}
\begin{aligned}
    & \P(|L_{3,2} - \E[L_{3,2}|\G_{-m}]| \geq t) \\
    & = \P\bigg(\bigg| \sum_{j: j \neq m} ({\pi}^{(m)}_{j}-\pi^*_j) ({S}_{mj} - S^*_{mj}) \bigg|\geq t\bigg)  \\
    & = \P\bigg(\bigg| \sum_{j: j \neq m} ({\pi}^{(m)}_{j}-\pi^*_j) \frac{p^*_{mj}}{d} ( \I_{(m,j) \in \calE} - p) \bigg|\geq t\bigg) \\
    &   \leq 2 \exp\left(-\frac{t^2/2 }{ \sigma^2 + \frac{r t}{3} }\right),
\end{aligned}
\end{equation}
where $\sigma^2 =  \frac{(1-p)\|{\pi}^{(m)} - \pi^*\|^2_\infty}{9np}$ and $r = \frac{(1-p)\|{\pi}^{(m)} - \pi^*\|_\infty}{3np}$. Substituting $t = \frac{5}{3} \|{\pi}^{(m)} - \pi^*\|_\infty \sqrt{\frac{\log n}{n p }} $, we obtain the following bound with probability at least $1 - O(n^{-4})$:
$$|L_{3,2} - \E[L_{3,2}|\G_{-m}]| \leq \frac{5}{3} \|{\pi}^{(m)} - \pi^*\|_\infty \sqrt{\frac{\log n}{n p }}. $$ 
Therefore, combining the above bound along with \cref{lem: Error bound on leave-out} and \cref{Bound on EL32s}, we obtain the following bound on $L_{3,2}$ in the sub-sampled case:
\begin{align}
    & |L_{3,2}|  \nonumber \leq    \frac{5}{3}\|{\pi}^{(m)} - \pi^*\|_\infty \sqrt{\frac{\log n}{n p }} \\
    & \nonumber \  \ \ \ \ \ \ \ \ \ +   \frac{1}{3\sqrt{n}}\left(\left\|{\pi}^{(m)}-{\pi}\right\|_{2}+\left\|{\pi}-\pi^*\right\|_{2}\right) \\
    &\nonumber \leq \frac{5}{3}( \|{\pi}^{(m)} - {\pi}\|_2 + \|{\pi} - \pi^*\|_\infty ) \sqrt{\frac{\log n}{n p }}  \\
    &\  + \frac{1}{3\sqrt{n}}\left( \|{\pi} - \pi^*\|_\infty + \|\pi^*\|_\infty  +\tilde{c} \sqrt{n}\|\pi^*\|_\infty \sqrt{\frac{\log n}{n p }} \right) \nonumber \\
    & \leq \frac{5}{3}(2 \|{\pi} - \pi^*\|_\infty + \|\pi^*\|_\infty   ) \sqrt{\frac{\log n}{n p }} \nonumber \\ 
    &\ + \frac{1}{3\sqrt{n}}\left( \|{\pi} - \pi^*\|_\infty + \|\pi^*\|_\infty  +\tilde{c} \sqrt{n}\|\pi^*\|_\infty \sqrt{\frac{\log n}{n p }} \right)  \label{combined bounds on L32}.
\end{align}

Finally, combining the bounds on $L_1$ (see \cref{lemS: Handling the first term}), $L_2$ (see \cref{Bound on L2s}), $L_{3}$ (see \cref{combined bounds on L32}) together, we get the following bound simultaneously for all $m \in [n]$:
\begin{align}
& |{\pi}_i - \pi^*_i |  \leq c \|\pi^*\|_{\infty} \sqrt{\frac{\log n}{np} } + \left(1-\frac{\delta}{6(1+\delta)} \right) \left|{\pi}_{i}-\pi^*_{i}\right| \nonumber \\ 
& \ \ \ \  \ \  + \frac{5}{3}(2 \|{\pi} - \pi^*\|_\infty + \|\pi^*\|_\infty ) \sqrt{\frac{\log n}{n p }}  \nonumber \\ \nonumber& +   \frac{1}{3\sqrt{n}}\left( \|{\pi} - \pi^*\|_\infty + \|\pi^*\|_\infty  + \tilde{c} \sqrt{n}\|\pi^*\|_\infty \sqrt{\frac{\log n}{n p }} \right) \\ 
& \  \ \ \ \ \ \ \ \ \ \ \ \ + \frac{1}{\sqrt{6np}}\left( \|{\pi} - \pi^*\|_\infty + \|\pi^*\|_\infty \right) .
\end{align}
Rearranging the terms and taking the maximum over $i \in [n]$ gives
\begin{align*}
& \|{\pi} - \pi^* \|_\infty \left(\frac{\delta}{6(1+\delta)} -  \frac{10}{3} \sqrt{\frac{\log n}{np}} - \frac{1}{\sqrt{6np}} \right) \\ 
&  \leq \|\pi^*\|_{\infty} \bigg( c \sqrt{\frac{\log n}{np}} + \frac{5}{3} \sqrt{\frac{\log n}{n p }} +   \frac{1}{3\sqrt{n}} \\ 
&  \ \ \ \ \ \ \ \ \ \ \ \ \ \ \ \ \ \ \ \ \ \ \ \ \ + \tilde{c}  \sqrt{\frac{\log n}{n p }} + \frac{1}{\sqrt{6np}}
\bigg).
\end{align*}
Thus, we have established the existence of a constant $c_9$ such that for $np \geq c_9 \log n$, the perturbation bound $\|\pi - \pi^*\|_\infty \leq \|\pi^*\|_\infty O(\sqrt{ {\log n}/{n p } })$ holds. This proves  \cref{thm: Perturbation Bound on Sparsified Stationary Distribution}. \qed

\subsection{Proof of \cref{lemS: Handling the first term}} 
\label[appendix]{Proof of lemS: Handling the first term}
For any fixed $i \in [n]$, the term $L_1$ can be simplified as
\begin{align*}
 L_{1} &=\sum_{j=1}^{n} \pi^*_{ j} ({S}_{ j i}-S^*_{j i}) \\ 
 & =\sum_{j : j \neq i}^{n} \pi^*_{ j} \left(\frac{p^*_{j i} \I_{(i,j) \in \calE}}{ d} - \frac{p^*_{j i}}{3 n}\right) \\ 
 & \ \ \ \ \ \ \ \ \ +\pi^*_{i}\left(\sum_{j: j \neq i}\left(\frac{p^*_{i j}}{3 n} - \frac{p^*_{i j} \I_{(i,j) \in \calE}}{ d}\right)\right) \\
& = \sum_{j: j \neq i} \frac{\left(\pi^*_{i} p^*_{ij} - \pi^*_{j} p^*_{ji}\right)}{3 n}\left(1-\frac{3 n \I_{(i,j) \in \calE} }{ d}\right) \\
&  = \sum_{j: j \neq i} \frac{\left(\pi^*_{ i} p^*_{i j} - \pi^*_{j} p^*_{j i} \right)}{3 n}\left(1-\frac{\I_{(i,j) \in \calE}}{p}\right),
\end{align*}
where the last step follows since $d= 3 n p$. Let $\tilde{\sigma}=\max_{i,j} \left|\pi^*_{i} p^*_{i j} - \pi^*_{ j} p^*_{j i}\right|,$ and for any fixed $i$, let
$$
    u_{j}=\frac{\left(\pi^*_{ i} p^*_{i j} -  \pi^*_{j} p^*_{j i}\right)}{3 n}\left(1-\frac{\I_{(i,j) \in \calE}}{p}\right).
$$
Now since each edge $(i, j)$, with $i>j$, is sampled uniformly, we can upper bound $L_{1}$ by using the Bernstein inequality as
$$
\mathbb{P}\left(|L_{1}|>t\right) \leq 2 \exp \bigg(\frac{-t^{2}/2}{\frac{\tilde{\sigma}^{2}(1-p)}{9 n p} + \frac{\tilde{\sigma}(1-p)}{9 n p} t} \bigg).
$$
This is because each of the terms $u_j$ is zero-mean and bounded, i.e., $|u_j| \leq \frac{\tilde{\sigma}(1-p)}{3 n p}$, and has variance bounded by
$$
\begin{aligned}
\E\!\left[u_{j}^{2}\right] & \leq \frac{\tilde{\sigma}^{2}}{9 n^{2}}(1-p)+\frac{\tilde{\sigma}^{2}(1-p)^{2}}{9 n^{2} p} =\frac{\tilde{\sigma}^{2}(1-p)}{9 n^{2} p}.
\end{aligned}
$$
Substituting $t=\frac{5\tilde{\sigma}}{3} \sqrt{\frac{\log n}{n p} }$, and using the fact that $p \geq \frac{\log n}{n}$, we get
$$
\P\left(|L_{1}|>\frac{5\tilde{\sigma}}{3} \sqrt{\frac{\log n}{n p} }\right) \leq O\bigg(\frac{1}{n^{4}}\bigg).
$$
Finally, the lemma follows using $\tilde{\sigma} \leq \|\pi^*\|_\infty$ and a simple union bound argument. \qed

\subsection{Proof of \Cref{lem: Error bound on leave-out}} \label[appendix]{Proof of Error bound on leave-out}

Since ${\pi}^{(m)}$ and ${\pi}$ are stationary distributions of ${S}^{(m)}$ and ${S}$, an application of \cref{lem:PertubationOfPi} gives
\begin{align}
\nonumber & \left\|\pi^{(m)}-{\pi}\right\|_{\pi^{*^{-1}}} \leq \\ 
 & \ \ \frac{\left\| ({\pi}^{(m)})^\T \left({S}^{(m)}-{S}\right)\right\|_{\pi^{*^{-1}}}}{1-\| (\Pi^*)^{\frac{1}{2}} S^* (\Pi^*)^{-\frac{1}{2}} -\sqrt{\pi^*}\sqrt{\pi^{*^\T} }\|_{2}- \left\|{S}-S^*\right\|_{\pi^{*^{-1}}}}. \label{usual step}
\end{align}
Since $S^*$ is the canonical Markov matrix corresponding to a complete graph, we get $\|S^*-\1_n \pi^{*^\T}\|_{\pi^{*^{-1}}} \leq 1-\frac{\delta}{6}$ by \Cref{lem:SpectralGapLB}. We provide a high probability upper bound on $\|{S}-S^*\|_{\pi^{*^{-1}}}$ in the sub-sampled scenario in \cref{lem: Specral Norm sub-sampling bound} (see \cref{Proof of lem: Specral Norm sub-sampling bound}). Combining these two bounds, we obtain that the denominator term is lower bounded by $\delta/12$ for $np \geq O(\log n)$. Finally, we can convert the bound in \cref{usual step} in terms of $\ell^2$-norms as 
\begin{align}
    \left\|\pi^{(m)}-{\pi}\right\|_{2} &\leq \frac{12}{\delta}\sqrt{\frac{\|\pi^*\|_\infty}{\pi^*_{\min}}} \left\| ({\pi}^{(m)})^\T \left({S}^{(m)}-{S}\right)\right\|_{2} \nonumber\\
    & \leq \frac{12}{\delta^2} \left\| ({\pi}^{(m)})^\T \left({S}^{(m)}-{S}\right)\right\|_{2},    \label{eq:usual step 2}
\end{align}
where the last step follows since the principal ratio for $\pi^*$ is upper bounded by $1/\delta^2$ by \cref{lem: Height of Perron–Frobenius Eigenvector}. Now, it remains to bound $\|(\pi^{(m)})^\T({S}^{(m)} -{S})\|_{2}$. For $j \neq m$, the $j$th entry of $(\pi^{(m)})^\T({S}^{(m)} -{S})$ is given by
$$
\begin{aligned}
 & \left[(\pi^{(m)})^\T\left({S}^{(m)}-{S} \right)\right]_{j} \\
 & \ \ \ \ \ \ \ \  ={\pi}^{(m)}_{j} \left({S}^{(m)}_{ jj}- {S}_{j j} \right)+{\pi}_{m}^{(m)}\left({S}^{(m)}_{ mj} -{S}_{m j} \right) \\
& \ \ \ \ \ \ \ \  =-{\pi}^{(m)}_{j} \left({S}^{(m)}_{ j m} -{S}_{j m} \right) + {\pi}_{m}^{(m)}\left({S}_{mj}^{(m)} -{S}_{m j} \right) .
\end{aligned}
$$
Observe that for $j \neq m$, $\big|{S}^{(m)}_{ j m} -{S}_{j m} \big| \leq \frac{1-p}{3np}$ if $(j, m) \in \mathcal{E}$ and $\big|{S}^{(m)}_{j m} -{S}_{j m} \big| \leq \frac{1 }{3n}$ if $(j, m) \notin \mathcal{E}$, and we have
\begin{equation} \label{off diag bound}
    \begin{aligned}
        & \left|\left[(\pi^{(m)})^\T\left({S}^{(m)}-{S} \right)\right]_{j}\right| \\
        & \ \ \ \ \ \ \ \ \ \ \ \ \ \ \  \leq \begin{cases}\frac{2(1-p)}{3np}\left\|\pi^{(m)}\right\|_{\infty}, & \text { if }(j, m) \in \mathcal{E}, \\ \frac{2 }{3n}\left\|\pi^{(m)}\right\|_{\infty}, & \text { otherwise. }\end{cases}
    \end{aligned}
\end{equation}
Similarly, for $j=m$, we get
$$
\begin{aligned}
 & \bigg|[(\pi^{(m)})^\T ( {S}^{(m)}-{S})]_{m} \bigg|\\
 &  \leq\left|{\pi}^{(m)}_{m}\left({S}^{(m)}_{ mm} -{S}_{mm}\right)\right|+\left|\sum_{j: j \neq m}{\pi}^{(m)}_{j} \left({S}^{(m)}_{ jm} -{S}_{j m}\right)\right| \\
& = \bigg|\underbrace{\sum_{j: j \neq m} {\pi}_{m}^{(m)} \left({S}^{(m)}_{mj} -{S}_{m j}\right) }_{ \zeta_{1}} \bigg| \\ 
&  \ \ \ \ \ \ \ \ \ \ \ \ \ \ \ \ \ \ \ \ \ \ \ \ \ \ \ \ +  \bigg| \underbrace{\sum_{j: j \neq m}{\pi}^{(m)}_{j}\left({S}^{(m)}_{ j m} -{S}_{j m}\right)}_{\zeta_{2}} \bigg|.
\end{aligned}
$$
Observe that $\E[{S}_{mj}] = S^{(m)}_{mj}$, and moreover, ${S}_{m j}^{(m)} - {S}_{m j} = \frac{p_{m j}}{d}\left(p-\I_{(m ,j) \in \mathcal{E}}\right)$. Thus, by the Bernstein inequality, we get
$$
\P(|\zeta_1| \geq t) \leq 2\exp\bigg(-\frac{t^2/2}{\frac{(1-p) \|{\pi}^{(m)}\|_\infty^2}{9np} + \frac{ (1-p)\|{\pi}^{(m)}\|_\infty}{3np}\frac{t}{3} } \bigg).
$$
Substituting $ t =  \frac{5}{3} \sqrt{\frac{ \log n}{np}} \left\|{\pi}^{(m)}\right\|_{\infty}$ gives (for $n p \geq \log n $) 
$$
\P\bigg(|\zeta_1| \geq \frac{5}{3} \sqrt{\frac{ \log n}{np}} \left\|{\pi}^{(m)}\right\|_{\infty}\bigg) 
\leq O\bigg(\frac{1}{n^4}\bigg),
$$
where in the last step, we have utilized the fact that $np \geq \log n$. Using the same technique, we can obtain a similar bound for $\zeta_2$. Now, using \cref{eq:usual step 2,off diag bound} and bounds on $\zeta_1$ and $\zeta_2$ to upper bound $\|(\pi^{(m)})^\T(S^{(m)} - S)\|_2$ gives
$$ 
\|{\pi}^{(m)} - \pi\|_2\! \leq\! {\frac{12}{\delta^2} } \left(\frac{10}{3} \sqrt{\frac{ \log n}{np}}\! + \! \frac{2\sqrt{d_{\max}}}{3np}\! +\!\frac{2\sqrt{n}}{3n}\right)\left\|{\pi}^{(m)}\right\|_{\infty}.
$$
Since $d_{\max} = 3np/2$, for $np \geq c_0 \log n$ with constant $c_0$ and parameter $n$ large enough, we get 
$$ \|{\pi}^{(m)} - {\pi}\|_2  \leq\frac{1}{2} \|{\pi}^{(m)}\|_{\infty}. $$
An application of the triangle inequality gives
$$
\|{\pi}^{(m)} - {\pi}\|_2  \leq\frac{1}{2} \|{\pi}^{(m)} - {\pi}\|_{2} + \frac{1}{2}\|{\pi} - \pi^*\|_\infty + \frac{1}{2}\|\pi^*\|_\infty.
$$
Hence, rearranging the terms completes the proof. \qed

\subsection{Proof of \cref{lem: L2 error subsampling}} \label[appendix]{Proof of L2 error subsampling} 
An application of \cref{lem:PertubationOfPi} gives
$$
\begin{aligned}
& \left\|{\pi}-\pi^*\right\|_{\pi^{*^{-1}}} \\ 
&\ \ \ \ \  \leq \frac{\|\pi^{*^{\T}}\left({S}-S^*\right)\|_{\pi^{*^{-1}}}}{1-\|\Pi^{*^{\frac{1}{2}}} S^* \Pi^{*^{-\frac{1}{2}}}- \sqrt{\pi^*} \sqrt{\pi^*}^{\T}\|_{2} - \|{S} - S^*\|_{\pi^{*^{-1}} }}.
\end{aligned}
$$
Utilizing \cref{lem: Specral Norm sub-sampling bound} and the same argument as in \Cref{Proof of Error bound on leave-out} (cf. \cref{eq:usual step 2}), we obtain  
$$\|{\pi}-\pi^*\|_{2} \leq \frac{12}{\delta^2}\|\pi^{*^\T}({S}-S^*)\|_{2}.$$
Now, it remains to bound $\|\pi^{*^{\T}}({S}-S^*)\|_{2}$. An application of the triangle inequality yields
\begin{equation} \label{triangleq to bound the terms}
\|\pi^{*^\T}({S}-S^*)\|_{2} \leq \|\pi^{*^\T}D\|_{2} + \|\pi^{*^\T}({S_0}-S_0^*)\|_{2} ,
\end{equation}
where $D$ is a diagonal matrix with $D_{ii} = {S}_{ii} - S^*_{ii}$, and ${S}_0, S^*_0$ are the same as the matrices ${S}, S^*$ but with diagonal entries set to $0$. Observe that
$$ (\pi^{*^\T} D)_i = \pi^*_i ({S}_{ii} - S^*_{ii}) = \pi^*_i \frac{1}{d} \sum_{j:j\neq i} p^*_{ij}(p - \I_{(i,j) \in \calE}).  $$
An application of the Bernstein inequality gives 
$$ \P(|(\pi^{*^\T} D)_i| \geq t) \leq 2\exp\left(-\frac{-t^2/2}{ \frac{\|\pi^*\|_\infty^2}{9 np} + \frac{\|\pi^*\|_\infty}{3np} \frac{t}{3} }\right) .$$
Substituting $t = \frac{5}{3}\|\pi^*\|_\infty\sqrt{\frac{\log n}{np}}$, we obtain 
$$ \P\bigg(\|\pi^{*^\T}D\|_\infty \geq  \frac{5}{3}\|\pi^*\|_\infty\sqrt{\frac{\log n}{np}} \bigg) \leq O\bigg(\frac{1}{n^3}\bigg). $$ 
Therefore, with probability at least $1 - O(n^{-3})$, we have $\|\pi^{*^\T}D\|_2 \leq \sqrt{n} \|\pi^*\|_{\infty}$. Applying the same technique to $(\pi^{*^\T}({S_0}-S_0^*))_i$ as in \cref{lem:L2error} (or \cite[Theorem 9]{Chenetal2019}) yields a similar high probability bound on $\|\pi^{*^\T}({S_0}-S_0^*)\|_2$, and thus, combining the bounds with \cref{triangleq to bound the terms} completes the proof. \qed

\subsection{Error Bound on Spectral Norm under Sub-sampling} \label{Proof of lem: Specral Norm sub-sampling bound}
\begin{lemma}[Spectral Norm Sub-sampling Bound] \label{lem: Specral Norm sub-sampling bound} 
For the setting of \cref{thm: Perturbation Bound on Sparsified Stationary Distribution}, there exists a constant $c>1$ such that for $n p \geq c \log n$, the following bound holds: 
\begin{equation*}
    \|{S} - S^*\|_{\pi^{*^{-1}}} \leq \sqrt{\frac{\|\pi^*\|_\infty}{\pi^*_{\min}} }\|{S} - S^*\|_{2} \leq \frac{\delta}{12}, 
\end{equation*}
with probability at least $1- O(n^{-3})$.
\end{lemma}

\begin{proof}
 Note that, as the stationary distribution $\pi^*$ corresponds to a complete graph, its principal ratio is bounded by $1/\delta^2$ as established in \cref{lem: Height of Perron–Frobenius Eigenvector}. Consequently, we have
$$
\left\|S-S^{*}\right\|_{\pi^{*}} \leq \sqrt{\frac{\|\pi^{*}\|_{\infty}}{\pi^*_{{\min}}} }  \|S -  S^{*}\|_{2} \leq \frac{1}{\delta} \|S -  S^{*}\|_{2}.
$$
In order to bound $\left\|S-S^{*}\right\|_{2}$, we will decompose it as
$$
\left\|S-S^{*}\right\|_{2} \leq\|\tilde{S}-\tilde{S}^{*} \|_{2} + \| D \|_{2},
$$
where $D$ is a diagonal matrix such that $D_{i i} =\left(S-S^{*}\right)_{i i} $ and $\tilde{S},\tilde{S}^{*}$ are the same as $S, S^{*}$ but with diagonal entries set to $0$. Now, observe that
$$
\begin{aligned}
D_{ii} & =\frac{1}{3 n} \sum_{j: j \neq i} p^*_{i j} \bigg (1  - \frac{ \I_{(i, j) \in \calE} }{p}  \bigg). 
\end{aligned}
$$
Therefore, by the Bernstein inequality, we get
$$
\begin{aligned}     
\P\left(\left|D_{i i}\right|>t\right) & \leq 2\exp\bigg(-\frac{t^{2}/2}{ \frac{1}{9 n p}+\frac{(1-p) t}{9n p}} \bigg).
\end{aligned}
$$
Using the fact that $np \geq \log n$, substituting $t  =\frac{5}{3} \sqrt{\frac{\log n}{n p}}$, and using a simple union bound argument, we obtain the following bound with probability at least $1 -O(n^{-3})$:
$$
\|D\|_2 \leq \frac{5}{3} \sqrt{\frac{\log n}{n p}}.
$$
For $np \geq (\frac{40}{\delta^2})^2 \log n$, we have that $\|D\|_2 \leq \frac{\delta^2}{24}$. Now to bound the term $\|\tilde{S}-\tilde{S}^{*}\|_{2}$, we will invoke \cite[Theorem 6.3]{candes2012exact} to obtain that for $np \geq 3 \log n$, there exists a constant $c_0$ such that we have the following bound with probability at least $1 - O(n^{-3})$:
$$
\|\tilde{S}^{*} - \tilde{S}\|_2 \leq  \sqrt{ \frac{3 c_0 \log n}{n p} }.
$$ 
Now for $n p \geq (1728 c_0/\delta^2) \log n$, we obtain that $\|\tilde{S}^{*} - \tilde{S}\|_2 \leq \frac{\delta^2 }{24}$, thus proving the theorem.
\end{proof}

\section{Proof of \cref{lem:L2error}} \label[appendix]{Proof of L2 error empirical} 

This argument follows the proof of \cite[Theorem 9]{Chenetal2019}. An application of \cref{lem:PertubationOfPi} yields
$$
\begin{aligned}
& \left\|\hat{\pi}-\pi\right\|_{\pi^{-1}}\! \leq \!\frac{\|\pi^{\T} (\hat{S}-S)\|_{\pi^{-1}}}{1-\|\Pi^{\frac{1}{2}} S \Pi^{-\frac{1}{2}} \!-\! \sqrt{\pi} \sqrt{\pi}^{\T}\|_{2} - \|\hat{S} - S\|_{\pi_{*}^{-1}}}.
\end{aligned}
$$
Utilizing \cref{combinedupperbound}, we get  
$$\|\hat{\pi}-\pi\|_{2} \leq \frac{8}{\xi^{2}} \sqrt{\frac{\|\pi\|_{\infty}}{\pi_{\min} }} \|\pi^{\T} (\hat{S}-S)\|_{2}. $$
By \cref{assm: Bounded Principal Ratio}, we know that the principal ratio of $\pi$ is bounded, as it is the stationary distribution corresponding to a complete graph.
Now, it remains to bound $\|\pi^{\T}(\hat{S}-S)\|_{2}.$ 
Consider the $i$th coordinate of $\pi^{\T}(\hat{S}-S)$:
$$
\begin{aligned}
    & \big(\pi^{\T}(\hat{S}-S)\big)_{i} \\
    & \ \ \ \ \ \ =\pi_{ i}\left(\hat{S}_{i i}-{S}_{ i i}\right)+\sum_{j:j \neq i} \pi_{j}\left(\hat{S}_{ ji}-S_{ji}\right)\\
    & \ \ \ \ \ \ = \underbrace{-\pi_{ i}\frac{1}{d}\sum_{j \neq i} \left(\hat{p}_{i j}-{p}_{ i j}\right)}_{\zeta_{1,i} }+\underbrace{\frac{1}{d}\sum_{j: j \neq i} \pi_{j}\left(\hat{p}_{i j}-p_{i j}\right)}_{\zeta_{2,i} }.
\end{aligned}
$$
Observe that each of $\zeta_{1,i}$ (or $\zeta_{2,i}$) is a sum of at most $k d_{\max}$ independent zero-mean random variables different $\zeta_{1,i}$ (or $\zeta_{2,i}$) are independent of one another. Moreover, using Hoeffding's inequality, we can show that either of $\zeta_{1,i}$ or $\zeta_{2,i}$ are sub-Gaussian as
$$ \forall i \in [n], \enspace \mathbb{P}\left(\left|\zeta_{1,i} \right| \geq t\right) \leq 2 \exp \left(-\frac{2 t^{2}}{\frac{1}{(kd)^{2}} d_{\max} k \|\pi \|_{\infty}^{2}}\right).
$$
Hence $\zeta_{1,i}$ (or $\zeta_{2,i}$) can be treated as a sub-Gaussian random variable with variance proxy $\sigma^{2} =   \frac{\|\pi\|_{\infty}^{2}}{8 k d_{\max} }$. Now, we rewrite  $\|\pi^\T (\hat{S} - S)\|_2^2$  as
$$
 \|\pi^\T (\hat{S} - S)\|_2^2 = \sum_{i=1}^n (\pi^{\T} (\hat{S}-S))_i^{2} \leq 2\sum_{i=1}^{n} \zeta_{1,i}^2 + \zeta_{2,i}^2 ,
$$
which is a sum of squares form of a sub-Gaussian vector. Therefore, $\E[  \|\pi^\T (\hat{S} - S)\|_2^2] \leq  4 n \sigma^{2}$. To find a high probability bound, we utilize the Hanson-Wright inequality  \cite[Theorem 1.1]{rudelson2013hanson} to obtain the following bound for some constant $c>0$:
$$
\begin{aligned}
& \mathbb{P}\left( \sum_{i=1}^n \zeta_{1,i}^2 + \zeta_{2,i}^2 -\mathbb{E} [ \zeta_{1,i}^2 + \zeta_{2,i}^2] \geq t\right) \\ 
& \leq \mathbb{P}\left( \sum_{i=1}^n \zeta_{1,i}^2  -\mathbb{E} [ \zeta_{1,i}^2] \geq \frac{t}{2}\right) + \P\left(\sum_{i=1}^n \zeta_{2,i}^2 - \E[\zeta_{2,i}^2] \geq \frac{t}{2} \right) \\ 
& \leq 2 \exp \left(-c \min \left\{\frac{t^{2}}{n \sigma^{4}}, \frac{t}{\sigma^{2}}\right\}\right).
\end{aligned}
$$
Substituting $t = O(\sigma^{2} \sqrt{n \log n})$, and a simple calculation yields that with probability at least $1-O\left(n^{-4}\right)$,
$$
\begin{aligned}
    \|\pi^\T (\hat{S} - S)\|_2^{2} & \leq \mathbb{E}\left[\|\pi^\T (\hat{S} - S)\|_2^2\right]+ \tilde{c} \sigma^{2} \sqrt{n \log n} \\ 
    & \leq \hat{c} n \sigma^{2}  \leq \frac{\hat{c} n}{8 k d_{\max} }\left\|\pi\right\|_{\infty}^{2}.
\end{aligned}
$$
This completes the proof. \qed

\section{Improved Bounds on the Spectral Norm}
The ensuing lemmata present bounds on the spectral norm of the error matrix for the estimation of mean and squared mean. 

\begin{lemma}[Spectral Norm of Estimation Error for Mean]\label{lem:Spectral Norm of Error Improved}
Let $A \in \R^{n \times n}$ be matrix such that $A_{ij} \sim \textup{Bin}(p_{ij},k)$ for every $(i,j) \in \calE$ such that $i \neq j$ and $0$ otherwise. Define a matrix $\hat{A}$ such that $\hat{A}_{ij} = \frac{1}{k} A_{ij} - p_{ij}$ for every $(i,j) \in \calE$ such that $i \neq j$ and $0$ otherwise. Then, there exists a constant $c > 0$ such that we have
\begin{align*}
    & \mathbb{P}\Bigg(\|\hat{A}\|_2 \geq \sqrt{\frac{d_{\max }}{k}}+c\Bigg(\frac{d_{\max }^{1 / 4}}{\sqrt{4 k}}(\log n)^{3 / 4}+\sqrt{\frac{t}{4 k}} \\ 
    &\ \ \ \ \ \ \ \  \ \ \ \ \ \ \ \ \ \ \ \ \ \ \ \ \  +\frac{d_{\max }{ }^{1 / 3}}{(2 k)^{2 / 3}} t^{2 / 3}+\frac{t}{k}\Bigg)\Bigg) \leq 4 n e^{-t} .
\end{align*}
\end{lemma}
\begin{lemma}[Spectral Norm of Estimation Error for Squared Mean]\label{lem:Spectral Norm of Error Improved2} 
Let $Z \in \R^{n \times n}$ be matrix such that $Z_{ij} \sim \textup{Bin}(p_{ij},k)$ for every $(i,j) \in \calE$ for $i \neq j$ and $0$ otherwise. Define a matrix $\hat{Y}$ such that $\hat{Y}_{ij} = \frac{Z_{ij}(Z_{ij} - 1)}{k(k-1)}$ for every $(i,j) \in \calE$ for $i \neq j$ and $0$ otherwise. Then, there exists a constant $c > 0$ such that we have
\begin{align*}
    & \mathbb{P}\left(\|\hat{Y} -\E[Y]\|_2 \geq \sqrt{\frac{24 d_{\max }}{k}}+ c\left(\frac{d_{\max }^{1 / 4}}{\sqrt{k}}(\log n)^{3 / 4} \right. \right. \\ 
    & \ \ \ \ \ \ \ \ \ \ \ \ \ \ \ \ \ \ \ \ \ \ \left. \left.  +\sqrt{\frac{6t}{k}}+\frac{d_{\max }{ }^{1 / 3}}{( k )^{1 / 3}} t^{2 / 3}+ t\right)\right) \leq 4 n e^{-t} .
\end{align*}
\end{lemma}
The proofs of \cref{lem:Spectral Norm of Error Improved} and \cref{lem:Spectral Norm of Error Improved2} are provided in \cref{Proof of lem:Spectral Norm of Error Improved} and \cref{Proof of lem:Spectral Norm of Error Improved2}, respectively. The proofs of both lemmata are based on the ensuing proposition from \cite{brailovskaya2022} (presented almost verbatim as in \cite{brailovskaya2022}).
\begin{proposition}[Spectral Norm Bound {\cite[Corollary 2.15]{brailovskaya2022}}] \label{thm:user friendly bounds} Let $Y=\sum_{i=1}^{d} Z_{i}$, where $Z_{1}, \ldots, Z_{n}$ are independent (possibly non-self-adjoint) ${n} \times {n}$ random matrices with $\E\left[Z_{i}\right]=0$. Then, there exists a constant $c > 0$ such that 
\begin{align}
& \P\big(\|Y\|_2 \geq \|\E [Y^\T Y]\|^{\frac{1}{2}}_2 \! + \! \|\E [Y Y^\T] \|^{\frac{1}{2}}_2 \nonumber \\ 
& \ \ \ \ \ \  + c(\nu(Y)^{\frac{1}{2}} \hat{\sigma}(Y)^{\frac{1}{2}}(\log n)^{\frac{3}{4}}+\hat{\sigma}_{*}(Y) t^{\frac{1}{2}} + \nonumber \\ 
& \ \ \ \ \ \ \ \ \ \ \ \ \ \ R(Y)^{\frac{1}{3}} \hat{\sigma}(Y)^{\frac{2}{3}} t^{\frac{2}{3}} +R(Y) t)\big) \leq 4 n e^{-t}  \label{user friendly bounds}
\end{align}
for all $t \geq 0$. Here, $\hat{\sigma}(Y), \hat{\sigma}_{*}(Y), \nu(Y), R(Y)$ are defined as
\begin{equation} \label{parameters}    
\begin{aligned}
& \hat{\sigma}(Y) \triangleq \max \left(\left\|\E [Y^\T Y]\right\|_2^{\frac{1}{2}},\left\|\E [Y Y^\T] \right\|_2^{\frac{1}{2}}\right), \\
& \hat{\sigma}_{*}(Y)\triangleq\sup _{\|v\|_2=\|w\|_2=1} \E\left[| v^\T(Y-\E [Y]) w|^{2}\right]^{\frac{1}{2}}, \\
& \nu(Y)\triangleq\|\operatorname{cov}(Y)\|_2^{\frac{1}{2}},\\
& R(Y)\triangleq \vertiii{ \max _{1 \leq i \leq n}\| Z_{i}\|_2 }_{\infty} ,
\end{aligned}
\end{equation}
where $\operatorname{cov}(Y) \in \mathbb{R}^{n^{2} \times n^{2}}$ such that $\operatorname{cov}(Y)_{i j, k l}=\mathbb{E}[Y_{i j} Y_{k l}]$, and $\vertiii{M}_\infty$ denotes the essential supremum of the random variable $|M|$.
\end{proposition}

\subsection{Proof of \cref{lem:Spectral Norm of Error Improved}} \label[appendix]{Proof of lem:Spectral Norm of Error Improved} 
For $(i, j) \in \mathcal{E}$ and $i \neq j$, we write $\hat{A}_{ij}$ as a sum of $k$ independent Bernoulli random variables $Z_{m,{ij}} \sim \Ber(p_{ij})$:
$$
\hat{A}_{i j}=\frac{1}{k} \sum_{m=1}^{k}\left(Z_{m,{i j}}-p_{i j}\right).
$$
Define a matrix $V^{i j, m}=\frac{Z_{m,{i j}}-p_{i j}}{k} e_{i} e_{j}^{\T} $ for $(i, j) \in \mathcal{E}$ and $m \in [k]$. Then,
$$
\hat{A} = \sum_{\substack{(i, j) \in \calE: \\ i \neq j}} \sum_{m=1}^{k} V^{i j, m}.
$$
Next, in order to apply \cref{thm:user friendly bounds} on $\hat{A}$, we need to calculate parameters $\hat{\sigma}(\hat{A}), \hat{\sigma}_{*}(\hat{A}), \nu(\hat{A}), R(\hat{A})$ in \cref{parameters}. Note that each entry of $A_{ij}$ is independent of one another. Therefore, for $i \neq j$, $\mathbb{E}[(\hat{A}^{\T}\hat{A})_{i j}]=0$. Moreover, we have
$$
\begin{aligned}
\E[(\hat{A}^{\T}\hat{A})_{j j}] & = \sum_{\substack{i:(i, j) \in \mathcal{E}, \\ i \neq j} }\E[(\hat{A}_{i j})^{2}]  =\sum_{\substack{i:(i, j) \in \calE, \\ i \neq j}} \frac{p_{i j}\left(1-p_{i j}\right)}{k}.
\end{aligned}
$$
Similarly, one can show that $\E[(\hat{A}\hat{A}^{\T})_{j j}]$ is also a diagonal matrix with diagonal entries given by
$$
\begin{aligned}
\E[(\hat{A}\hat{A}^{\T})_{ii}] & = \sum_{\substack{j:(i, j) \in \mathcal{E}, \\ i \neq j} }\E[(\hat{A}_{i j})^{2}]  =\sum_{\substack{j:(i, j) \in \calE, \\ i \neq j}} \frac{p_{i j}\left(1-p_{i j}\right)}{k}.
\end{aligned}
$$

Therefore, both $\E[\hat{A}^{\T}\hat{A}]$ and $\mathbb{E}[\hat{A}\hat{A}^{\T}]$ are diagonal matrices with diagonal entries bounded above by $\frac{d_{\max }}{4 k}$. Therefore $\hat{\sigma}(\hat{A}) \leq \sqrt{\frac{d_{\max }}{4 k}}$. Now we will bound $\hat{\sigma}_{*}(\hat{A})$. For a fixed $v, w \in \mathbb{R}^{n}$, we have
\begin{align*}
\mathbb{E}[|u^{\T}\hat{A} w|^{2}] & =\mathbb{E}\left[\left|\sum_{\substack{(i,j) \in \calE: \\ i \neq j}} \hat{A}_{i j} v_{i} w_{i }\right|^{2}\right] \\
&=\sum_{\substack{(i,j) \in \calE: \\ i \neq j}} v_{i }^{2} w_{j}^{2}\mathbb{E}[\hat{A}_{i j}^{2}] \\
& =\sum_{\substack{(i,j) \in \calE: \\ i \neq j}} v_{i}^{2}w_{j}^{2} \frac{p_{i j}\left(1 - p_{i j}\right)}{k}. 
\end{align*}
Taking supremum on both sides with respect to $v$ and $w$ such that $\|v\|_{2}=1$ and $\|w\|_{2}=1$ gives $\hat{\sigma}_*(\hat{A}) \leq \sqrt{\frac{1}{4k}}.$ Also $\operatorname{cov}(\hat{A}) \in \mathbb{R}^{n^{2} \times n^{2}}$ is a diagonal matrix with its diagonal entries $\operatorname{cov}(\hat{P}-P)_{i j, i j}$ given by
$$
\operatorname{cov}(\hat{A})_{i j, i j}=\frac{p_{i j}\left(1-p_{i j}\right)}{k}, \text{ if } (i,j) \in \calE .
$$
Therefore, $\nu(\hat{A}) \leq \frac{1}{\sqrt{4 k}}$. Finally, $R(\hat{A}) \leq \frac{1}{k}$ as $\left|V^{i j, m}\right| \leq \frac{1}{k}$. Substituting the above bounds into \cref{user friendly bounds} proves the lemma. \qed

\subsection{Proof of \cref{lem:Spectral Norm of Error Improved2}} \label[appendix]{Proof of lem:Spectral Norm of Error Improved2}
 Define $\hat{V}^{ij} = (\hat{Y}_{ij} - p_{ij}^2)e_i e_j^\T$ and note that we can decompose $\hat{Y} - \E[\hat{Y}]$ in terms of $\hat{V}^{ij} $ as 
 $$\hat{Y} - \E[\hat{Y}] = \sum_{\substack{i:(i, j) \in \calE, \\
i \neq j}} \hat{V}^{ij}.$$
Next, we will apply \cref{thm:user friendly bounds} on the $\hat{V}^{ij}$, but first we need to calculate the parameters $\hat{\sigma}(\hat{Y}-\E[\hat{Y}]).$  
Using the same reasoning as in the proof of \cref{lem:Spectral Norm of Error Improved}, for $(i,j) \in \calE$ and $i \neq j$, we have $\mathbb{E}[((\hat{Y}-\E[\hat{Y}])^{\T}(\hat{Y}- \E[\hat{Y}]))_{i j}]=0$. Moreover,
$$
\begin{aligned}
& \E[((\hat{Y}-\E[\hat{Y}])^{\T}(\hat{Y}-\E[\hat{Y}]))_{jj} ]   =\sum_{\substack{i:(i, j) \in \calE, \\
i \neq j}} \mathbb{E}\left[\left(\hat{Y}_{i j}-p_{i j}^{2}\right)^{2}\right] \\
& =\sum_{\substack{i:(i, j) \in \calE, \\
i \neq j}} \frac{-2(2 k-3) p_{i j}^{4}+4(k-2) p_{i j}^{3}+2 p_{i j}^{2}}{k(k-1)}\leq \frac{6d_{\max}}{k} .
\end{aligned}
$$
Therefore, we obtain
$$
\hat{\sigma}(\hat{Y}-\E[\hat{Y}]) \leq \sqrt{\frac{6 d_{\max }}{k}}.
$$
Similarly, using the same reasoning as in the proof of \cref{lem:Spectral Norm of Error Improved}, we can show that
$$
\begin{aligned}
    \hat{\sigma}_{*}(\hat{Y}-\E[\hat{Y}]) &\leq \sqrt{\frac{6}{k}},  \nu(\hat{Y}-\E[\hat{Y}]) \leq \sqrt{\frac{6}{k}}, \\ 
    \text { and } R(\hat{Y}-\E[\hat{Y}]) & \leq 1. 
\end{aligned}
$$
Therefore, the proof follows by plugging in these bounds into \cref{user friendly bounds}. \qed

\section{Stability of the BTL Assumption} \label[appendix]{sec: Stability of the BTL Assumption}

In this appendix, we analyze the stability of the BTL assumption in the context of rankings for complete graphs. Specifically, we present an observation regarding how induced rankings behave when BTL models are perturbed by small amounts bounded by the scaling of the critical threshold of our testing framework. Recall that the \emph{BTL ranking} orders agents based on the stationary distribution $\pi$ of the canonical Markov matrix (even for general pairwise comparison models, although usually, BTL rankings are used in the context of BTL models). Meanwhile, the \emph{Borda ranking} is more general, as it does not rely on the BTL assumption and instead orders agents based on their Borda counts or scores \cite{borda1781} (defined next). We define the \emph{Borda count} $\tau_i(P)$ of an agent $i \in [n]$ as the (scaled) probability that $i$ beats any other agent selected uniformly at random \cite{borda1781}: 
\begin{equation}
        \tau_i(P) \triangleq \sum_{j= 1}^n (1 - p_{ij}).
\end{equation}
If the BTL assumption holds, then the Borda ranking equals the BTL ranking. Our goal here is to determine the size of the deviation from the BTL condition for a pairwise comparison model that is sufficient to produce a discrepancy between the BTL and Borda rankings. In particular, we will treat the Borda rankings as the ``ground truth.'' For simplicity in this section, we consider the symmetric setting where pairwise comparison matrices $P$ satisfy $p_{ij} + p_{ji} =1$ for all $i,j \in [n]$.  
The next proposition shows that the stability of the BTL assumption decreases as $n$ grows, i.e., as $n$ increases, smaller deviations from the BTL condition can lead to inconsistent BTL and Borda rankings. 
\begin{proposition}[Stability of BTL Assumption] \label{Prop: Stability of BTL Model}
Given any pairwise comparison matrix $P \in (0,1)^{n \times n}$ over a complete graph, the following are true:

\begin{enumerate}
\item Define the error matrix $\mathsf{E} \triangleq \Pi P + P^\T \Pi-\mathbf{1}_n\pi^\T$. For any agents $i,j \in [n]$, $i \neq j$, the relative BTL and Borda rankings of $i$ and $j$ have the following relationship: $\pi_i \geq \pi_j$ if and only if
$$
    \tau_{i}(P)-\tau_{j}(P) \geq \sum_{k=1}^{n} \frac{\mathsf{E}_{i k}}{\pi_{i}+\pi_{k}}-\frac{\mathsf{E}_{j k}}{\pi_{j}+\pi_{k}}.
$$
\item There exists a sub-sequence of pairwise comparison matrices $P \in (0,1)^{n \times n}$ such that agents $1$ and $2$ have constant Borda count difference $\Delta \tau \triangleq \tau_{1}(P)-\tau_{2}(P) > 0$, but are ranked in the opposite order in a BTL ranking, i.e., $\pi_{1} < \pi_{2}$.\footnote{We drop the subscript $n$ for the sub-sequence $P_n$ and write it as $P$ for clarity.} Moreover, the deviation of $P$ from the BTL condition decays as 
    $$\frac{\left\|\Pi P  +P  \Pi-\1_n \pi^{\T}\right\|_{\F} }{n \|\pi\|_\infty} \leq \frac{c_{11}}{\sqrt{n}}, $$ 
for some constant $c_{11} > 0$. 
\end{enumerate}
\end{proposition}

\begin{proof} ~\newline
\indent
   \textbf{Part 1:} Let $\Delta \tau = \tau_i(P) - \tau_j(P)$, observe that
    \begin{align*}
      & \tau_i(P) - \tau_j(P)  = \sum_{k=1}^n p_{jk} - p_{ik} \\ 
      & = \sum_{k=1}^n p_{jk} - \frac{\pi_{k}}{\pi_j + \pi_k} + \sum_{k=1}^n  \frac{\pi_{k}}{\pi_i + \pi_k}  - p_{ik} \\
      & \ \ \ \ \ \ \ \ \ + (\pi_{i} -\pi_j) \sum_{k=1}^n \frac{\pi_k}{(\pi_j + \pi_k)(\pi_i + \pi_k)}  \\ 
      & = \sum_{k=1}^n \frac{\mathsf{E}_{jk}}{\pi_j + \pi_k} - \frac{\mathsf{E}_{ik}}{\pi_i + \pi_k} + \sum_{k=1}^n \frac{\pi_k (\pi_{i} -\pi_j) }{(\pi_j + \pi_k)(\pi_i + \pi_k)} .
    \end{align*}
    Clearly, $\pi_i \geq \pi_j$ if and only if 
        \[
            \Delta \tau \geq \sum_{k=1}^n \frac{\mathsf{E}_{jk}}{\pi_j + \pi_k} - \frac{\mathsf{E}_{ik}}{\pi_i + \pi_k}.
        \]
 
   \textbf{Part 2:} 
    The construction of the matrix $P \in (0,1)^{n\times n}$ is as follows. (In the proof, we drop the subscript $n$ in the sub-sequence $P_n$ for brevity.) For simplicity, we assume $n$ is even; otherwise, we can replace $n/2$ in the construction with $\lceil n/2 \rceil$ to get the corresponding results. Below, we define the pairwise comparison matrix $P_n$ for $j \geq i$. The rest of the values can be obtained by the skew-symmetric condition $p_{ij} + p_{ji}=1$. 
    $$
        p_{i j}=\left\{\begin{array}{cl}
        1 / 2 & 1 \leq i \leq j \leq \frac{n}{2} \\
        1 / 2 & \frac{n}{2} < i \leq j \leq n \\
        1 / 2+2 \eta / n & i=1, j > n / 2 \\
        1 / 2 & i=2, \frac{n}{2} < j \leq n-l \\
        \frac{1}{2}+(\eta-\alpha) / l & i=2, n-l <j \leq n   \\
        1 / 2 & 3 \leq i \leq \frac{n}{2}, \frac{n}{2} < j \leq n
        \end{array}\right. ,
    $$
    where $l$ is a integer less than $n / 2$ and will be defined below.
    We are primarily interested in the first two rows of $P$. Their respective Borda scores are $\tau_{1}(P)=\frac{n}{2}-\eta$ and $\tau_{2}(P)=\frac{n}{2}-\eta+\alpha$ and hence $\Delta \tau=\tau_{1}(P)-\tau_{2}(P)= -\alpha$. Clearly, under Borda count ranking, item $2$ is ranked higher than $1$. We will show that $1$ is always ranked higher than $2$ under the BTL ranking. Moreover, we will show that the deviation of $P$ from the BTL condition decays as 
    $$\left\|\Pi P+P \Pi-\1_n \pi^{\T}\right\|_{\F} \leq \frac{c}{\sqrt{n}} . $$
    To show both these facts, we need to calculate the stationary distribution $\pi$ of the canonical Markov matrix corresponding to $P$. We set $l=\lceil 2 \eta\rceil$ and $\eta=4 \alpha n$ and $\alpha=0.01$. Let $\pi_{1}=a$ and $\pi_{2}=b$.
    Using the symmetric structure of $P$, it is easy to see that: 
    $$
        \begin{aligned}
        & \pi_{3}=\pi_{4}=\ldots =\pi_{\frac{n}{2}}=c, \\
        & \pi_{n / 2+1}=\ldots =\pi_{n-l}=d, \\
        & \pi_{n-l+1}=\cdots \cdot=\frac{\pi}{2}=e.
        \end{aligned}
    $$
    Here, $e$ denotes a variable and should not be confused with the mathematical constant.
    The above equations and the fact that $\pi$ is a probability vector give
    $$
        a+b+c\left(\frac{n}{2}-2\right)+\left(\frac{n}{2}-l\right) d+l e=1,
    $$
    and by the stationarity of $\pi$, we have the following set of equations: 
    $$
    \begin{aligned}
        & a\left(\frac{n-1}{2}+\eta\right)=\frac{b}{2}+\frac{c}{2}\left(\frac{n}{2}-2\right)+d\left(\frac{n}{2}-l\right)\left(\frac{1}{2}-\frac{2 \eta}{n}\right) \\
        & \ \ \ \ \ \ \ \ \ \ \ \ \ \ \ \ \ \ \ \ \ \ \ \ \ \ \ \ \ \ \ \ \ \ \ \ \ \ \  +l e\left(\frac{1}{2}-\frac{2 \eta}{n}\right), \\
        & b\left(\frac{n-1}{2}+\eta-\alpha\right)=\frac{a}{2}+\frac{c}{2}\left(\frac{n}{2}-2\right)+\frac{d}{2}\left(\frac{n}{2}-l\right) \\
        & \ \ \ \ \ \ \ \ \ \ \ \ \ \ \ \ \ \ \ \ \ \ \ \ \ \ \ \ \ \ \ \ \ \ \ \ \ \ \  +l e\left(\frac{1}{2}-\frac{\eta-\alpha}{l}\right), \\
        & c\left(\frac{n-1}{2}\right)=\frac{a+b}{2}+\frac{c}{2}\left(\frac{n}{2}-3\right)+\frac{d}{2}\left(\frac{n}{2}-l\right)+\frac{l e}{2}, \\
        & d\left(\frac{n-1}{2}-\frac{2 \eta}{n}\right)= a\left( \frac{1}{2} + \frac{2\eta}{n}\right) +\frac{b}{2}+\frac{c}{2}\left(\frac{n}{2}-2\right) \\
        & \ \ \ \ \ \ \ \ \ \ \ \ \ \ \ \ \ \ \ \ \ \ \ \ \ \ \ \ \ \ \ \ \ \ \ \ \ \ \  +\frac{d}{2}\left(\frac{n}{2}-l-1\right)+\frac{l e}{2}.
    \end{aligned}
    $$
    For simplicity, we restrict our attention to sub-sequences where the matrix dimension $n$ is such that $2\eta$ is an integer. Under this assumption, the system of equations can be solved using Wolfram Mathematica, and the results are presented based on the dominant terms of the polynomials. The coefficients of the polynomials are accurate to four decimal places. 
    $$
    \begin{aligned}
        a & = \frac{0.8518}{n}\frac{\left(n^{9}-1.4541 n^{8}+0.6307 n^{7} + O(n^{6}) \right)}{g(n)}, \\
        b & = \frac{0.8518}{n}\frac{\left(n^{9}-1.5491 n^{8}+0.6723 n^{7} + O(n^{6}) \right)}{g(n)}, \\
        c & = \frac{1}{n}, \quad d  = \frac{1}{n}\frac{\left(n^{9}-1.1063 n^{8}+0.2125n^{7} + O(n^{6} \right)}{g(n)}, \\
        e & = \frac{1}{n} \frac{\left(n^{9}+0.7455n^{8}-0.5282 n^{7} + O(n^6)\right)}{g(n)},
    \end{aligned}
    $$
    where $g(n)$ is same across all terms and is given as $g(n) = n^{9}-1.4026n^{8}+0.5877 n^{7} + O(n^6)$.
    {Observe that for $n$ large enough}, we have that $b<a$ although the difference between them decreases with $n$.    
    Now it remains to show  that $\|\Pi P +  P\Pi - \1_n \pi^{\T}\|_\F$ decays as $O(1/\sqrt{n})$. To show this, we decompose $P$ as 
    $$ P = \frac{1}{2}\1_n\1_n^\T + Q, $$
    where $Q$ contains the residual terms. Note that $Q$ has only $n+2l$ non-zero terms. Hence, we can upper bound $\| \Pi P +  P\Pi - \1_n\pi^{\T} \|_{\F} $ as
    $$
      \begin{aligned}
    \big\| \Pi P + & P\Pi - \1_n\pi^{\T} \big\|_{\F} \\
    & = \left\| \frac{1}{2} (\pi \1_n^\T + \1_n\pi^\T)+ \Pi Q + Q \Pi - \1_n \pi^{\T} \right\|_{\F} \\
     & \leq \frac{1}{2}\left\| \pi \1_n^\T - \1_n \pi^{\T}   \right\|_{\F}  + \left\| \Pi Q + Q \Pi \right\|_{\F} \\
     &\leq \left\| a \1_n^\T -  \pi^{\T}   \right\|_{2} + \left\| b \1_n^\T -  \pi^{\T}   \right\|_{2} \\
     & \ \ \ \ \ \ \ \ \  + \frac{1}{2} \left\| u \1_{n-2}^\T - \1_{n-2} u^{\T}  \right\|_{\F}  + \left\| \Pi Q + Q \Pi \right\|_{\F} ,
      \end{aligned}
    $$
    where $u \in \R^{n-2}$ is a vector such that and $u_i = \pi_{i+2}$ for $i\in [n-2]$. It is easy to see that $\| a \1_n^\T -  \pi^{\T}   \|_{2} \leq O(1/\sqrt{n})$ as each term of the vector is bounded above by $2/n$. Moreover,
    observe that for any pair $x,y \in \{c,d,e\}$ the absolute difference $|x-y| \leq 2/n^2$ and hence, $\| u \1_{n-2}^\T - \1_{n-2} u^{\T}   \|_{\F} $ is $O(1/n^2)$. Now it is easy to show that 
    $$\| \Pi Q + Q \Pi   \|_{\F} \leq \frac{c}{\sqrt{n}},$$
    as $\Pi Q + Q \Pi$ has only $O(n)$ non zero terms each of which is bounded above by $2/n$.     
    The proof then follows by a simple calculation. 
\end{proof}

Part 1 of \cref{Prop: Stability of BTL Model} highlights the relationship between Borda and BTL rankings in terms of the weighted sum of entries of the error matrix. Moreover, part 2 conveys that the BTL assumption may potentially generate a ``wrong'' ranking (with respect to the Borda ranking) when the underlying pairwise comparison matrix is $O(1/\sqrt{n})$-separation distance away from the class of BTL models. 
Interestingly, this $O(1/\sqrt{n})$ deviation coincides with the critical threshold for the BTL testing problem (up to constant factors). 
Interestingly, this $O(1/\sqrt{n})$ deviation coincides with the critical threshold for the BTL testing problem (up to constant factors) when considering constant $k$ - we note that our stability analysis is conducted with respect to the true underlying probabilities rather than empirical probabilities which explains why our results do not exhibit dependence on $k$. It would be interesting to further explore the stability of the BTL assumption in the context of rankings in future work. We emphasize that our results here only consider stability in the complete graph case. When the induced graph is not complete, defining Borda rankings becomes ambiguous. 
Indeed, for a general graph, Borda rankings may not coincide with BTL rankings even under the BTL model. Finally, to verify our calculations in the proof, we also plot numerically computed values in \cref{fig:stability} that confirm the theoretical calculation that $\sqrt{n}\|\Pi P+ P\Pi - \1_n\pi^\T \|_\F$ converges to a constant.
    \begin{figure}
        \centering
        \includegraphics[width = 0.35\textwidth]{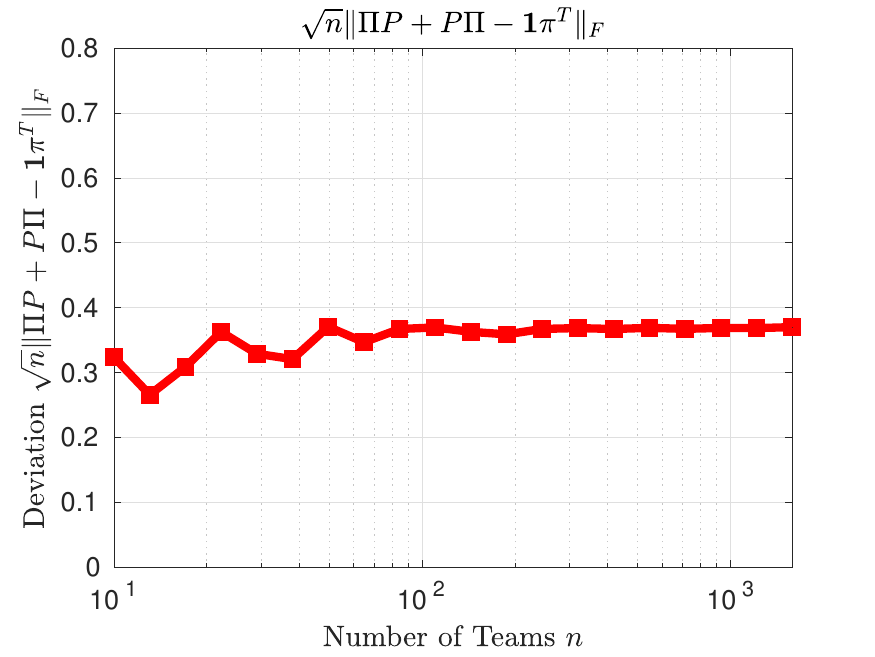}
        \caption{Plot of $\sqrt{n}\|\Pi P+ P\Pi - \1\pi^\T \|_\F$.}
        \label{fig:stability}
    \end{figure}
    
\section*{Acknowledgment}
The authors would like to thank Dongmin Lee and William Lu for discussions.

\balance

\bibliographystyle{IEEEtran}
\bibliography{journalv10_final_references}

\end{document}